\documentclass{article}

\newif\ifarxiv
\arxivtrue
\ifarxiv
\usepackage[final]{neurips_2020}
\else

\usepackage[final]{neurips_2020}

\fi

\usepackage[utf8]{inputenc} % allow utf-8 input
\usepackage[T1]{fontenc}    % use 8-bit T1 fonts
\usepackage{hyperref}       % hyperlinks
\hypersetup{colorlinks=true,linkcolor=black,anchorcolor=black,citecolor=black,filecolor=black,menucolor=black,runcolor=black,urlcolor=black}
\usepackage{booktabs}       % professional-quality tables
\usepackage{amsfonts}       % blackboard math symbols
\usepackage{nicefrac}       % compact symbols for 1/2, etc.
\usepackage{microtype}      % microtypography
\usepackage{tcolorbox}
\usepackage{subfig}
\usepackage[ruled,vlined]{algorithm2e}
\usepackage{graphicx}
\usepackage[percent]{overpic}

\usepackage{float}
\usepackage{amsmath}
\usepackage{amssymb}
\usepackage{amsthm}
\usepackage{xcolor}
\pdfoutput=1
\newcommand{\dho}{\partial}
\newcommand{\cO}{\mathcal{O}}
\newtheorem{theorem}[]{Theorem}

\newtheorem*{theoreminfo*}{Theorem (Informal)}

\newtheorem*{conjecture*}{Main point}
\newtheorem{lemma}[]{Lemma}
\usepackage{natbib}
\newcommand{\GG}[1]{}

\newcommand{\ALc}[1]{}

\newcommand{\Ltot}{L_{\mathrm{tot}}}
\newcommand{\bR}{\mathbb{R}}

\newcommand{\AutoLtwo}{\textsc{Auto$L_2$}}

\newcommand{\lexpp}[1]{\mathbb{E}_{#1}\left[}
\newcommand{\rexp}{\right]}

\title{On the training dynamics of deep networks with $L_2$ regularization}

\author{%
  Aitor Lewkowycz  \\
  Google \\ Mountain View, CA \\
  \texttt{alewkowycz@google.com}
  \And Guy Gur-Ari \\
  Google \\ Mountain View, CA \\
  \texttt{guyga@google.com} \\
}

\begin{document}

\maketitle

\begin{abstract}
We study the role of $L_2$ regularization in deep learning, and uncover simple relations between the performance of the model, the $L_2$ coefficient, the learning rate, and the number of training steps.
These empirical relations hold when the network is overparameterized. 
They can be used to predict the optimal regularization parameter of a given model.
In addition, based on these observations we propose a dynamical schedule for the regularization parameter that improves performance and speeds up training.
We test these proposals in modern image classification settings.
Finally, we show that these empirical relations can be understood theoretically in the context of infinitely wide networks.
We derive the gradient flow dynamics of such networks, and compare the role of $L_2$ regularization in this context with that of linear models. 
\end{abstract}

\section{Introduction}

Machine learning models are commonly trained with $L_2$ regularization.
This involves adding the term $\frac{1}{2} \lambda \| \theta \|_2^2$ to the loss function, where $\theta$ is the vector of model parameters and $\lambda$ is a hyperparameter.
In some cases, the theoretical motivation for using this type of regularization is clear.
For example, in the context of linear regression, $L_2$ regularization increases the bias of the learned parameters while reducing their variance across instantiations of the training data; in other words, it is a manifestation of the bias-variance tradeoff.
In statistical learning theory, a ``hard'' variant of $L_2$ regularization, in which one imposes the constraint $\| \theta \|_2 \le \epsilon$, is often employed when deriving generalization bounds.

In deep learning, the use of $L_2$ regularization is prevalent and often leads to improved performance in practical settings \citep{hinton1986}, although the theoretical motivation for its use is less clear.
Indeed, it well known that overparameterized models overfit far less than one may expect \citep{2016arXiv161103530Z}, and so the classical bias-variance tradeoff picture does not apply \citep{2017arXiv170608947N,2018arXiv181211118B,2020JSMTE..02.3401G}.
There is growing understanding that this is caused, at least in part, by the (implicit) regularization properties of stochastic gradient descent (SGD) \citep{soudry2017implicit}.
The goal of this paper is to improve our understanding of the role of $L_2$ regularization in deep learning.

\subsection{Our contribution}
\label{sec:our_contr}
We study the role of $L_2$ regularization when training over-parameterized deep networks, taken here to mean networks that can achieve training accuracy 1 when trained with SGD. 
Specifically, we consider the early stopping performance of a model, namely the maximum test accuracy a model achieves during training, as a function of the $L_2$ parameter $\lambda$.
We make the following observations based on the experimental results presented in the paper.
\begin{enumerate}
\item The number of SGD steps until a model achieves maximum performance is $t_* \approx \frac{c}{\lambda}$, where $c$ is a coefficient that depends on the data, the architecture, and all other hyperparameters.
  We find that this relationship holds across a wide range of $\lambda$ values.
\item If we train with a fixed number of steps, model performance peaks at a certain value of the $L_2$ parameter.
However, if we train for a number of steps proportional to $\lambda^{-1}$ then performance improves with decreasing $\lambda$.
In such a setup, performance becomes independent of $\lambda$ for sufficiently small $\lambda$. 
Furthermore, performance with a small, non-zero $\lambda$ is often better than performance without any $L_2$ regularization.
\end{enumerate}

Figure~\ref{fig:fig1a} shows the performance of an overparameterized network as a function of the $L_2$ parameter $\lambda$.
When the model is trained with a fixed steps budget, performance is maximized at one value of $\lambda$.
However, when the training time is proportional to $\lambda^{-1}$, performance improves and approaches a constant value as we decrease $\lambda$.

\begin{figure}[]
  \centering
  \subfloat[$\lambda$ independence]{
    \includegraphics[width=0.33\textwidth]{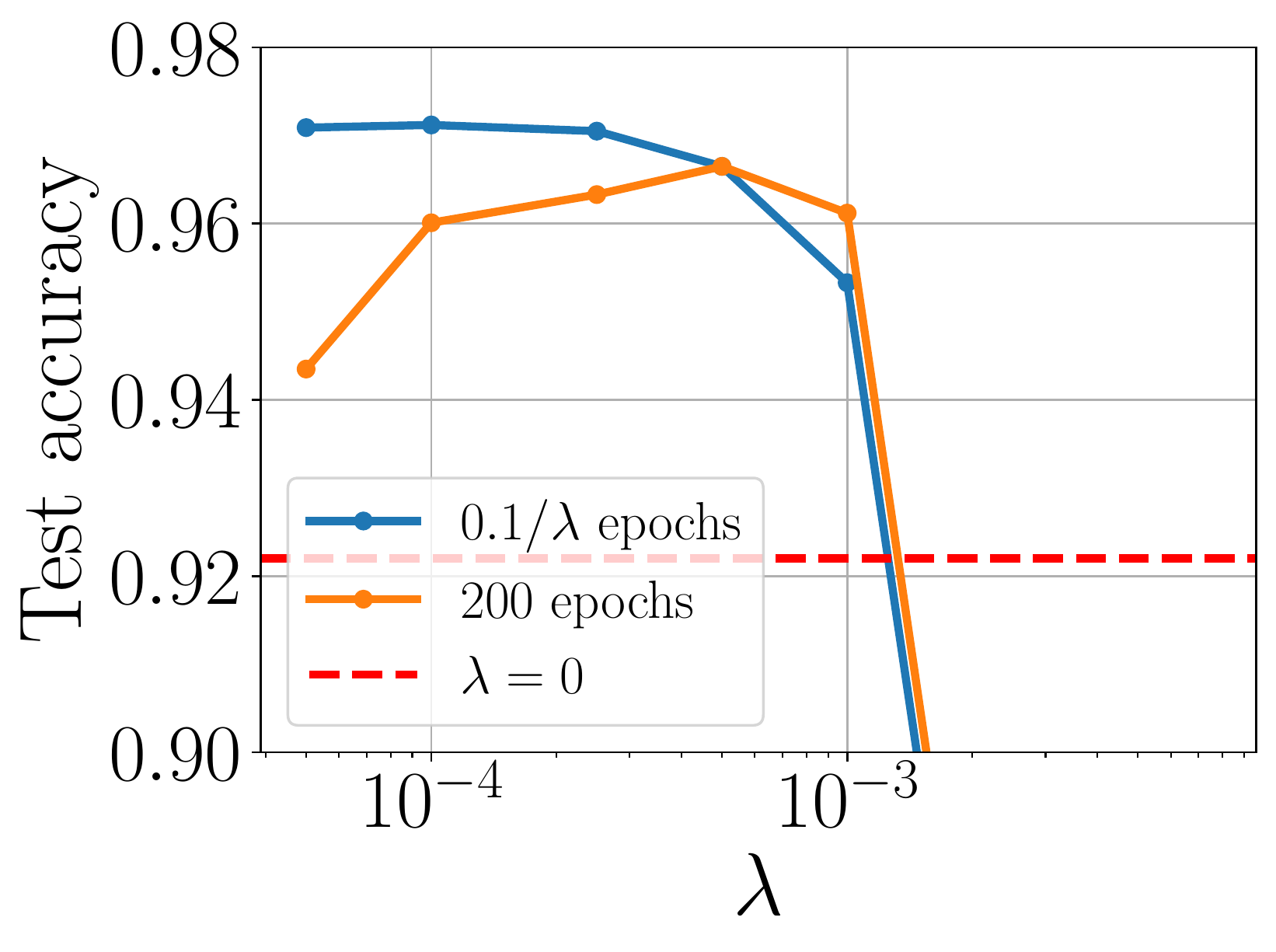}
    \label{fig:fig1a}
  }
  \subfloat[Optimal $\lambda$ prediction]{
    \includegraphics[width=0.33\textwidth]{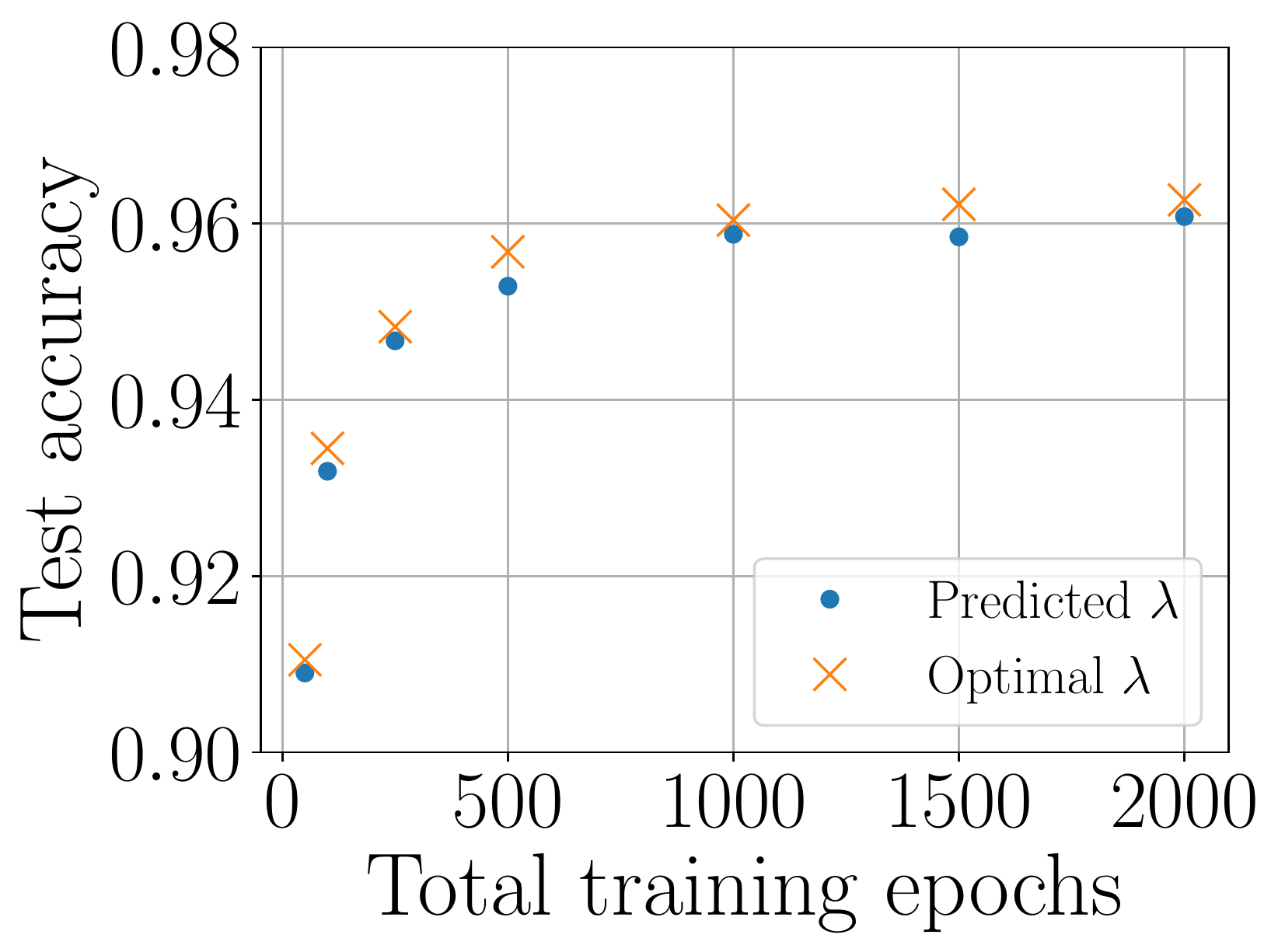}
    \label{fig:fig1b}
  }
  \subfloat[\AutoLtwo~schedule]{
    \includegraphics[width=0.33\textwidth]{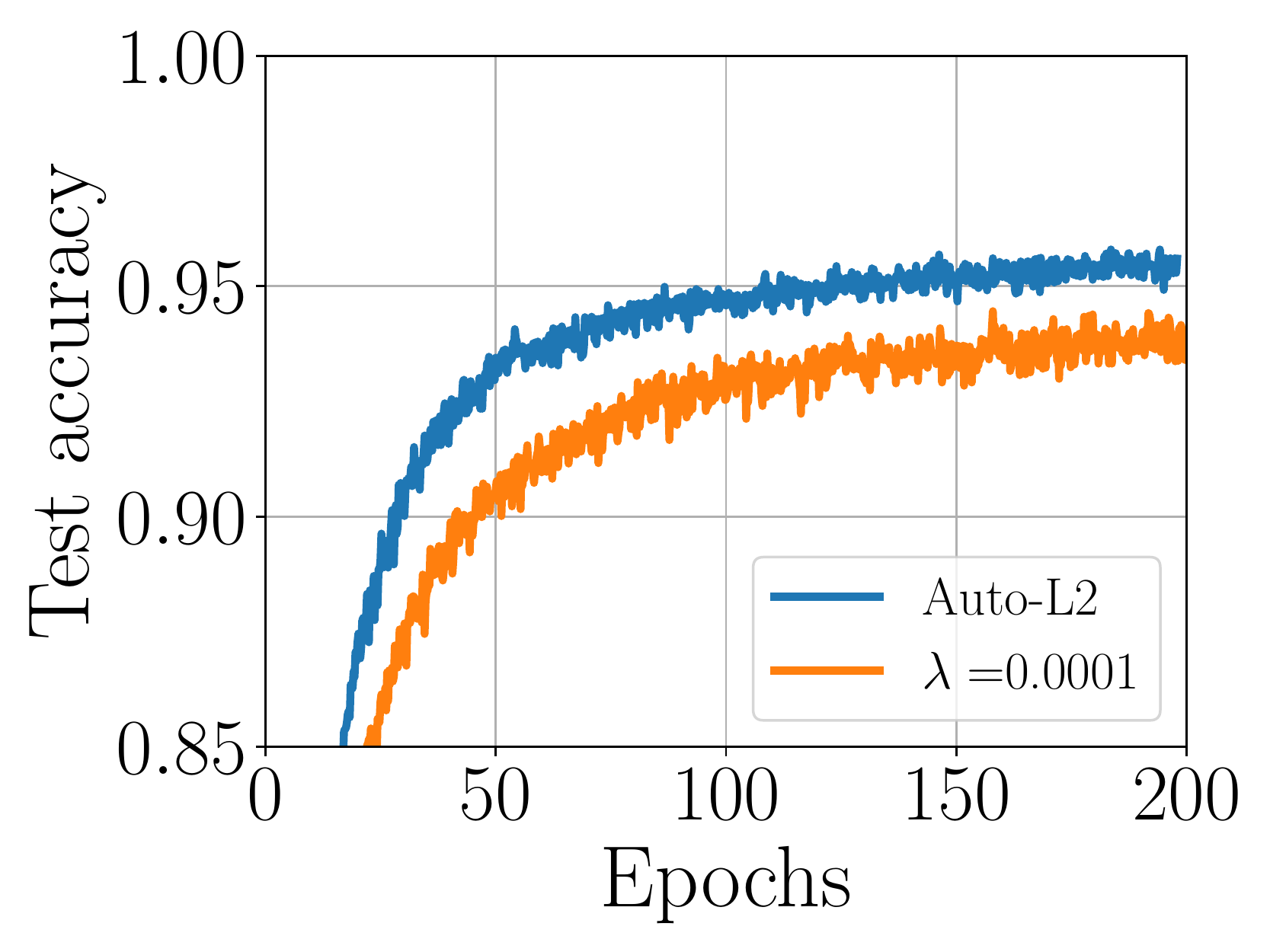}
    \label{fig:fig1c}
  }
  \caption{Wide ResNet 28-10 trained on CIFAR-10 with momentum and data augmentation. (a) Final test accuracy vs. the $L_2$ parameter $\lambda$. 
  When the network is trained for a fixed amount of epochs, optimal performance is achieved at a certain value of $\lambda$.
  But when trained for a time proportional to $\lambda^{-1}$, performance plateaus and remains constant down to the lowest values of $\lambda$ tested. This experiment includes a learning rate schedule.
  (b) Test accuracy vs. training epochs for predicted optimal $L_2$ parameter compared with the tuned parameter. 
  (c) Training curves with our dynamical $L_2$ schedule, compared with a tuned, constant $L_2$ parameter.
  }
  \label{fig:fig1}
\end{figure}

As we demonstrate in the experimental section, these observations hold for a variety of training setups which include different architectures, data sets, and optimization algorithms.
In particular, when training with vanilla SGD (without momentum), we observe that the number of steps until maximum performance depends on the learning rate $\eta$ and on $\lambda$ as $t_* \approx \frac{c'}{\eta \cdot \lambda}$.
The performance achieved after this many steps depends only weakly on the choice of learning rate.

\paragraph{Applications.}
We present two practical applications of these observations.
First, we propose a simple way to predict the optimal value of the $L_2$ parameter, based on a cheap measurement of the coefficient $c$.
Figure~\ref{fig:fig1b} compares the performance of models trained with our predicted $L_2$ parameter with that of models trained with a tuned parameter.
In this realistic setting, we find that our predicted parameter leads to performance that is within 0.4\% of the tuned performance on CIFAR-10, at a cost that is marginally higher than a single training run.
As shown below, we also find that the predicted parameter is consistently within an order of magnitude of the optimal, tuned value.

As a second application we propose \AutoLtwo, a dynamical schedule for the $L_2$ parameter. 
The idea is that large $L_2$ values achieve worse performance but also lead to faster training.
Therefore, in order to speed up training one can start with a large $L_2$ value and decay it during training (this is similar to the intuition behind learning rate schedules).
In Figure~\ref{fig:fig1c} we compare the performance of a model trained with \AutoLtwo  ~against that of a tuned but constant $L_2$ parameter, and find that \AutoLtwo ~outperforms the tuned model both in speed and in performance.

\paragraph{Learning rate schedules.}
Our empirical observations apply in the presence of learning rate schedules.
In particular, Figure~\ref{fig:fig1a} shows that the test accuracy remains approximately the same if we scale the training time as $1/\lambda$.
As to our applications, in section \ref{sec:L2opt} we propose an algorithm for predicting the optimal $L_2$ value in the presence of learning rate schedules, and the predicted value gives comparable performance to the tuned result. 
As to the \AutoLtwo~algorithm, we find that in the presence of learning rate schedules it does not perform as well as a tuned but constant $L_2$ parameter. We leave combining \AutoLtwo~with learning rate schedules to future work.

% as discussed in section \ref{sec:L2opt} we found that our predicted optimal $L_2$ value is within an order of magnitude of the tuned value when using learning rate schedules, when accounting for the schedule by weighing each step according to the learning rate.
% Therefore, as in the case of fixed learning rate, our prediction is beneficial for hyperparameter tuning.
% For the AutoL2 scheduler, in the presence of learning rate decay we find that it performs similarly to a tuned, fixed $L_2$ parameter but does not exceed it.
% This is still beneficial, in that it saves the need to tune the $L_2$ parameter.

\paragraph{Theoretical contribution.}
Finally, we turn to a theoretical investigation of the empirical observations made above.
As a first attempt at explaining these effects, consider the following argument based on the loss landscape.
For overparameterized networks, the Hessian spectrum evolves rapidly during training \citep{2017arXiv170604454S,2018arXiv181204754G,2019arXiv190110159G}.
After a small number of training steps with no $L_2$ regularization, the minimum eigenvalue is found to be close to zero.
In the presence of a small $L_2$ term, we therefore expect that the minimal eigenvalue will be approximately $\lambda$.
In quadratic optimization, the convergence time is inversely proportional to the smallest eigenvalue of the Hessian
\footnote{
  In linear regression with $L_2$ regularization, optimization is controlled by a linear kernel $K = X^T X + \lambda I$, where $X$ is the sample matrix and $I$ is the identity matrix in parameter space.
Optimization in each kernel eigendirection evolves as $e^{-\gamma t}$ where $\gamma$ is the corresponding eigenvalue.
When $\lambda>0$ and the model is overparameterized, the lowest eigenvalue of the kernel will be typically close to $\lambda$, and therefore the time to convergence will be proportional to $\lambda^{-1}$. 
}
\ifarxiv
, see \cite{pmlr-v89-ali19a} for a recent discussion
\fi
.
Based on this intuition, we may then expect that convergence time will be proportional to $\lambda^{-1}$.
The fact that performance is roughly constant for sufficiently small $\lambda$ can then be explained if overfitting can be mostly attributed to optimization in the very low curvature directions 
\ifarxiv
\citep{rahaman2018spectral,whitening}.
\else
\citep{rahaman2018spectral}.
\fi
Now, our empirical finding is that the time it takes the network to reach maximum accuracy is proportional to $\lambda^{-1}$.
In some cases this is the same as the convergence time, but in other cases (see for example Figure~\ref{fig:optimalL2a}) we find that performance decays after peaking and so convergence happens later.
Therefore, the loss landscape-based explanation above is not sufficient to fully explain the effect.

To gain a better theoretical understanding, we consider the setup of an infinitely wide neural network trained using gradient flow.
We focus on networks with positive-homogeneous activations, which include deep networks with ReLU activations, fully-connected or convolutional layers, and other common components.
By analyzing the gradient flow update equations of such networks, we are able to show that the performance peaks at a time of order $\lambda^{-1}$ and deteriorates thereafter.
This is in contrast to the performance of linear models with $L_2$ regularization, where no such peak is evident.
These results are consistent with our empirical observations, and may help shed light on the underlying causes of these effects.

According to known infinite width theory, in the absence of explicit regularization, the kernel that controls network training  is constant \citep{NTK-paper}.
Our analysis extends the known results on infinitely wide network optimization, and indicates that the kernel decays in a predictable way in the presence of $L_2$ regularization.
We hope that this analysis will shed further light on the observed performance gap between infinitely wide networks which are under good theoretical control, and the networks trained in practical settings \citep{Arora-CNN,CNN-GP,wei2018regularization,lewkowycz2020large}.

\paragraph{Related works.}

$L_2$ regularization in the presence of batch-normalization \citep{bn} has been studied in \citep{laarhoven2017l2,hoffer2018norm,zhang2018mechanisms}. 
These papers discussed how the effect of $L_2$ on scale invariant models is merely of having an effective learning rate (and no $L_2$). This was made precise in \citet{li2019exponential} where they showed that this effective learning rate is $\eta_{\rm eff}=\eta e^{2 \eta \lambda t}$ (at small learning rates). 
Our theoretical analysis of large width networks will have has the same behaviour when the network is scale invariant. 
Finally, in parallel to this work, \citet{li2020reconciling} carried out a complementary analysis of the role of $L_2$ regularization in deep learning using a stochastic differential equation analysis.
Their conclusions regarding the effective learning rate in the presence of $L_2$ regularization are consistent with our  observations.

\section{Experiments}

\paragraph{Performance and time scales.}

We now turn to an empirical study of networks trained with $L_2$ regularization.
In this section we present results for a fully-connected network trained on MNIST, a Wide ResNet \citep{WRN} trained on CIFAR-10, and CNNs trained on CIFAR-10. The experimental details are in SM \ref{sec:expdetails}.
The empirical findings discussed in section \ref{sec:our_contr} hold across this variety of overparameterized setups.

Figure~\ref{fig:timedynamics} presents experimental results on fully-connected and Wide ResNet networks. 
Figure~\ref{fig:CNN} presents experiments conducted on CNNs. We find that the number of steps until optimal performance is achieved (defined here as the minimum time required to be within $.5\%$ of the maximum test accuracy) scales as $\lambda^{-1}$, as discussed in Section~\ref{sec:our_contr}. Our experiments span $6$ decades of $\eta \cdot \lambda$ (larger $\eta,\lambda$ won't train at all and smaller would take too long to train). Moreover, when we evolved the networks until they have reached optimal performance, the maximum test accuracy for smaller $L_2$ parameters did not get worse. 
We compare this against the performance of a model trained with a fixed number of epochs, reporting the maximum performance achieved during training.
In this case, we find that reducing $\lambda$ beyond a certain value does hurt performance.

While here we consider the simplified set up of vanilla SGD and no data augmentation, our observations also hold in the presence of momentum and data augmentation, see SM \ref{sec:morewrn} for more experiments. We would like to emphasize again that while the smaller $L_2$ models can reach the same test accuracy as its larger counterparts, models like WRN28-10 on CIFAR-10 need to be trained for a considerably larger number of epochs to achieve this.\footnote{The longer experiments ran for 5000 epochs while one usually trains these models for $\sim$300 epochs.}

\begin{figure}
   \subfloat[FC]{
     \includegraphics[width=0.33\textwidth]{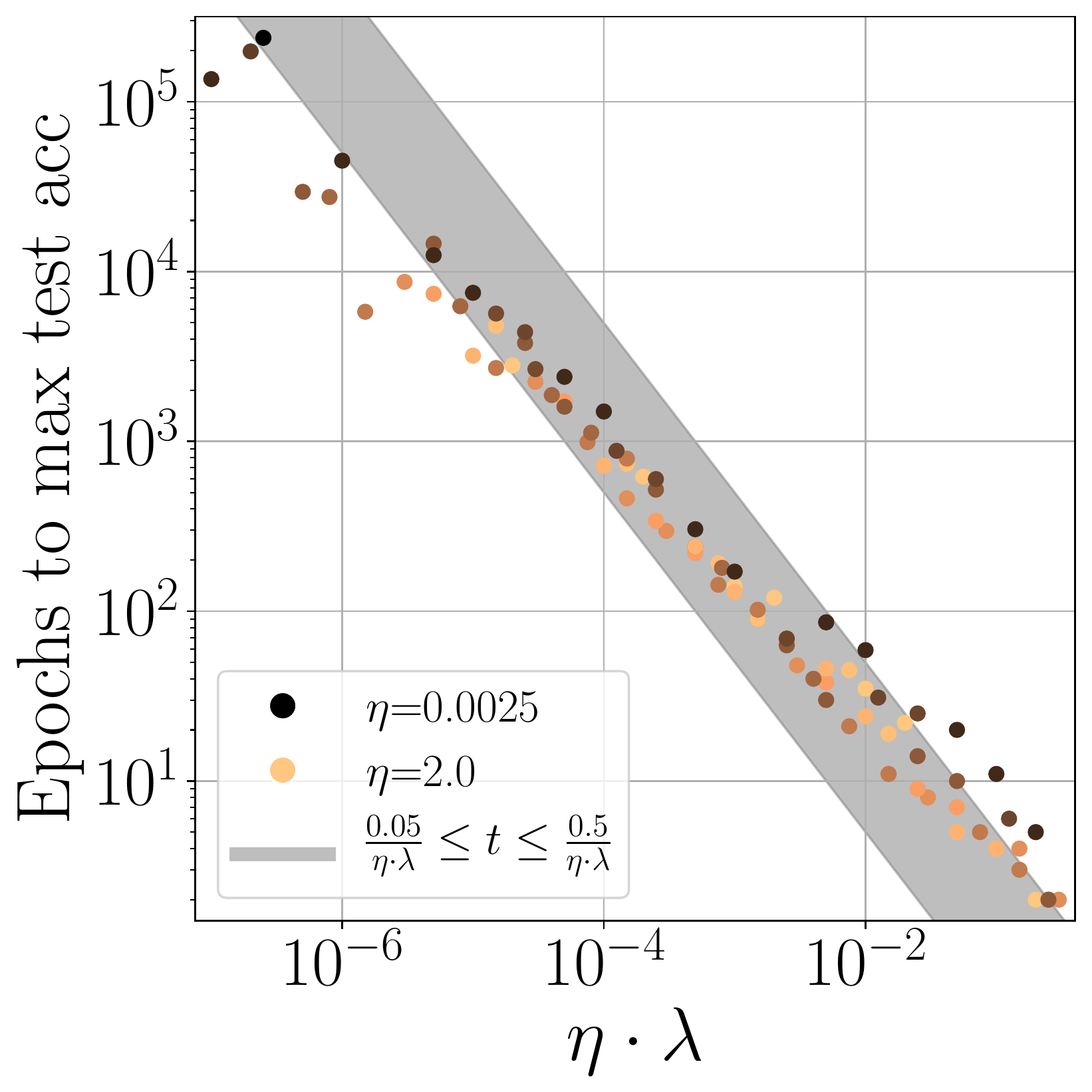}
     }
   \subfloat[FC]{
     \includegraphics[width=0.33\textwidth]{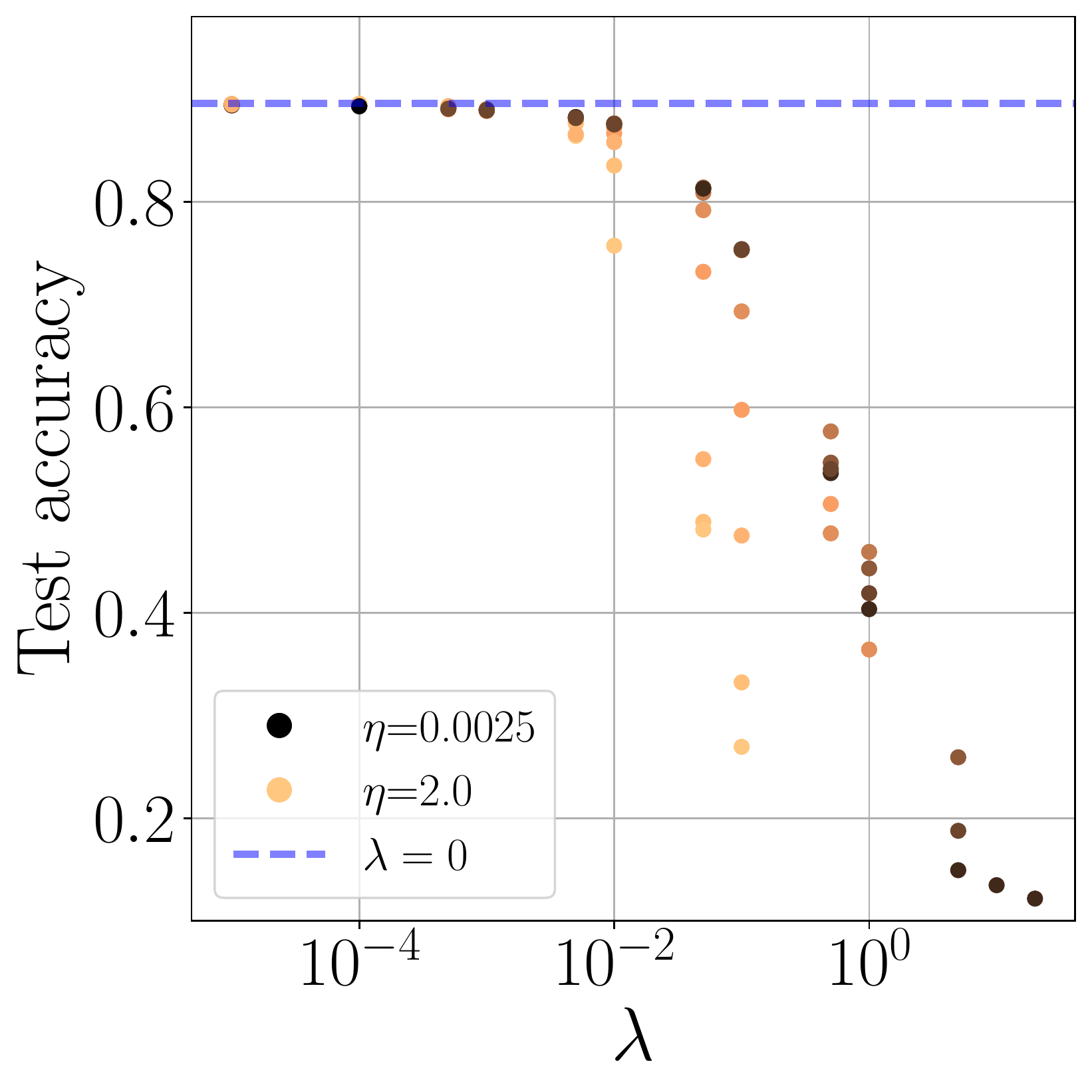}
     }
    \subfloat[FC]{
     \includegraphics[width=0.33\textwidth]{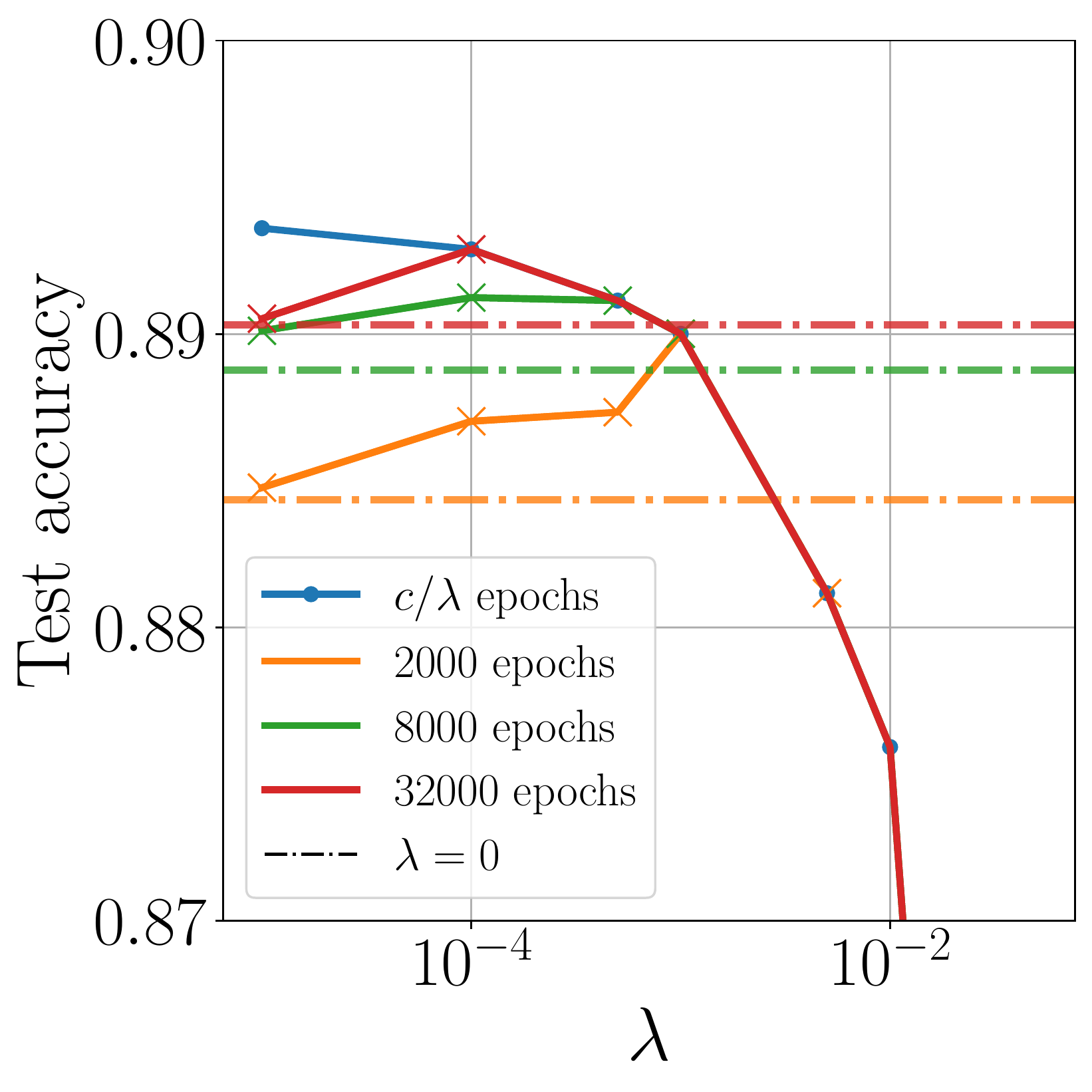}
     }
     \\
   \subfloat[WRN]{
     \includegraphics[width=0.33\textwidth]{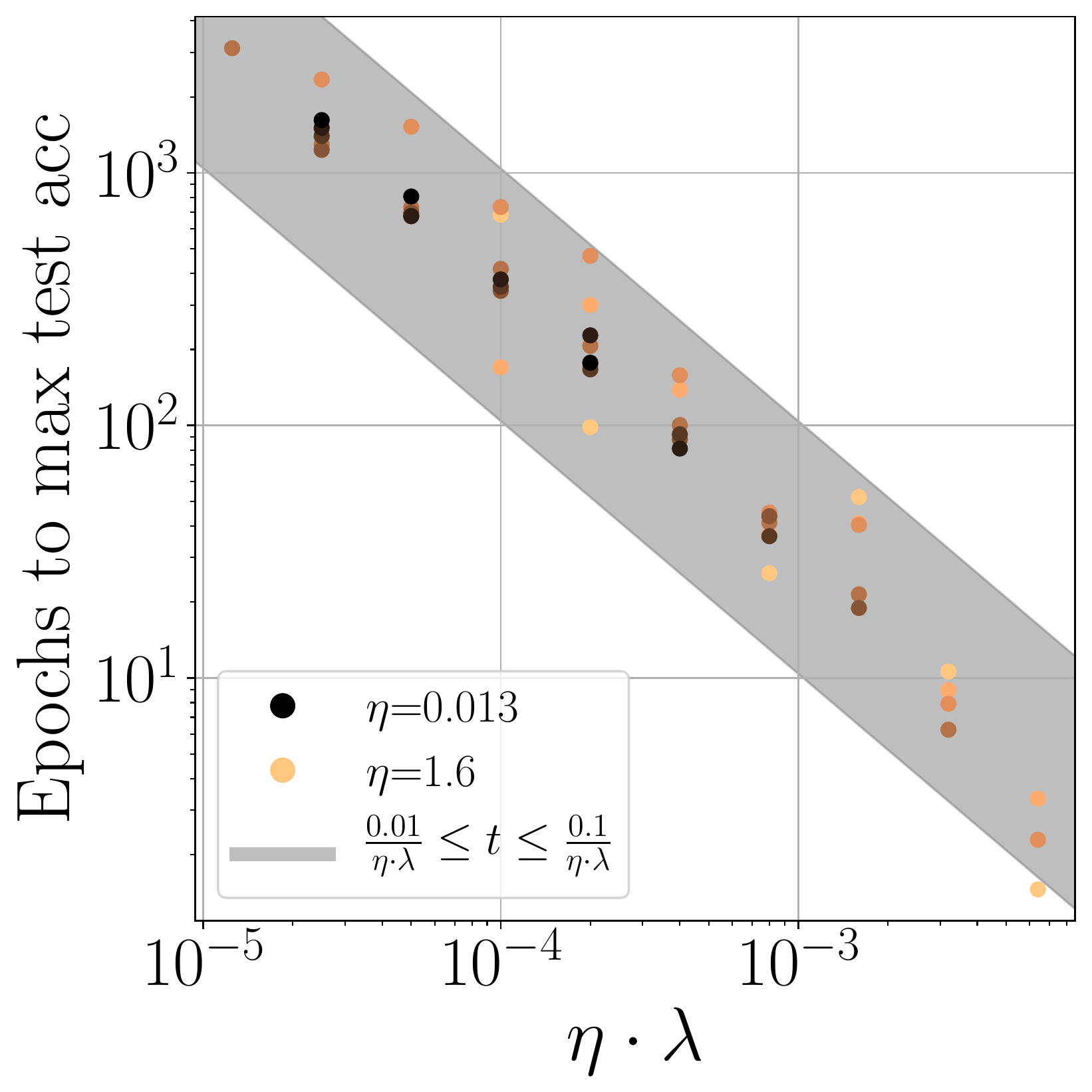}
     }
   \subfloat[WRN]{
     \includegraphics[width=0.33 \textwidth]{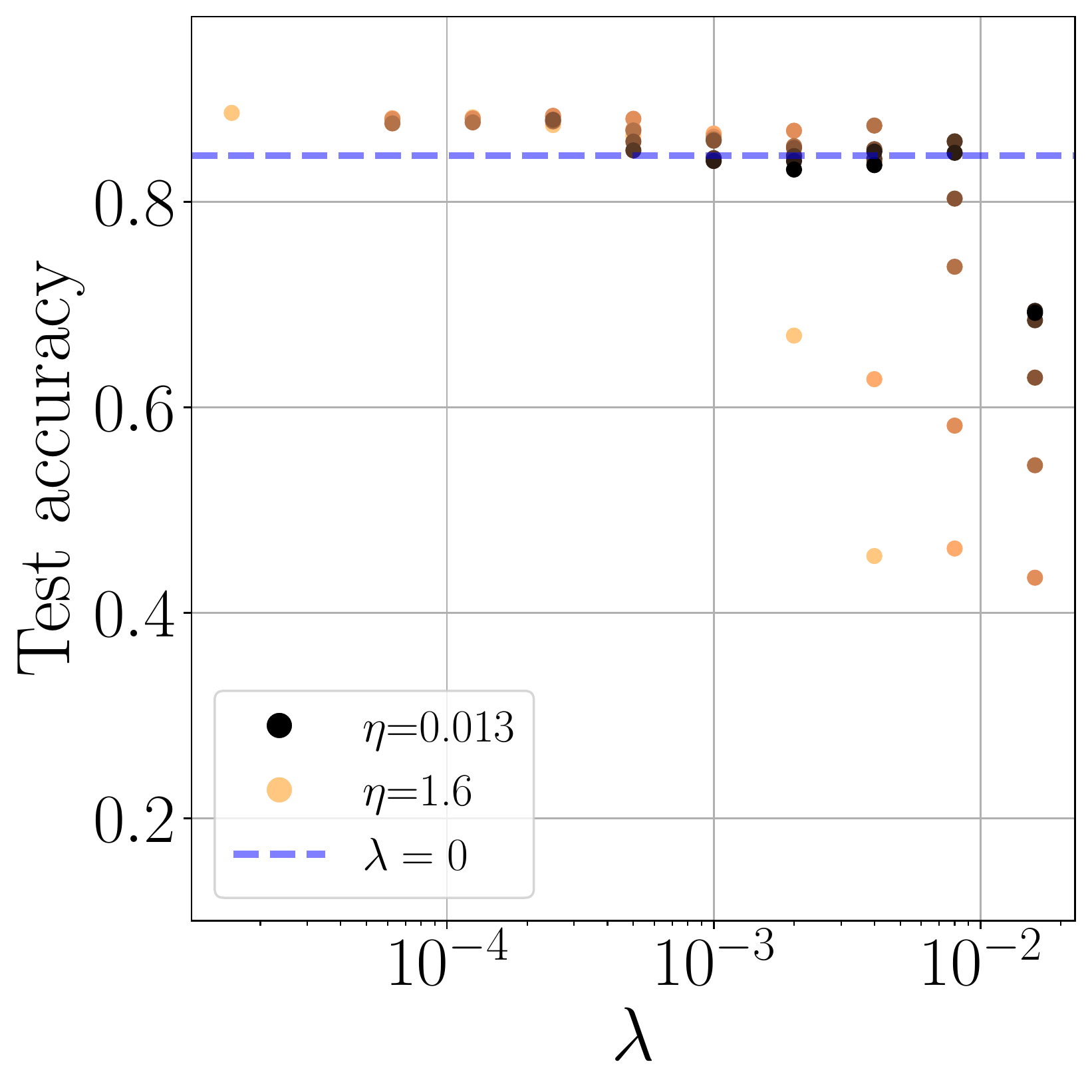}
      }
       \subfloat[WRN]{
     \includegraphics[width=0.33\textwidth]{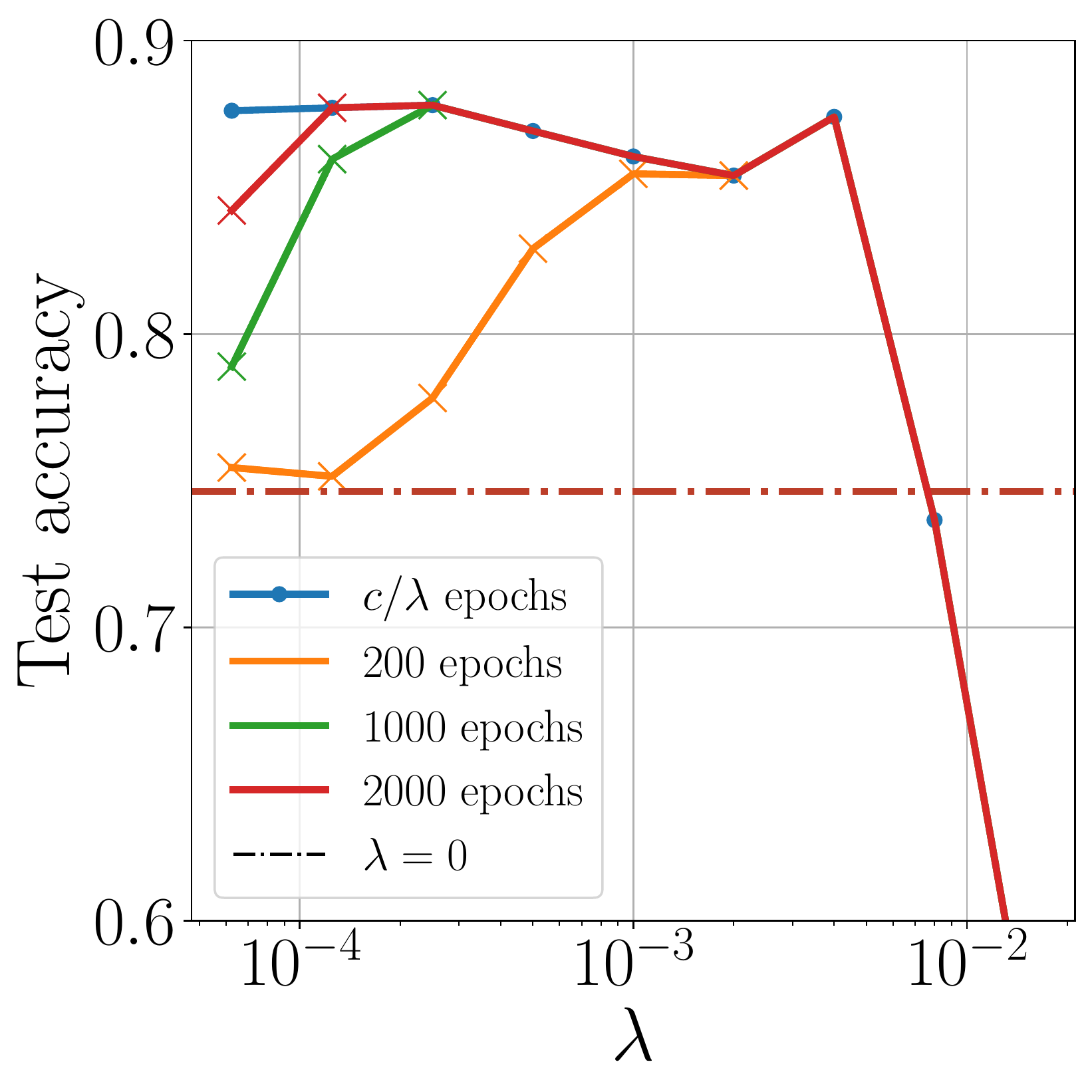}
     }

\caption{Sweep over $\eta$ and $\lambda$ illustrating how smaller $\lambda$'s require longer times to achieve the same performance. In the left, middle plots, the learning rates are logarithmically spaced between the values displayed in the legend, the specific values are in the SM \ref{sec:expdetails}. {\bf Left:} Epochs to maximum test accuracy (within $.5\%$), {\bf Middle:} Maximum test accuracy 
(the $\lambda=0$ line denotes the maximum test accuracy achieved among all learning rates), {\bf Right:} Maximum test accuracy for a fixed time budget. {\bf (a,b,c)} Fully connected 3-hidden layer neural network evaluated in $512$ MNIST samples, evolved for $t \cdot \eta \cdot \lambda = 2$.  $\eta=0.15$ in (c). {\bf (d,e,f)} A Wide Residual Network 28-10 trained on CIFAR-10 without data augmentation, evolved for $t \cdot \eta \cdot \lambda = 0.1$. In (f), $\eta=0.2$. The $\lambda=0$ line was evolved for longer than the smallest $L_2$ but there is still a gap. 
}
\label{fig:timedynamics}
\end{figure}
\begin{figure}
\subfloat[CNN No BN]{
  \includegraphics[width=0.33\textwidth]{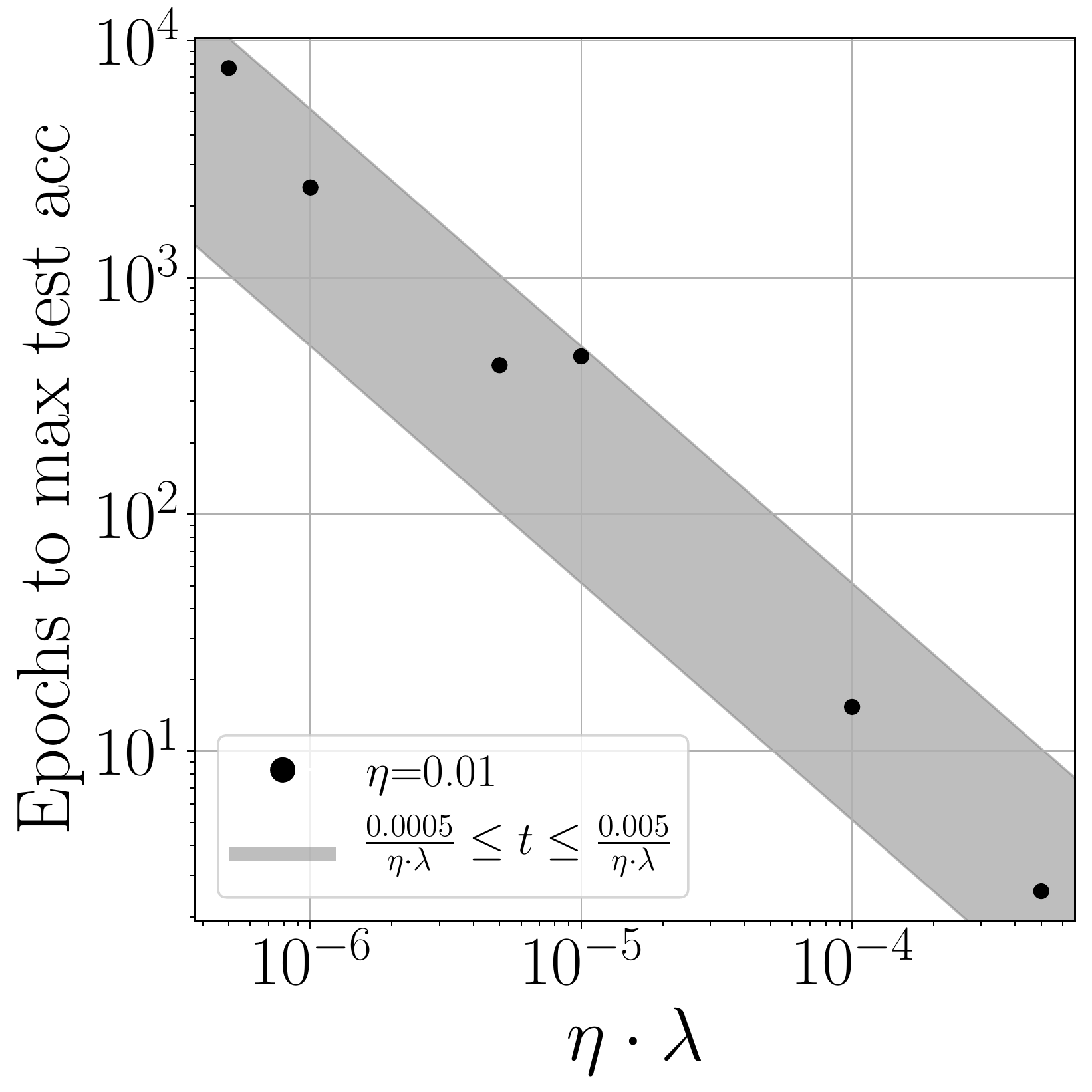}
} 
\subfloat[CNN No BN]{
  \includegraphics[width=0.33\textwidth]{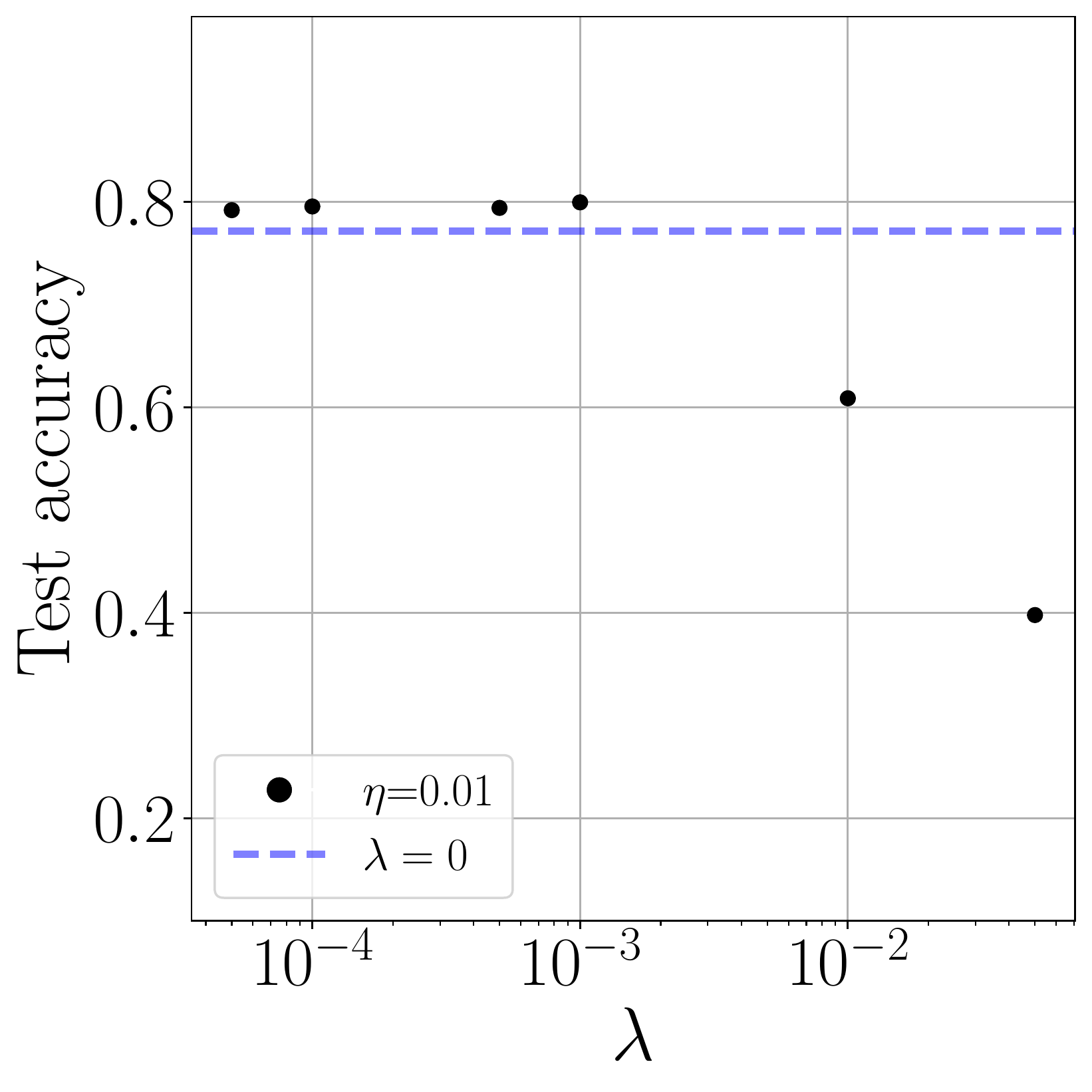}
} 
\subfloat[CNN No BN]{
  \includegraphics[width=0.33\textwidth]{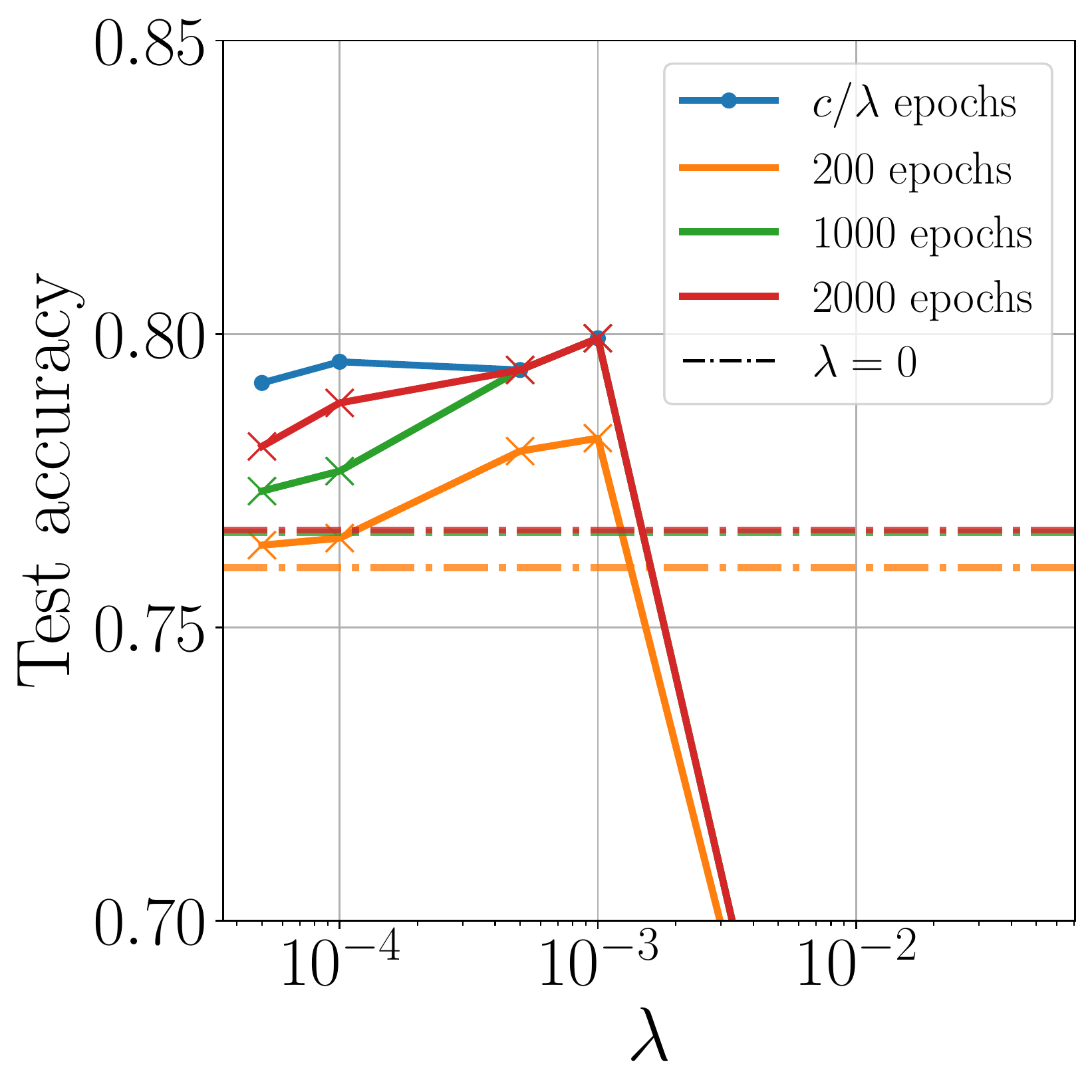}
} 
\\
\subfloat[CNN BN]{
  \includegraphics[width=0.33\textwidth]{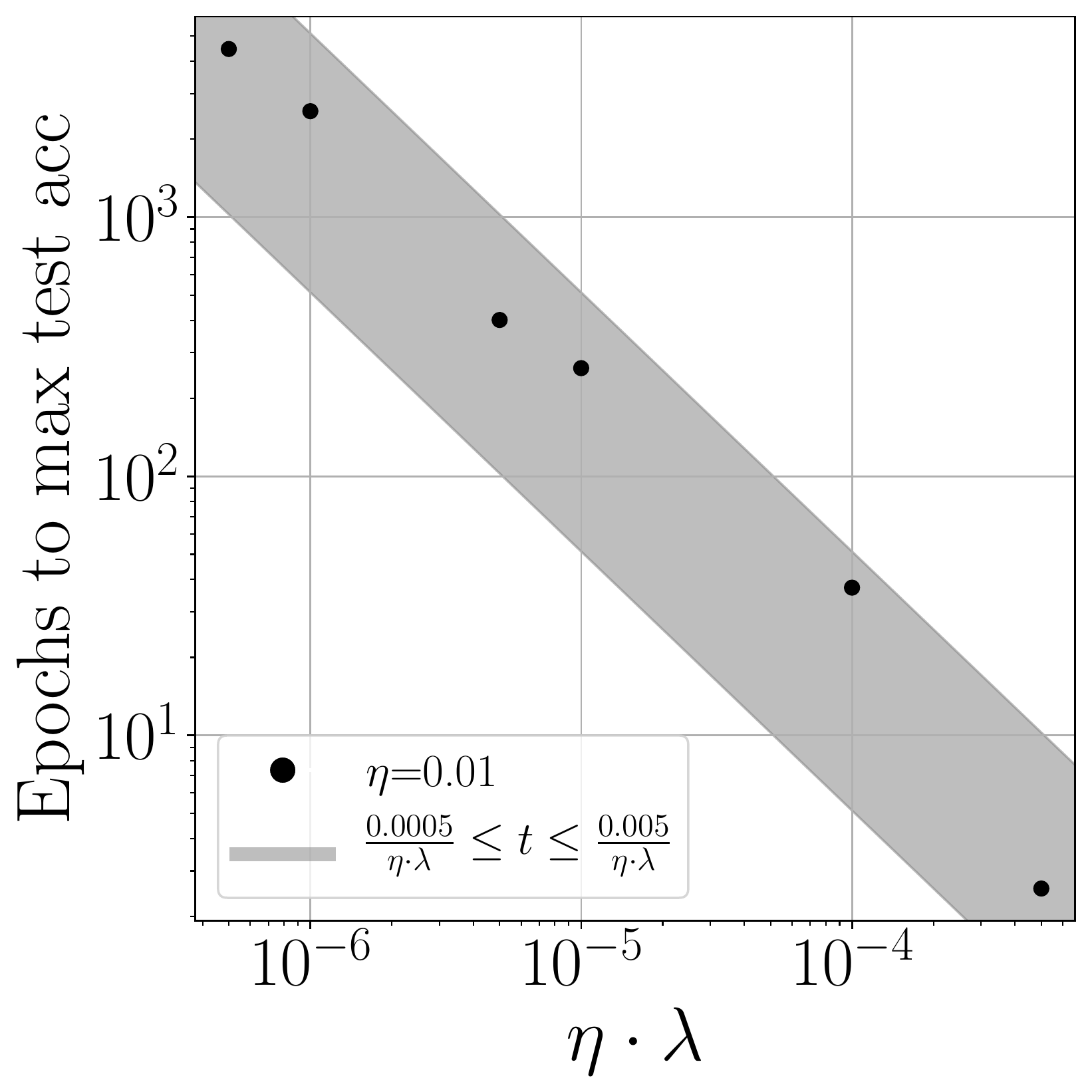}
} 
\subfloat[CNN BN]{
  \includegraphics[width=0.33\textwidth]{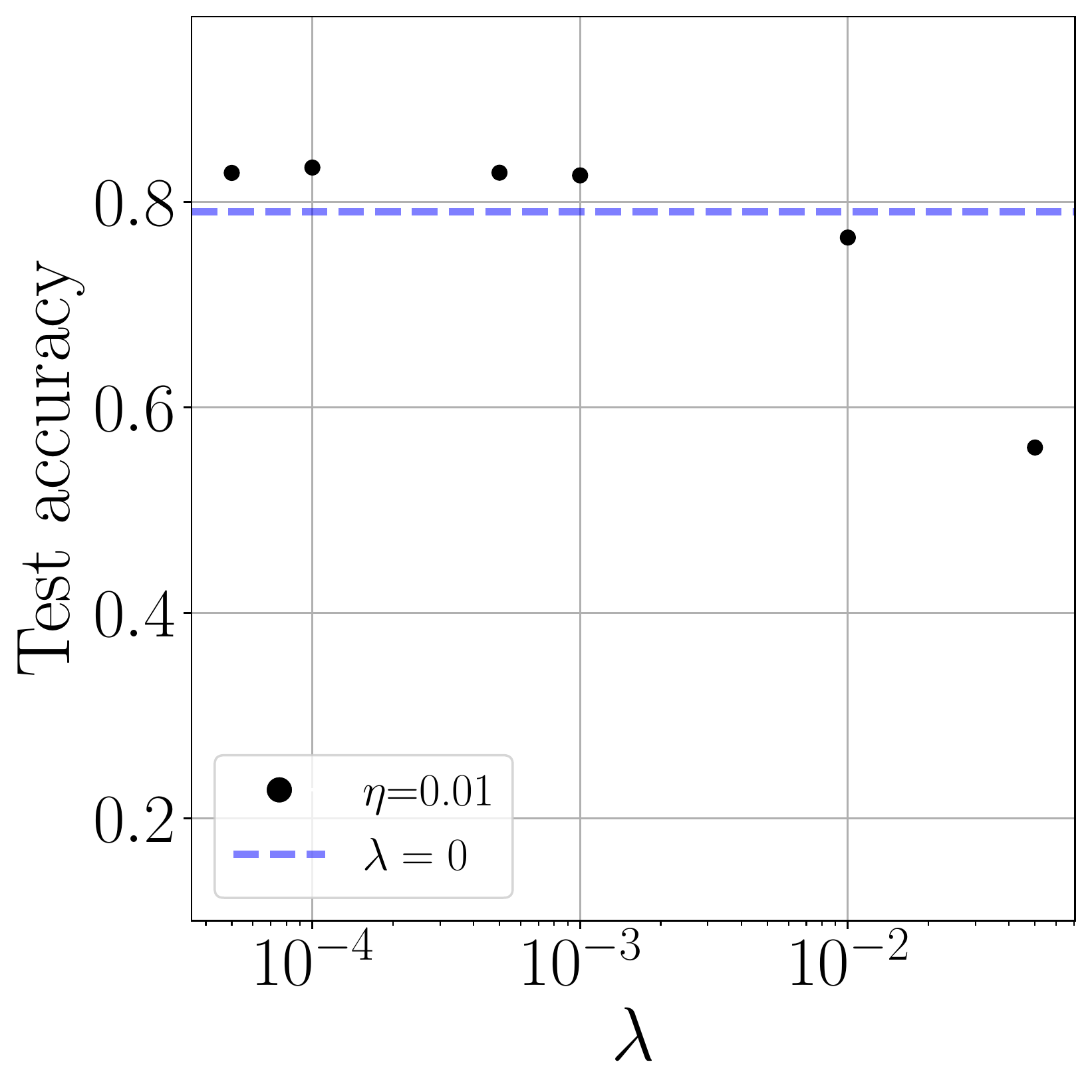}
} 
\subfloat[CNN BN]{
  \includegraphics[width=0.33\textwidth]{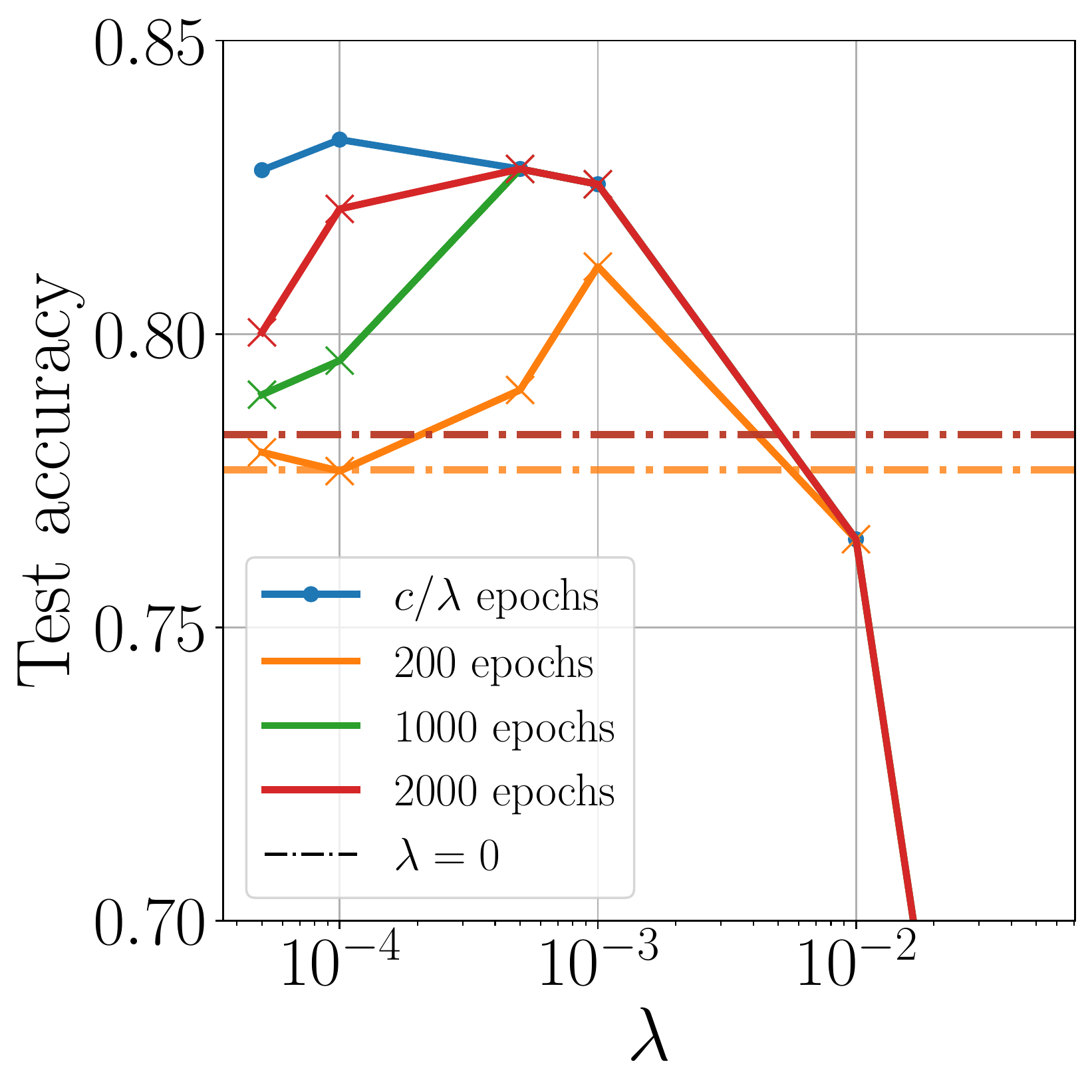}
} 
    \caption{CNNs trained with and without batch-norm with learning rate $\eta=0.01$.
    Presented results follow the same format as Figure~\ref{fig:timedynamics}.}
    \label{fig:CNN} 

\end{figure}

\paragraph{Learning rate schedules.}
\label{sec:lrsch}
So far we considered training setups that do not include learning rate schedules.
Figure~\ref{fig:fig1a} shows the results of training a Wide ResNet on CIFAR-10 with a learning rate schedule, momentum, and data augmentation. 
The schedule was determined as follows.
Given a total number of epochs $T$, the learning rate is decayed by a factor of $0.2$ at epochs $
\lbrace 0.3\cdot T,0.6\cdot T,0.9 \cdot T \rbrace$.
We compare training with a fixed $T$ against training with $T \propto \lambda^{-1}$.
We find that training with a fixed budget leads to an optimal value of $\lambda$, below which performance degrades.
On the other hand, training with $T \propto \lambda^{-1}$ leads to improved performance at smaller $\lambda$, consistent with our previous observations.

\section{Applications}

We now discuss two practical applications of the empirical observations made in the previous section.

\paragraph{Optimal $L_2$.}
\label{sec:L2opt}

We observed that the time $t_*$ to reach maximum test accuracy is proportional to $\lambda^{-1}$, which we can express as $t_{*} \approx \frac{c}{\lambda}$.
This relationship continues to hold empirically even for large values of $\lambda$.
When $\lambda$ is large, the network attains its (significantly degraded) maximum performance after a relatively short amount of training time.
We can therefore measure the value of $c$ by training the network with a large $L_2$ parameter until its performance peaks, at a fraction of the cost of a normal training run.

Based on our empirical observations, given a training budget $T$ we predict that the optimal $L_2$ parameter can be approximated by $\lambda_{\rm pred} = c/T$. This is the smallest $L_2$ parameter such that model performance will peak within training time $T$.
Figure~\ref{fig:fig1b} shows the result of testing this prediction in a realistic setting: a Wide ResNet trained on CIFAR-10 with momentum$ =0.9$ , learning rate $\eta=0.2$ and data augmentation.
The model is first trained with a large $L_2$ parameter for 2 epochs in order to measure $c$, and we find $c \approx 0.0066$, see figure \ref{fig:optimalL2a}.
We then compare the tuned value of $\lambda$ against our prediction for training budgets spanning close to two orders of magnitude, and find excellent agreement: the predicted $\lambda$'s have a performance which is rather close to the optimal one. Furthermore, the tuned values are always within an order of magnitude of our predictions see figure \ref{fig:optimalL2c}. 

%\ifarxiv
So far we assumed a constant learning rate.
In the presence of learning rate schedules, one needs to adjust the prediction algorithm.
Here we address this for the case of a piecewise-constant schedule.
For compute efficiency reasons, we expect that it is beneficial to train with a large learning rate as long as accuracy continues to improve, and to decay the learning rate when accuracy peaks.
Therefore, given a fixed learning rate schedule, we expect the optimal $L_2$ parameter to be the one at which accuracy peaks at the time of the first learning rate decay.
Our prediction for the optimal parameter is then $\lambda_{\rm pred} = c/T_1$, where $T_1$ is the time of first learning rate decay, and the coefficient $c$ is measured as before with a fixed learning rate.
In our experiments, this prediction is consistently within an order of magnitude of the optimal parameter, and gives comparable performance.
For example, in the case of Figure~\ref{fig:fig1a} with $T=200$ epochs and $T_1 = 0.3T$, we find $\lambda_{\rm pred} \approx 0.0001$ (leading to test accuracy $0.960$), compared with the optimal value $0.0005$ (with test accuracy $0.967$).

%% What we had in the arxiv version
% We have not studied this thoroughly in the presence of learning rate schedules, but intuitively, one wants to evolve for as long as possible with the initial large learning rate and it seems natural to think that the optimal $\lambda$ corresponds to the smallest $L_2$ which makes it to the (constant learning rate) peak before the first decay. Using the learning rate schedule of section \ref{sec:lrsch}, this implies $\lambda_{\rm pred} \approx \frac{c}{0.3 T}$. In the setup of $T=200$ epochs of figure \ref{fig:fig1a}, this predicts $\lambda_{\rm pred} \approx 0.0001$ (test accuracy $0.960$), which is close to the optimal at $0.0005$ (with test accuracy $0.967$).  While the predicted $L_2$ is close, this has substantial implications for performance, but we can nevertheless use it as a reference for hyperparameter tuning. We leave a more precise estimation of the optimal $L_2$ for future work. 
%\fi

\begin{figure}[ht!]
  \centering
    \subfloat[]{\includegraphics[width=0.33 \textwidth]{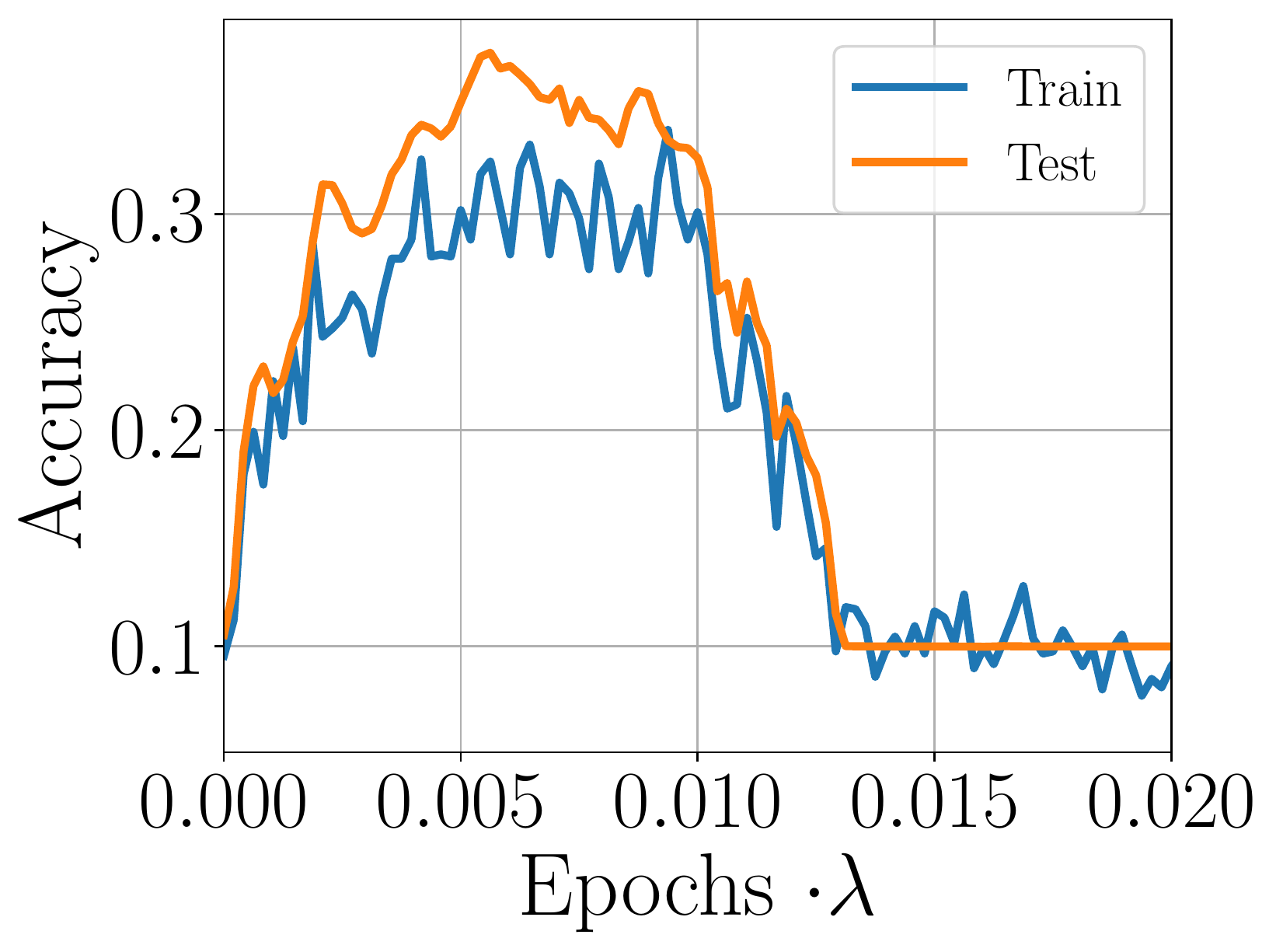}
      \label{fig:optimalL2a} }
   \subfloat[]{\includegraphics[width=0.33\textwidth]
   {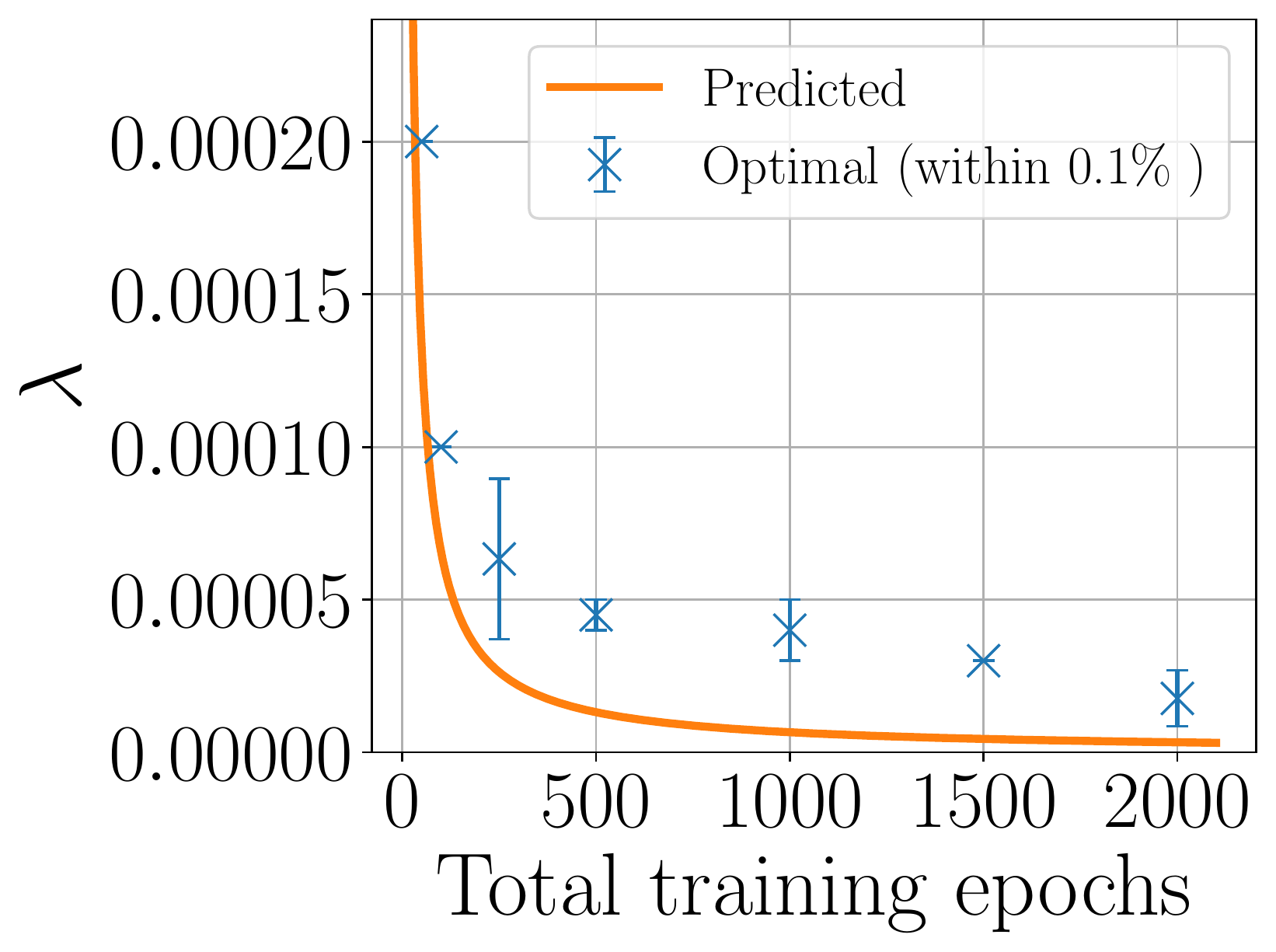}

\label{fig:optimalL2c} }
\caption{Wide ResNet trained with momentum and data augmentation. (a) We train the model with a large $L_2$ parameter $\lambda=0.01$ for 2 epochs and measure the coefficient $c=t_* \cdot \lambda \approx 0.0066$, representing the approximate point along the $x$ axis where accuracy is maximized. 
(b) Optimal (tuned) $\lambda$ values compared with the theoretical prediction.
The error bars represent the spread of values that achieve within $0.1\%$ of the optimal test accuracy.
}
\label{fig:optimalL2}
\end{figure}

\paragraph{\AutoLtwo: Automatic $L_2$ schedules.}
\label{sec:L2sch}

We now turn to another application, based on the observation that models trained with larger $L_2$ parameters reach their peak performance faster.
It is therefore plausible that one can speed up the training process by starting with a large $L_2$ parameter, and decaying it according to some schedule.
Here we propose to choose the schedule dynamically by decaying the $L_2$ parameter when performance begins to deteriorate.
See SM \ref{sec:AutoL2} for further details.

\AutoLtwo~ is a straightforward implementation of this idea: We begin training with a large parameter, $\lambda=0.1$, and we decay it by a factor of 10 if either the empirical loss (the training loss without the $L_2$ term) or the training error increases.
To improve stability, immediately after decaying we impose a refractory period during which the parameter cannot decay again.
Figure \ref{fig:fig1c}
compares this algorithm against the model with the optimal $L_2$ parameter.
We find that \AutoLtwo~trains significantly faster and achieves superior performance. See SM \ref{sec:AutoL2} for other architectures. 

In other experiments we have found that this algorithm does not yield improved results when the training procedure includes a learning rate schedule.
We leave the attempt to effectively combine learning rate schedules with $L_2$ schedules to future work.

\section{Theoretical results}

We now turn to a theoretical analysis of the training trajectory of networks trained with $L_2$ regularization.
We focus on infinitely wide networks with positively-homogeneous activations.
Consider a network function $f : \bR^d \to \bR$ with model parameter $\theta\in\bR^p$. 
The network initialized using NTK parameterization \citep{NTK-paper}: the initial parameters are sampled i.i.d. from $\mathcal{N}(0,1)$.
The model parameters are trained using gradient flow with loss $\Ltot = L + \frac{\lambda}{2} \|\theta\|_2^2$, where $L = \sum_{(x,y)\in S} \ell(x,y)$ is the empirical loss, $\ell$ is the sample loss, and $S$ is the training set of size $N_{\rm samp}$.

We say that the network function is $k$-homogeneous if $f_{\alpha \theta}(x) = \alpha^k f_\theta(x)$ for any $\alpha > 0$.  
As an example, a fully-connected network with $L$ layers and ReLU or linear activations is $L$-homogeneous. Networks made out of convolutional, max-pooling or batch-normalization layers are also $k$-homogeneous.\footnote{Batch normalization is often implemented with an $\epsilon$ parameter meant to prevent numerical instabilities. Such networks are only approximately homogeneous.} 
See \citet{li2019exponential} for a discussion of networks with homogeneous activations.

\citet{NTK-paper} showed that when an infinitely wide, fully-connected network is trained using gradient flow (and without $L_2$ regularization), its network function obeys the differential equation $\frac{df}{dt}(x) = - \sum_{x' \in S} \Theta_0(x,x') \ell'(x')$, where $t$ is the gradient flow time and $\Theta_t(x,x') = \nabla_\theta f_t(x)^T \nabla_\theta f_t(x)$ is the Neural Tangent Kernel (NTK).

\citet{feynman-diagrams} presented a conjecture that allows one to derive the large width asymptotic behavior of the network function, the Neural Tangent Kernel, as well as of combinations involving higher-order derivatives of the network function.
\ifarxiv
The conjecture was shown to hold for networks with polynomial activations \citep{aitken2020}, and has been verified empirically for commonly used activation functions.
\fi
In what follows, we will assume the validity of this conjecture.
The following is our main theoretical result.
\begin{theorem}
\label{thm1}
Consider a $k$-homogeneous network, and assume that the network obeys the correlation function conjecture of \cite{feynman-diagrams}.
In the infinite width limit, the network function $f_t(x)$ and the kernel $\Theta_t(x,x')$ evolve according to the following equations at training time $t$.
\begin{align}
    \frac{df_t(x)}{dt} &= - 
    e^{-2(k-1)\lambda t}
    \sum_{(x',y') \in S} \Theta_0(x,x')
    \frac{\dho \ell(x',y')}{\dho f_t}
    - \lambda k f_t(x)
    \label{eq:mainthm}
    \,, \\
    \frac{d\Theta_t(x,x')}{dt} &= -2 (k-1) \lambda \Theta_t(x,x') \,.
    \label{eq:mainthm_theta}
\end{align}

\end{theorem}
The proof hinges on the following equation, which holds for $k$-homogeneous functions: $\sum_\mu \theta_\mu \partial_\mu \partial_{\nu_1} \cdots \partial_{\nu_m} f(x) = (k-m) \partial_{\nu_1} \cdots \partial_{\nu_m} f(x)$.
This equation allows us to show that the only effect of $L_2$ regularization at infinite width is to introduce simple terms proportional to $\lambda$ in the gradient flow update equations for both the function and the kernel.

We refer the reader to the SM for the proof.
We mention in passing that the case $k=0$ corresponds to a scaling-invariant network function which was studied in \citet{li2019exponential}.
In this case, training with $L_2$ term is equivalent to training with an exponentially increasing learning rate.

For commonly used loss functions, and for $k>1$, we expect that the solution obeys $\lim_{t\to\infty} f_t(x) = 0$.
We will prove that this holds for MSE loss, but let us first discuss the intuition behind this statement.
At late times the exponent in front of the first term in \eqref{eq:mainthm} decays to zero, leaving the approximate equation $\frac{df(x)}{dt} \approx - \lambda k f(x)$ and leading to an exponential decay of the function to zero.
Both the explicit exponent in the equation, and the approximate late time exponential decay, suggest that this decay occurs at a time $t_{\rm decay} \propto \lambda^{-1}$.
Therefore, we expect that the minimum of the empirical loss to occur at a time proportional to $\lambda^{-1}$, after which the bare loss will increase because the function is decaying to zero. 
We observe this behaviour empirically for wide fully-connected networks and for Wide ResNet in the SM. 

\ifarxiv
Furthermore, notice that if we include the $k$ dependence, the decay time scale is approximately $t_{\rm decay} \propto (k \lambda)^{-1}$.
Models with a higher degree of homogeneity (for example deeper fully-connected networks) will converge faster.
\fi

We now focus on MSE loss and solve the gradient flow equation \eqref{eq:mainthm} for this case.
\begin{theorem}\label{thm:mse}
Let the sample loss be $\ell(x,y) = \frac{1}{2} (f(x) - y)^2$, and assume that $k\ge 2$.
Suppose that, at initialization, the kernel $\Theta_0$ has eigenvectors $\hat{e}_a \in \bR^{N_{\rm samp}}$ with corresponding eigenvalues $\gamma_a$.
Then during gradient flow, the eigenvalues evolve as $\gamma_a(t) = \gamma_a e^{-2(k-1)\lambda t}$ while the eigenvectors are static.
Suppose we treat $f \in \bR^{N_{\rm samp}}$ as a vector defined on the training set.
Then each mode of the function, $f_a := (\hat{e}_a)^T f \in \bR$, evolves independently as
\begin{align}
    f_a(x;t) &= e^{
    \frac{\gamma_a(t)}{2(k-1)\lambda} - k\lambda t
    }
    \Bigg\{ 
    e^{-\frac{\gamma_a}{2(k-1)\lambda}} f_a(x;0)
    + \gamma_a y_a
    \int_0^t \! dt' \, \exp \left[ 
    - \frac{\gamma_a(t')}{2(k-1)\lambda} - (k-2) \lambda t'
    \right]
    \Bigg\} 
    \,. \label{eq:fmodes}
\end{align}
Here, $y_a := (\hat{e}_a)^T y$.
At late times, $\lim_{t\to\infty} f_t(x) = 0$ on the training set.
\end{theorem}

The properties of the solution \eqref{eq:fmodes} depend on whether the ratio $\gamma_a/\lambda$ is greater than or smaller than 1, as illustrated in Figure~\ref{fig:l2_mse_modes}.
When $\gamma_a / \lambda > 1$, the function approaches the label mode $y_{\rm mode} = y_a$ at a time that is of order $1/\gamma_a$.
This behavior is the same as that of a linear model, and represents ordinary learning.
Later, at a time of order $\lambda^{-1}$ the mode decays to zero as described above; this late time decay is not present in the linear model.
Next, when $\gamma_a / \lambda < 1$ the mode decays to zero at a time of order $\lambda^{-1}$, which is the same behavior as that of a linear model.

\begin{figure}
    \centering
    \subfloat[][2-layer network]{
    \includegraphics[width=0.33\textwidth]{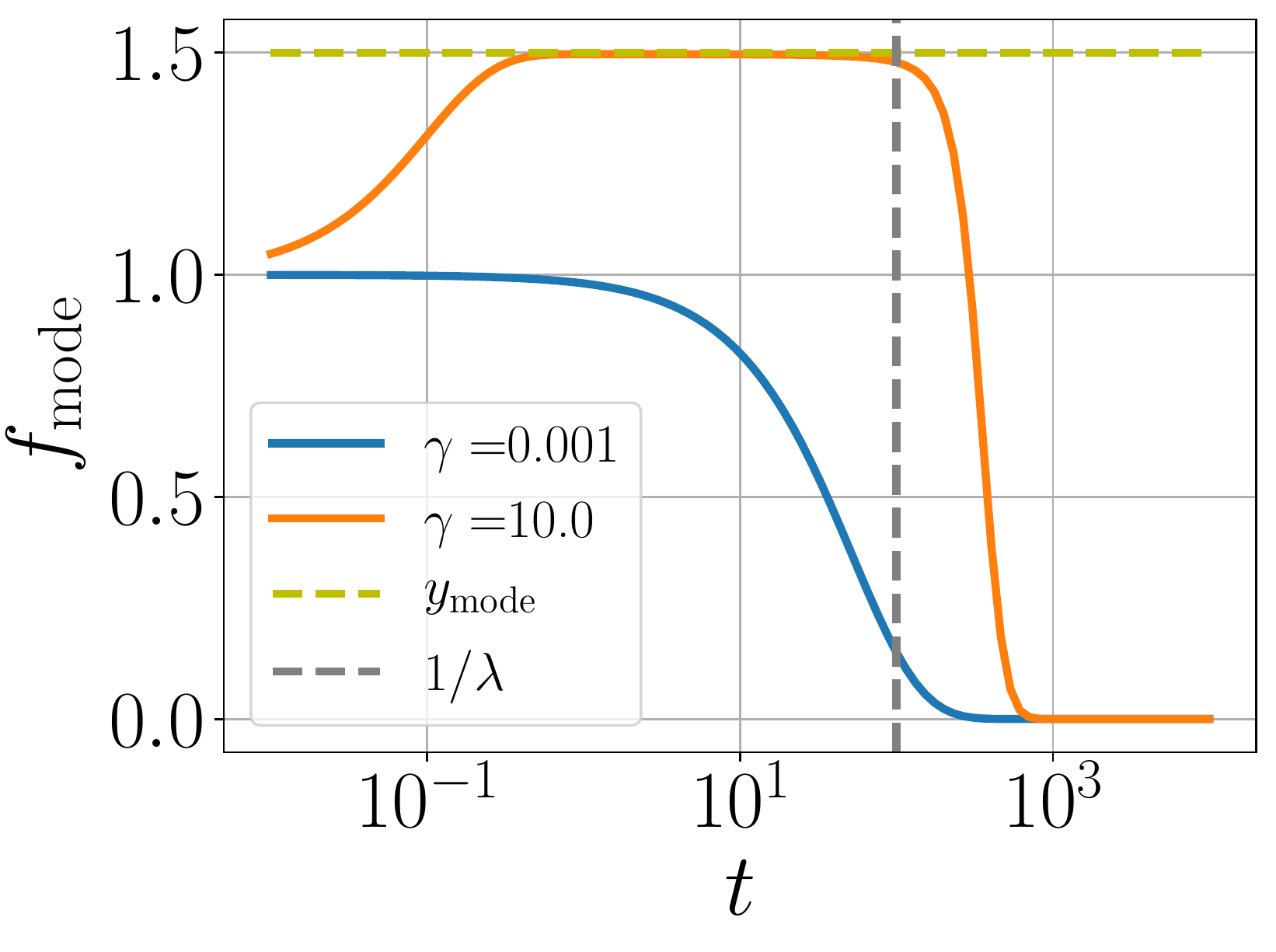}
    }
    \subfloat[][linear model]{
    \includegraphics[width=0.33\textwidth]{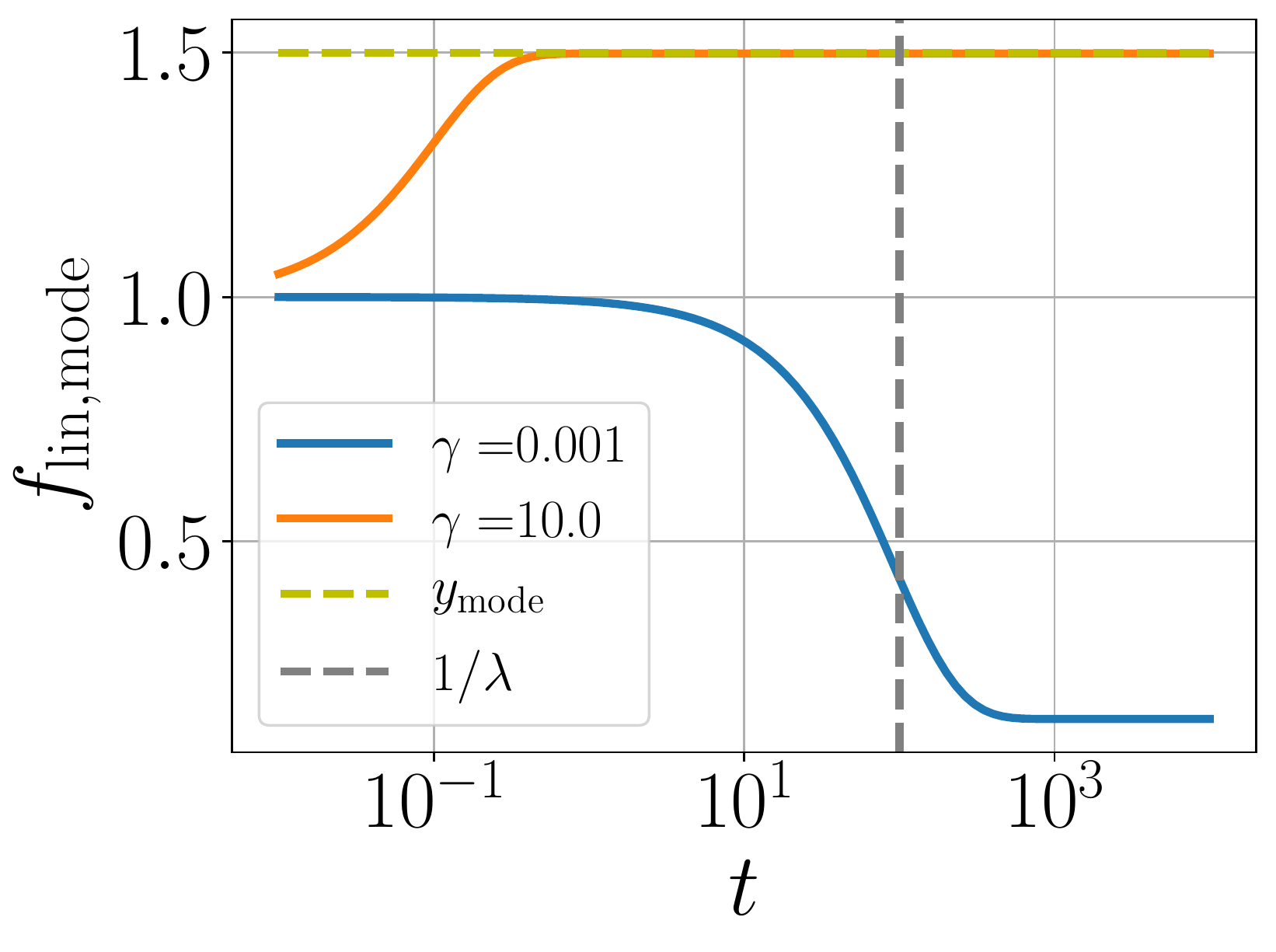}
    }
    \subfloat[][losses]{
    \includegraphics[width=0.33\textwidth]{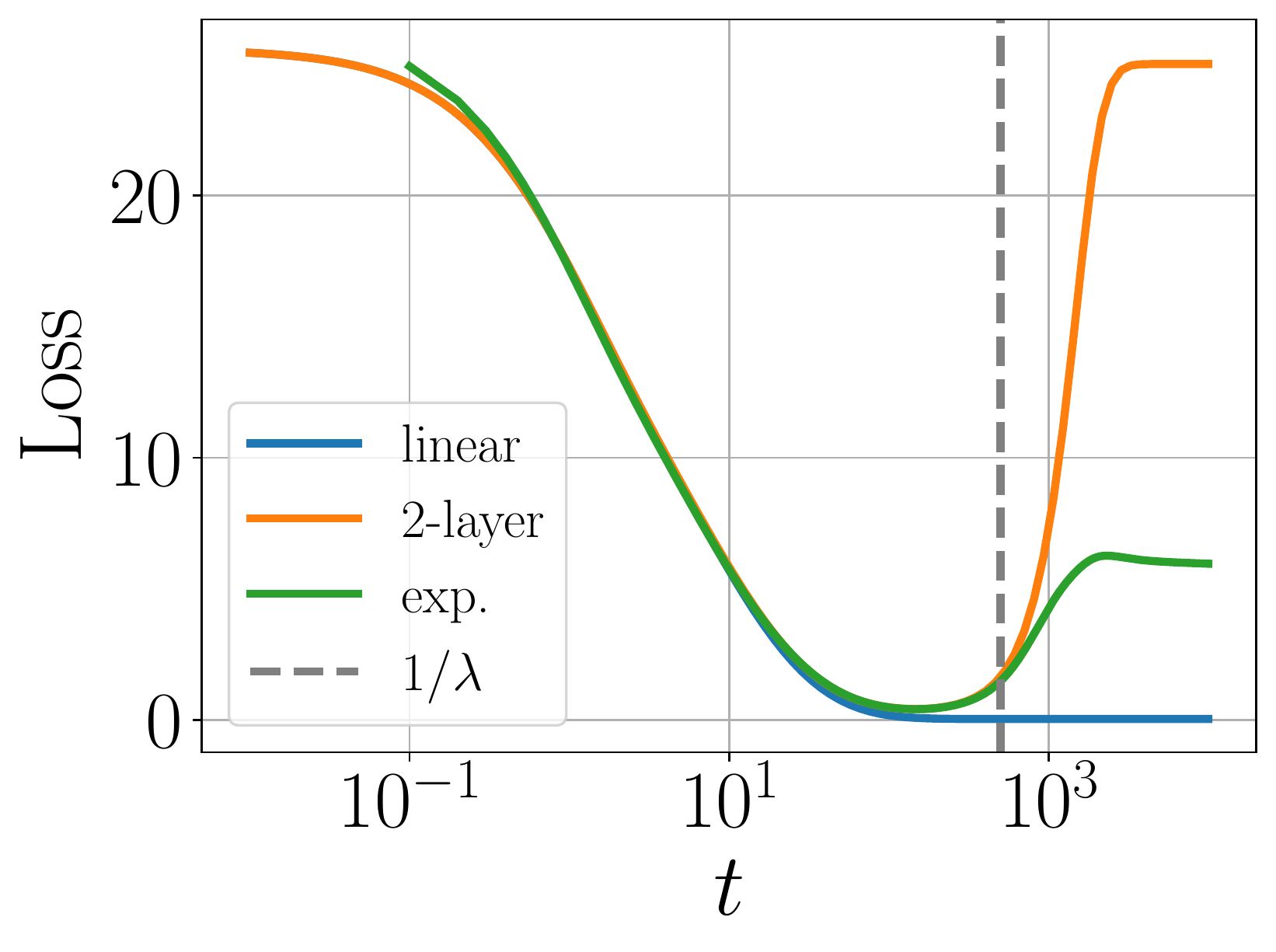}
    \label{fig:thloss}
    }
    \vspace*{-2mm}
    \caption{
    (a) The theoretical evolution of an infinitely wide 2-layer network with $L_2$ regularization ($k=2$, $\lambda=0.01$).
    Two modes are shown, representing small and large ratios $\gamma/\lambda$. 
    (b) The same, for a linear model ($k=1$).
    (c) Training loss vs. time for a wide network trained on a subset of MNIST with even/odd labels, with $\lambda=0.002$. We compare the kernel evolution with gradient descent for a  2-layer ReLU network.
      The blue and orange curves are the theoretical predictions when setting $k=1$ and $k=2$ in the solution \eqref{eq:fmodes}, respectively.
    The green curve is the result of a numerical experiment where we train a 2-layer ReLU network with gradient descent.
    We attribute the difference between the green and orange curves at late times to finite width effects.
 }
    \label{fig:l2_mse_modes}
\end{figure}

\paragraph{Generalization of wide networks with $L_2$.}
It is interesting to understand how $L_2$ regularization affects the generalization performance of wide networks.
This is well understood for the case of linear models, which correspond to $k=1$ in our notation, to be an instance of the bias-variance tradeoff.
In this case, gradient flow converges to the function $f_*(x) = \Theta(x,X) (\Theta + \lambda I)^{-1}(X,X) Y$, where $X \in \bR^{N_{\rm samp} \times d}$ are the training samples, $Y \in \bR^{N_{\rm samp}}$ is are the labels, and $x \in \bR^d$ is any input.
When $\lambda=0$, the solution is highly sensitive to small perturbations in the inputs that affect the flat modes of the kernel, because the kernel is inverted in the solution.
In other words, the solution has high variance.
Choosing $\lambda > 0$ reduces variance by lifting the low kernel eigenvalues and reducing sensitivity on small perturbations, at the cost of biasing the model parameters toward zero. \ifarxiv While a linear model is the prototypical case of $k=1$, the previous late time solution $f_*(x)$ is valid for any $k=1$ model. In particular, any homogeneous model that has batch-normalization in the pre-logit layer will satisfy this property. It would be interesting to understand the generalization properties of such these models based on these solutions. 
\fi

Let us now return to infinitely wide networks.
These behave like linear models with a fixed kernel when $\lambda=0$, but as we have seen when $\lambda > 0$ the kernel decays exponentially.
Nevertheless, we argue that this decay is slow enough such that the training dynamics follow that of the linear model (obtained by setting $k=1$ in eq. \eqref{eq:mainthm}) up until a time of order $\lambda^{-1}$, when the function begins decaying to zero.
This can be seen in Figure~\ref{fig:thloss}, which compares the training curves of a linear and a 2-layer network using the same kernel.
We see that the agreement extends until the linear model is almost fully trained, at which point the 2-layer model begins deteriorating due to the late time decay.
Therefore, if we stop training the 2-layer network at the loss minimum, we end up with a trained and regularized model. It would be interesting to understand how the generalization properties of this model with decaying kernel differ from those of the linear model.

\paragraph{Finite-width network.}

Theorem \ref{thm1} holds in the strict large width, fixed $\lambda$ limit for NTK parameterization. 
At large but finite width we expect \eqref{eq:mainthm} to be a good description of the training trajectory at early times, until the kernel and function because small enough such that the finite-width corrections become non-negligible.
Our experimental results imply that this approximation remains good until after the minimum in the loss, but that at late times the function will not decay to zero; see for example Figure~\ref{fig:thloss}.
See the SM for further discussion for the case of deep linear models.
We reserve a more careful study of these finite width effects to future work.  

\section{Discussion}
In this work we consider the effect of $L_2$ regularization on overparameterized networks.
We make two empirical observations: (1) The time it takes the network to reach peak performance is proportional to $\lambda$, the $L_2$ regularization parameter, and (2) the performance reached in this way is independent of $\lambda$ when $\lambda$ is not too large. 
We find that these observations hold for a variety of overparameterized training setups; see the SM for some examples where they do not hold. We expect the peak performance to depend on $\lambda$ and $\eta$, but not on other quantities such as the initialization scale. 
We verify this empirically in SM \ref{sec:different_init}.

Motivated by these observations, we suggest two practical applications.
The first is a simple method for predicting the optimal $L_2$ parameter at a given training budget.
The performance obtained using this prediction is close to that of a tuned $L_2$ parameter, at a fraction of the training cost.
The second is \AutoLtwo, an automatic $L_2$ parameter schedule.
In our experiments, this method leads to better performance and faster training when compared against training with a tuned $L_2$ parameter.
We find that these proposals work well when training with a constant learning rate; we leave an extension of these methods to networks trained with learning rate schedules to future work.

We attempt to understand the empirical observations by analyzing the training trajectory of infinitely wide networks trained with $L_2$ regularization.
We derive the differential equations governing this trajectory, and solve them explicitly for MSE loss.
The solution reproduces the observation that the time to peak performance is of order $\lambda^{-1}$.
This is due to an effect that is specific to deep networks, and is not present in linear models: during training, the kernel (which is constant for linear models) decays exponentially due to the $L_2$ term.

\newcommand{\acktext}{
The authors would like to thank Yasaman Bahri, Ethan Dyer, Jaehoon Lee, Behnam Neyshabur, and Sam Schoenholz for useful discussions. We especially thank Behnam for encouraging us to use our scaling law observations to come up with a schedule for the $L_2$ parameter.
\ifarxiv
\else
The authors have nothing to disclose.
\fi
}

\ifarxiv
\section*{Acknowledgments}
\acktext
\else
\begin{ack}
\acktext
\end{ack}
\fi

\ifarxiv
\else
\section*{Broader Impact}
This work does not present any foreseeable societal consequence.
\fi

\bibliography{references}

\newpage
\setcounter{equation}{0}
\setcounter{figure}{0}
\setcounter{table}{0}
\setcounter{page}{1}
\setcounter{section}{0}

\renewcommand{\theequation}{S\arabic{equation}}
\renewcommand{\thefigure}{S\arabic{figure}}
\renewcommand{\thetable}{S\arabic{table}}

\section*{Supplementary material}
\appendix
\section{Experimental details}
\label{sec:expdetails}
We are using JAX \citep{jax2018github}. 

All the models except for section \ref{sec:mse} have been trained with Softmax loss normalized as ${\cal L}(\lbrace x,y \rbrace_B) =\frac{1}{2 k |B|}\sum_{(x,y) \in B,i} y_i\log p_i(x), p_i(x)=\frac{e^{f^i(x)}}{\sum_j e^{f^j(x)}}$, where $k$ is the number of classes and $y^i$ are one-hot targets. 

All experiments that compare different learning rates and  $L_2$ parameters use the same seed for the weights at initialization and we consider only one such initialization (unless otherwise stated) although we have not seen much variance in the phenomena described. We will be using standard normalization with LeCun initialization $W \sim {\cal N}(0,\frac{\sigma_w^2}{N_{in}}),b \sim {\cal N}(0,\sigma_b^2)$. 

Batch Norm: we are using  JAX's Stax implementation of Batch Norm which doesn't keep track of training batch statistics for test mode evaluation. 

Data augmentation: denotes flip, crop and mixup.

We consider 3 different networks:
\begin{itemize}
    \item WRN: Wide Resnet 28-10 \citep{WRN}  with has batch-normalization and batch size $1024$ (per device batch size of $128$), $\sigma_w=1,\sigma_b=0$. Trained on CIFAR-10.
    \item FC: Fully connected, three hidden layers with width $2048$ and ReLU activation and batch size $512$,$\sigma_w=\sqrt{2},\sigma_b=0$. Trained on 512 samples of MNIST.
    \item CNN: We use the following architecture: 
$\text{Conv}_{1}(300) \rightarrow \text{Act} \rightarrow \text{Conv}_{2}(300) \rightarrow \text{Act} \rightarrow \text{MaxPool((6,6), 'VALID')} \rightarrow \text{Conv}_{1}(300) \rightarrow \text{Act} \rightarrow \text{Conv}_{2}(300) \rightarrow \text{MaxPool((6,6), 'VALID')} \rightarrow \text{Flatten()} \rightarrow \text{Dense}(500) \rightarrow \text{Dense}(10)$. $\text{Dense}(n)$ denotes a fully-connected layer with output dimension $n$. $\text{Conv}_{1}(n), \text{Conv}_{2}(n)$ denote convolutional layers with 'SAME' or 'VALID' padding and $n$ filters, respectively; all convolutional layers use $(3,3)$ filters.  MaxPool((2,2), 'VALID') performs max pooling with 'VALID' padding and a (2,2) window size. Act denotes the activation: `(Batch-Norm $\rightarrow$) ReLU ' depending on whether we use Batch-Normalization or not. We use batch size $128$, $\sigma_w=\sqrt{2},\sigma_b=0$. Trained on CIFAR-10 without data augmentation.
\end{itemize}

The WRN experiments are run on v3-8 TPUs and the rest on P100 GPUs. 

Here we describe the particularities of each figure. Whenever we report performance for a given time budget, we report the maximum performance during training which does not have to happen at the end of training.

\textbf{Figure \ref{fig:fig1a} } WRN trained using momentum$=0.9$, data augmentation and a learning rate schedule where $\eta(t=0)=0.2$ and then decays $\eta \rightarrow 0.2 \eta $ at $\lbrace 0.3\cdot T,0.6\cdot T,0.9 \cdot T \rbrace$, where $T$ is the number of epochs. 
We compare training with a fixed $T=200$ training budget, against training with $T(\lambda) = 0.1/\lambda$.
This was chosen so that $T(0.0005)=200$.

\textbf{Figures \ref{fig:fig1b}, \ref{fig:optimalL2}, \ref{fig:optimalL22c}.} WRN trained using momentum$=0.9$, data augmentation and $\eta=0.2$ for $\lambda \in (5\cdot 10^{-6} , 10^{-5}, 5 \cdot 10^{-5}, 0.0001, 0.0002, 0.0004, 0.001, 0.002)$. The predicted $\lambda$ performance of \ref{fig:fig1b} was computed at $\lambda=0.0066/T \in( 0.000131,6.56\cdot 10^{-5},2.63\cdot 10^{-5},1.31\cdot 10^{-5},6.56\cdot 10^{-6},4.38\cdot 10^{-6},3.28\cdot 10^{-6}
) $ for $T \in (50,100,250,500,1000,1500,2000)$ respectively.

\textbf{Figures \ref{fig:fig1c},\ref{fig:L2schedule}.} WRN trained using momentum$=0.9$, data augmentation and $\eta=0.2$, evolved for $200$ epochs. The \AutoLtwo~ algorithm is written explicitly in SM \ref{sec:AutoL2} and make measurements every 10 steps.

\textbf{Figure \ref{fig:timedynamics}a,b,c.} FC trained using SGD  $\frac{2}{\eta \lambda}$ epochs with learning rate and $L_2$ regularizations  $\eta \in (0.0025, 0.01, 0.02, 0.025, 0.03, 0.05, 0.08, 0.15, 0.3, 0.5, 1, 1.5, 2, 5, 10, 25, 50)$, $\lambda \in (
0, 10^{-5}, 0.0001, 0.0005, 0.001, 0.005, 0.01, 0.05, 0.1, 0.5, 1, 5, 10, 20, 50, 100)$. The  $\lambda=0$ model was evolved for $10^6/\eta$ epochs which is more than the smallest $\lambda$.

\textbf{Figure \ref{fig:timedynamics}d,e,f.} WRN trained using SGD without data augmentation for $\frac{0.1}{\eta \lambda}$ epochs for the following hyperparameters $\eta \in (0.0125, 0.025,  0.05,   0.1,    0.2,    0.4,    0.8,    1.6  ), \lambda \in (0, 1.5625 \cdot 10^{-5}, 6.25 \cdot 10^{-5}, 1.25 \cdot 10^{-4}, 2.5 \cdot 10^{-4}, 5 \cdot 10^{-4}, 10^{-3}, 2 \cdot 10^{-3}, 4 \cdot 10^{-3}, 8 \cdot 10^{-3}, 0.016)$, as long as the total number of epochs was $\le 4000$ epochs (except for $\eta=0.2,\lambda=6.25 \cdot 10^{-5}$ which was evolved for $8000$ epochs). We evolved the $\lambda=0$ models for $10000$ epochs.  

\textbf{Figure \ref{fig:cifarfc}.} Fully connected depth $3$ and width $64$ trained on CIFAR-10 with batch size $512$, $\eta=0.1$ and cross-entropy loss.

\textbf{Figure \ref{fig:imagenet}.} ResNet-50 trained on ImageNet with batch size 8192, using the implementation in \texttt{https://github.com/tensorflow/tpu}.

  \textbf{Figure \ref{fig:l2_mse_modes}}
  (a,b)   plots $f_t$ in equation \ref{eq:fmodes} with $k=2$ (for 2-layer) and $k=1$ (for linear), for different values of $\gamma$ and $\lambda=0.01$. (c) The empirical kernel of a $2-$layer ReLU network of width 5,000 was evaluated on $200$-samples of MNIST with even/odd labels. The linear, $2-$layer curves come from evolving equation \ref{eq:mainthm} with the previous kernel and setting $k=1,k=2$, respectively . The experimental curve comes from training the 2-layer ReLU network with width $10^5$ and learning rate $\eta=0.01$ (the time is $\text{step} \times \eta$).

\textbf{Figure \ref{fig:CNN}a,b,c.} CNN without BN trained using SGD for $\frac{0.01}{\eta \lambda}$ epochs for the following hyperparameters $\eta=0.01, \lambda \in (0, 5 \cdot 10^{-5}, 0.0001, 0.0005, 0.001, 0.01, 0.05, 0.1, 0.25, 0.5, 1, 2 )$. with $\lambda=0$ was evolved for $21000$ epochs.

\textbf{Figure \ref{fig:CNN}d,e,f.} CNN with BN trained using SGD for a time $\frac{0.01}{\eta \lambda}$ for the following hyperparameters $\eta=0.01, \lambda=
0, 5 \cdot 10^{-5}, 0.0001, 0.0005, 0.001, 0.01, 0.05, 0.1, 0.25, 0.5, 1, 2$. The model with $\lambda=0$ was evolved for $9500$ epochs, which goes beyond where all the other $\lambda$'s have peaked.
 
\textbf{Figure \ref{fig:catapult}.} FC trained using SGD and MSE loss for  $\frac{1}{\eta \lambda}$ epochs and the following hyperparameters $\eta \in (0.001, 0.005, 0.01, 0.02, 0.035, 0.05, 0.15, 0.3), \lambda \in (
10^{-5}, 0.0001, 0.0005, 0.001, 0.005, 0.01, 0.05, 0.1, 0.5, 1, 5, 50, 100)$. For $\lambda=0$, it was trained for $10^5/\eta$ epochs.

\textbf{Rest of SM figures.} Small modifications of experiments in previous figures, specified explicitly in captions. 

\section{Details of theoretical results}

In this section we prove the main theoretical results.
We begin with two technical lemmas that apply to $k$-homogeneous network functions, namely network functions $f_\theta(x)$ that obey the equation $f_{a\theta}(x) = a^k f_\theta(x)$ for any input $x$, parameter vector $\theta$, and $a>0$.

\begin{lemma}\label{lemma:ibp}
Let $f_\theta(x)$ be a $k$-homogeneous network function.
Then $\sum_\mu \theta_\mu \partial_\mu \partial_{\nu_1} \cdots \partial_{\nu_m} f(x) = (k-m) \partial_{\nu_1} \cdots \partial_{\nu_m} f(x)$.
\end{lemma}
\begin{proof}
We prove by induction on $m$.
For $m=0$, we differentiate the homogeneity equation with respect to $a$.
\begin{align}
    0 &= \left. \frac{\dho}{\dho a} \right|_{a=1} \left( f_{a\theta}(x) - a^k f_\theta(x) \right)
    = \sum_\mu \frac{\dho f(x)}{\dho \theta_\mu} \theta_\mu - k f_\theta(x) \,.
\end{align}
For $m>0$,
\begin{align}
    \sum_\mu \theta_\mu \partial_\mu \partial_{\nu_1} \cdots \partial_{\nu_m} f(x) &=
    \partial_{\nu_m} \left( \sum_\mu \theta_\mu \partial_{\mu} \partial_{\nu_1} \cdots \partial_{\nu_{m-1}} f(x) \right) 
    - \sum_\mu \left( \partial_{\nu_m} \theta_\mu \right) \partial_{\mu} \partial_{\nu_1} \cdots \partial_{\nu_{m-1}} f(x)
    \cr &= \dho_{\nu_m} (k-m+1) \partial_{\nu_1} \cdots \partial_{\nu_{m-1}} f(x)
    - \sum_\mu \delta_{\mu\nu_m} \partial_{\nu_m} \theta_\mu \partial_{\mu} \partial_{\nu_1} \cdots \partial_{\nu_{m-1}} f(x)
    \cr
    &= (k-m) \partial_{\nu_1} \cdots \partial_{\nu_{m-1}} \dho_{\nu_m} f(x) \,.
\end{align}

\end{proof}

\begin{lemma}\label{lemma:asymp}
Consider a $k$-homogeneous network function $f_\theta(x)$, and a correlation function $C(x_1,\dots,x_m)$ that involves derivative tensors of $f_\theta$.
Let $L = \sum_{x \in S} \ell(x) + \frac{1}{2} \lambda \| \theta \|_2^2$ be a loss function, where $S$ is the training set and $\ell$ is the sample loss.
We train the network using gradient flow on this loss function, where the update rule is $\frac{d\theta^\mu}{dt} = - \frac{dL}{d\theta^\mu}$.
If the conjecture of \citet{feynman-diagrams} holds, and if the conjecture implies that $C = \cO(n^{-1})$ where $n$ is the width, then $\frac{dC}{dt} = \cO(n^{-1})$ as well.
\end{lemma}

\begin{proof}
The cluster graph of $C$ has $m$ vertices; we denote by $n_e$ ($n_o$) the number of even (odd) components in the graph (we refer the reader to \citet{feynman-diagrams} for a definition of the cluster graph and other terminology used in this proof).
By assumption, $n_e + (n_o - m)/2 \le -1$.

We can write the correlation function as $C(x_1,\dots,x_m) = \sum_{\rm indices} \lexpp{\theta} \dho f(x_1) \cdots \dho f(x_m) \rexp$, where $\lexpp{\theta} \cdot \rexp$ is a mean over initializations, $\dho f(x)$ is shorthand for a derivative tensor of the form $\dho_{\mu_1\dots\mu_a} f(x) := \frac{\dho^a f(x)}{\dho \theta^{\nu^1} \cdots \dho \theta^{\nu^a}}$ for some $a$, and the sum is over all the free indices of the derivative tensors.
Then $\frac{dC}{dt} = \sum_{b=1}^m C_b$, where $C_b := \lexpp{\theta} \dho f(x_1) \cdots \frac{d \dho f(x_b)}{dt} \cdots \dho f(x_m) \rexp$.
To bound the asymptotic behavior of $dC/dt$ it is therefore enough to bound the asymptotics of each $C_b$.

Notice that each $C_b$ is obtained from $C$ by replacing a derivative tensor $\dho f$ with $d (\dho f)/dt$ inside the expectation value.
Let us see how this affects the cluster graph.
For any derivative tensor $\dho_{\mu_1\dots\mu_a} f(x) := \dho^a f(x) / \dho \theta^{\nu^1} \cdots \dho \theta^{\nu^a}$, we have
\begin{align}
    \frac{d}{dt} \dho_{\mu_1\dots\mu_a} f(x) &=
    \sum_\nu \dho_{\mu_1\dots\mu_a \nu} f(x) \frac{d\theta^\nu}{dt}
    \cr &= - \sum_\nu \dho_{\mu_1\dots\mu_a \nu} f(x) \left[
    \sum_{x' \in S} \dho_\nu f(x') \ell'(x') + \lambda \theta^\nu
    \right]
   \nonumber \\
    &= - \sum_{\nu,x'} \dho_{\mu_1\dots\mu_a \nu} f(x) \dho_\mu f(x') \ell'(x')
    - (k-a) \lambda \dho_{\mu_1\dots\mu_a} f(x) \,. \label{eq:rule}
\end{align}
In the last step we used lemma~\ref{lemma:ibp}.
We now compute how replacing the derivative tensor $\dho f$ by each of the terms in the last line of \eqref{eq:rule} affects the cluster graph, and specifically the combination $n_e + (n_o - m)/2$.

The second term is equal to the original derivative tensor up to an $n$-independent factor, and therefore does not change the asymptotic behavior.
For the first term, the $\ell'$ factor leaves $n_e$ and $n_o - m$ invariant so it will not affect the asymptotic behavior.
The additional $\dho_\mu f$ factor increases the number of vertices in the cluster graph by 1, namely it changes $m \mapsto m+1$.
In addition, it increases the size of the graph component of $\dho_{\mu_1 \dots \mu_a} f(x)$ by 1, therefore either turning an even sized component into an odd sized one or vice versa.
In terms of the number of components, it means we have $n_e \mapsto n_e \pm 1, n_o \mapsto n_o \mp 1$. Therefore, $n_e + (n_o - m)/2 \mapsto n_e + (n_o - m)/2 \pm 1 \mp \frac{1}{2} - \frac{1}{2} \le n_e + (n_o - m)/2 \le 1$.
Therefore, it follows from the conjecture that $C_b = \cO(n^{-1})$ for all $b$, and then $dC/dt = \cO(n^{-1})$.
\end{proof}

We now turn to the proof of Theorems~\ref{thm1} and \ref{thm:mse}.
\begin{proof}[Proof (Theorem~\ref{thm1})]
A straightforward calculation leads to the following gradient flow equations for the network function and kernel.
\begin{align}
    \frac{d f_t(x)}{dt} &= - \sum_{x'\in S} \Theta_t(x,x') \ell'(x') - k \lambda f_t(x) \,, \label{eq:dfc} \\
    \frac{d \Theta_t(x,x')}{dt} &=
    - 2 (k-1) \lambda \Theta_t(x,x') 
    + T_t(x,x') + T_t(x',x)
    \,, \\
    T_t(x,x') &= - \sum_{x''\in S} \dho_{\mu\nu} f_t(x) \dho_\mu f_t(x') \dho_\nu f_t(x'') \ell'(x'') \,.
\end{align}
Here $\ell' = d\ell/df$.
In deriving these we used the gradient flow update $\frac{d\theta^\mu}{dt} = - \frac{\dho L}{\dho \theta^\mu}$ and Lemma~\ref{lemma:ibp}.
It was shown in \citet{feynman-diagrams} that $\lexpp{\theta} T_0 \rexp = \cO(n^{-1})$.
If then follows from Lemma~\ref{lemma:asymp} that $\lexpp{\theta} \frac{d^m T_0}{dt^m} \rexp = \cO(n^{-1})$ for all $m$, where the expectation value is taken at initialization.
Furthermore, the results of \citet{feynman-diagrams} imply that $\mathrm{Var}\! \left[ \frac{d^m T_0}{dt^m} \right] = \cO(n^{-2})$ and therefore $\frac{d^m T_0}{dt^m} \underset{p} \rightarrow 0$.\footnote{See appendix D in \citet{feynman-diagrams}.}
In the strict infinite width limit we can therefore neglect the $T$ contribution in the following equation, and write
\begin{align}
    \frac{d^m \Theta_0(x,x')}{dt^m} &= [- 2 (k-1) \lambda]^m \Theta_0(x,x') \,,\quad
    m=0,1,\dots\,.
\end{align}
The solution of this set of equations (labelled by $m$) is the same as for the any-time equation $\frac{d}{dt} \Theta_t(x,x') = -2 (k-1) \lambda \Theta_t(x,x')$, and the solution is given by
\begin{align}
    \Theta_t(x,x') = e^{-2 (k-1) \lambda t} \Theta_0(x,x') \,.
\end{align}
\end{proof}

\begin{proof}[Proof (Theorem~\ref{thm:mse})]
The evolution of the kernel eigenvalues, and the fact that its eigenvectors do not evolve, follow immediately from \eqref{eq:mainthm_theta}.
The solution \eqref{eq:fmodes} can be verified directly by plugging it into \eqref{eq:mainthm} after projecting the equation on the eigenvector $\hat{e}_a$.
Finally, the fact that the function decays to zero at late times can be seen from \eqref{eq:fmodes} as follows.
From the assumption $k \ge 2$, notice that $\exp\! \left[ 
    - \frac{\gamma_a(t')}{2(k-1)\lambda} - (k-2) \lambda t'
    \right] \le 1$ when $t' \ge 0$.
    Therefore, we can bound each mode as follows.
\begin{align}
    |f_a(x;t)| &\le
    e^{-k\lambda t} \left[
    e^{
    \frac{(\gamma_a(t)-\gamma_a)}{2(k-1)\lambda} 
    }
    |f_a(x;0)|
    + |\gamma_a y|
    e^{
    \frac{\gamma_a(t)}{2(k-1)\lambda} 
    }
    \int_0^t \! dt' 
    \right] \,.
\end{align}
Therefore, $\lim_{t\to\infty} |f_a(x;t)| = 0$.
\end{proof}

For completeness we now write down the solution \eqref{eq:fmodes} in functional form, for $x \in S$ in the training set.
\begin{align}
    f_t(x) &= e^{-k \lambda t}
    \Bigg\{ 
    \sum_{x' \in S}
    \exp \left[
    \frac{\Theta_t-\Theta_0}{2(k-1)\lambda}
    \right] \! (x,x') \,
    f_0(x')
    \cr &\quad \qquad \quad
    +
    \sum_{x',x'' \in S}
    \int_0^t \! dt' \, 
    e^{- (k-2) \lambda t'}
    \exp \left[ 
    \frac{\Theta_t-\Theta_{t'}}{2(k-1)\lambda} 
    \right]\! (x,x') \, \Theta_0(x',x'') \, y(x'')
    \Bigg\} 
    \,. \cr
    \Theta_t(x,x') &= e^{-2(k-1)\lambda t} \Theta_0(x,x')
    \label{eq:theta_evolve} \,.
\end{align}
Here, $\exp(\cdot)$ is a matrix exponential, and $\Theta_t$ is a matrix of size $N_{\rm samp} \times N_{\rm samp}$.

\subsection{Deep linear fixed point analysis}

\newcommand{\cev}[1]{\reflectbox{\ensuremath{\vec{\reflectbox{\ensuremath{#1}}}}}}

Let's consider a deep linear model $f(x)=\beta W_L....W_0.x$, with $\beta=n^{-L/2}$ for NTK normalization and $\beta=1$ for standard normalization. The gradient descent equation will be:

\begin{eqnarray}
    \Delta W_{a b}^l = -\eta \lambda W_{a b}^l-\eta \beta \vec{W}_{a}^{l+1} \cev{W}_{b \alpha}^{l-1} \sum_{(x,y)\in S} \ell'(x,y) x_{\alpha}
    \end{eqnarray}
where we defined:
\begin{equation}
     \vec{W}^{l} \equiv W^L...W^l , \cev{W}_{\alpha}^{l} \equiv W^l...W_{\alpha}^0
\end{equation}
Evolution will stop when the fixed point ( $\Delta W=0$) is reached:
    \begin{eqnarray}
         W^{L>l>0}_{a b}= \vec{W}^{l+1}_{a} \hat{W}^{l-1}_{b}; W^{L}_a=\hat{W}_a^{L-1}; W^{0}_{a \alpha}=\vec{W}_a^{1} \hat{W}^{-1}_{\alpha}
         \\
          \hat{W}^{l-1}_{b} \equiv -\frac{\beta}{\lambda} \cev{{W}}^{l-1}_{b \alpha}\sum_{(x,y)\in S} \ell'(x,y) x_{\alpha} ; \cev{W}_{a \alpha}^{-1}=\delta_{a \alpha}
  \end{eqnarray}
  
 Furthermore note that:
  \begin{eqnarray}
  \vec{W}^l.\hat{W}^{l-1}=-\frac{1}{\lambda}\sum_{(x,y)\in S} \ell'(x,y) f(x) =\tilde{f}
\end{eqnarray}

Now, we would like to show that, at the fixed point
\begin{equation}
        \vec{W}^l_a=\tilde{f}^{L-l} \hat{W}^{l-1}_a
\end{equation}

This follows from induction:
\begin{eqnarray}
\vec{W}^L=W^L=\hat{W}^{L-1} \\
\vec{W}^l=\vec{W}^{l+1} (\vec{W}^{l+1} \hat{W}^{l-1})=\tilde{f}^{L-l-1} \hat{W}^{l}  (\vec{W}^{l+1} \hat{W}^{l-1})=\tilde{f}^{L-l} \hat{W}^{l-1}
\end{eqnarray}
Which has a trivial solution if $\tilde{f}=0$. Let's assume that it is non-trivial. If we contract the previous equation with $\hat{W}^{l-1}$ we get:
\begin{equation}
    \tilde{f}=\tilde{f}^{L-l} ||\hat{W}^{l-1}||^{2} 
    \end{equation}
    We can finally set $l=0$ and simplify:
    \begin{equation}
    \frac{\lambda^{L+1}}{\beta^2}
    =[- \sum_{(x,y)\in S} \ell'(x,y) f(x)]^{L-1}   \sum_{(x,y),(x',y')\in S} \ell'(x,y) \ell'(x',y') x.x'
\end{equation}

At large $n$, to obtain a non-trivial fixed point $f(x)$ should be finite as $n \rightarrow \infty$. From the previous equation, this implies that $\frac{\lambda^{L+1}}{\beta^2}= \theta(n^{0})$. In NTK normalization $\beta^2=n^{-L}$, for $\lambda = \theta(n^{-\frac{L}{L+1}})$, we will get a non-trivial ($f(x) \not = 0$) fixed point. This also implies that these corrections will be important for $\lambda = \theta(n^0)$ in standard normalization since there $\beta=1$. Note that if $\frac{\lambda^{L+1}}{\beta^2} \not = \Omega(n^0)$, we expect that we get the $\lambda=0$ solution $\ell'(x,y)=0$.

We can be very explicit if we consider $L=1$ and one sample with $x,y=1,1$ for MSE loss. The fixed point has a logit:
\begin{equation}
    \lambda \sqrt{n}=f-1
\end{equation}
which is only different from $0,1$ for fixed $\lambda^2 n$. 

\newpage
\section{More on experiments}

\subsection{Training accuracy = 1 scale}

We can see how the time it takes to rech training accuracy $1$ depends very mildly on $\lambda$, and for small enough learning rates it scales like $1/\eta$.

\begin{figure}[ht!]
\centering
\subfloat[FC]{
  \includegraphics[width=0.33\textwidth]{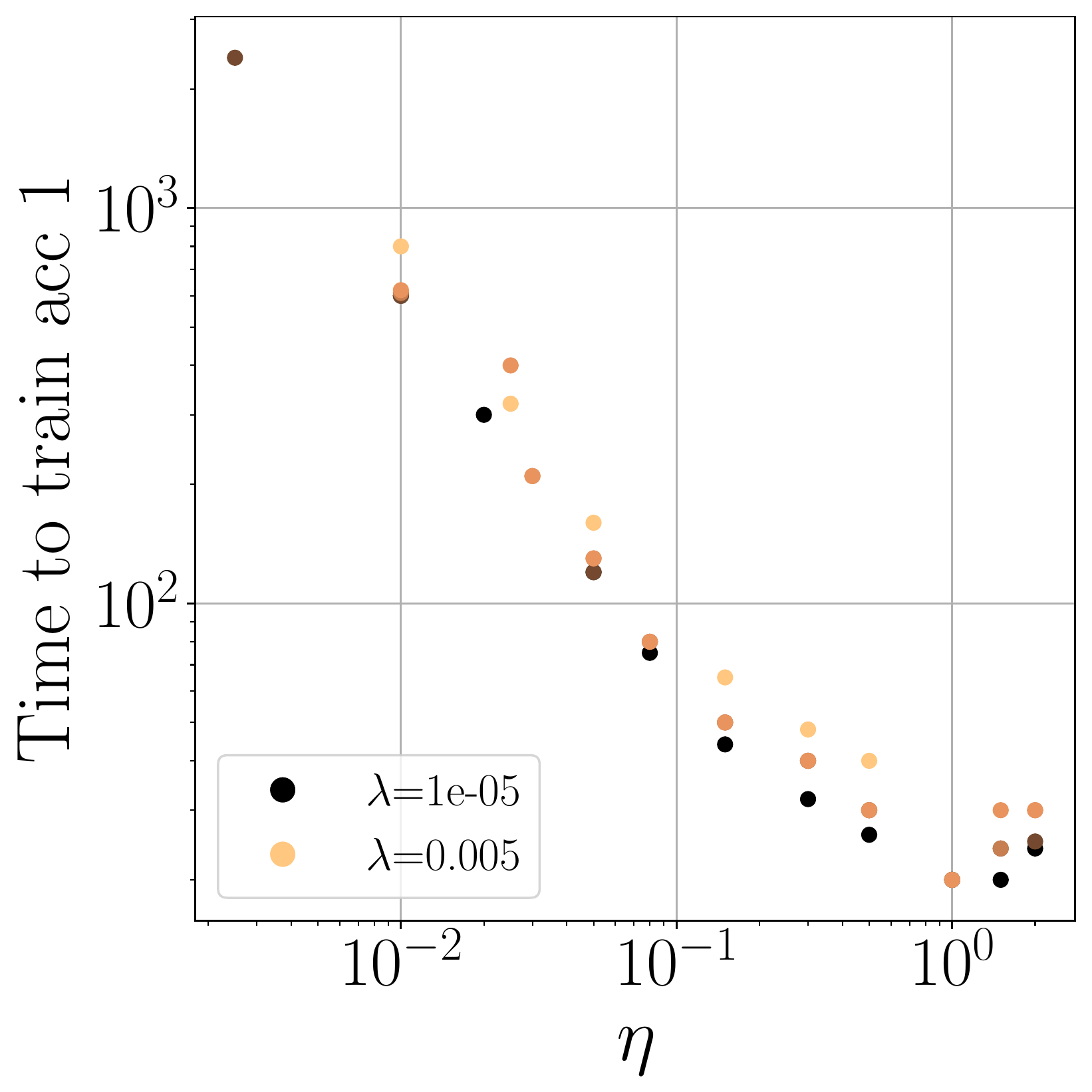}
} 
\subfloat[WRN]{
  \includegraphics[width=0.33\textwidth]{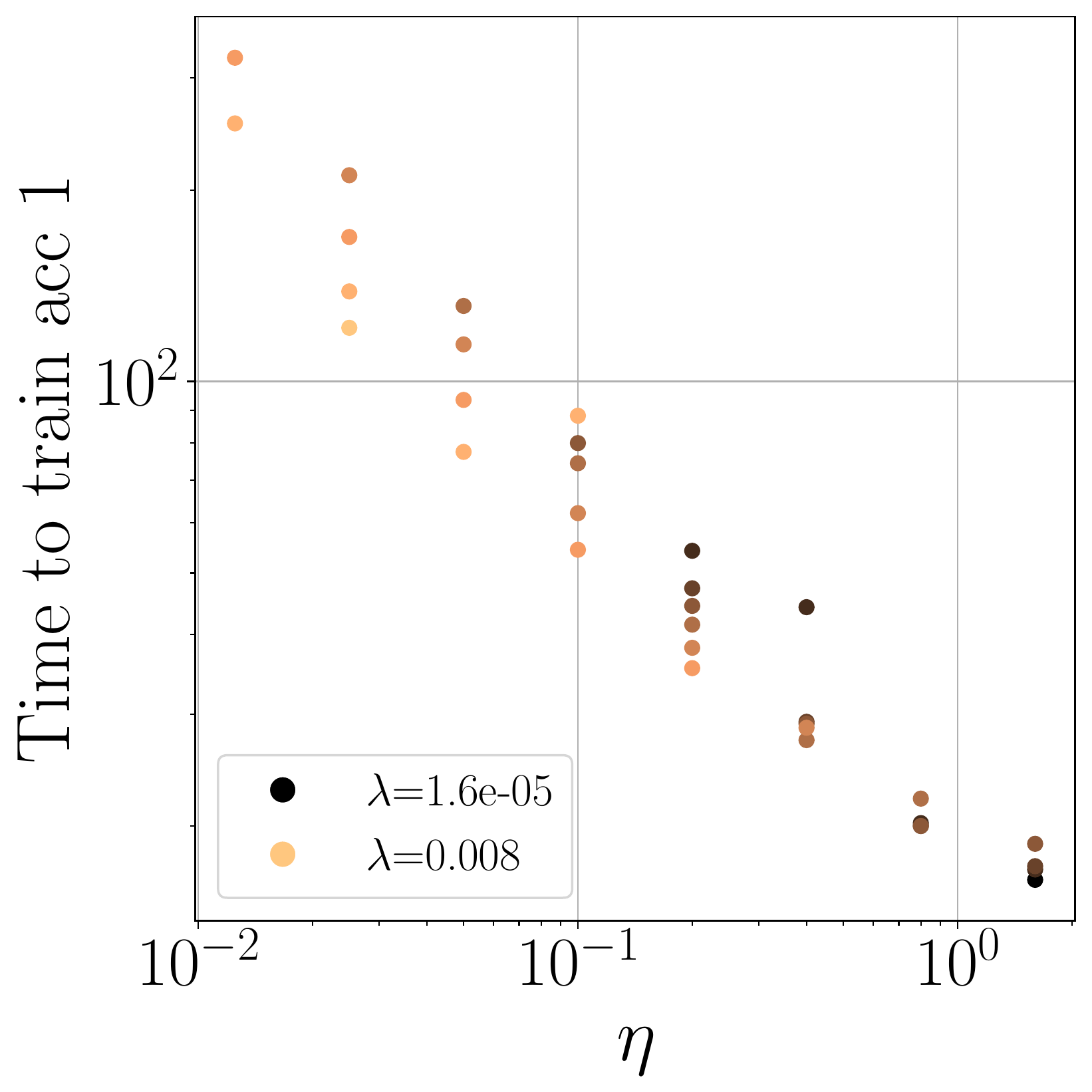}
} 
    \caption{Training accuracy vs learning rate the setup of figure \ref{fig:timedynamics}. The specific values for the $\eta,\lambda$ sweeps are in \ref{sec:expdetails}. }
    
\end{figure}

\subsection{ More WRN experiments}
\label{sec:morewrn}
We can also study the previous in the presence of momentum and data augmentation. These are the experiments that we used in figure \ref{fig:optimalL2}, evolved until convergence. As discussed before, in the presence of momentum the $t_{*}$ depends on $\eta$, so we will fixed the learning rate $\eta=0.2$.

\begin{figure}[H]

  \subfloat[]
  {
    \includegraphics[width=0.33 \textwidth]{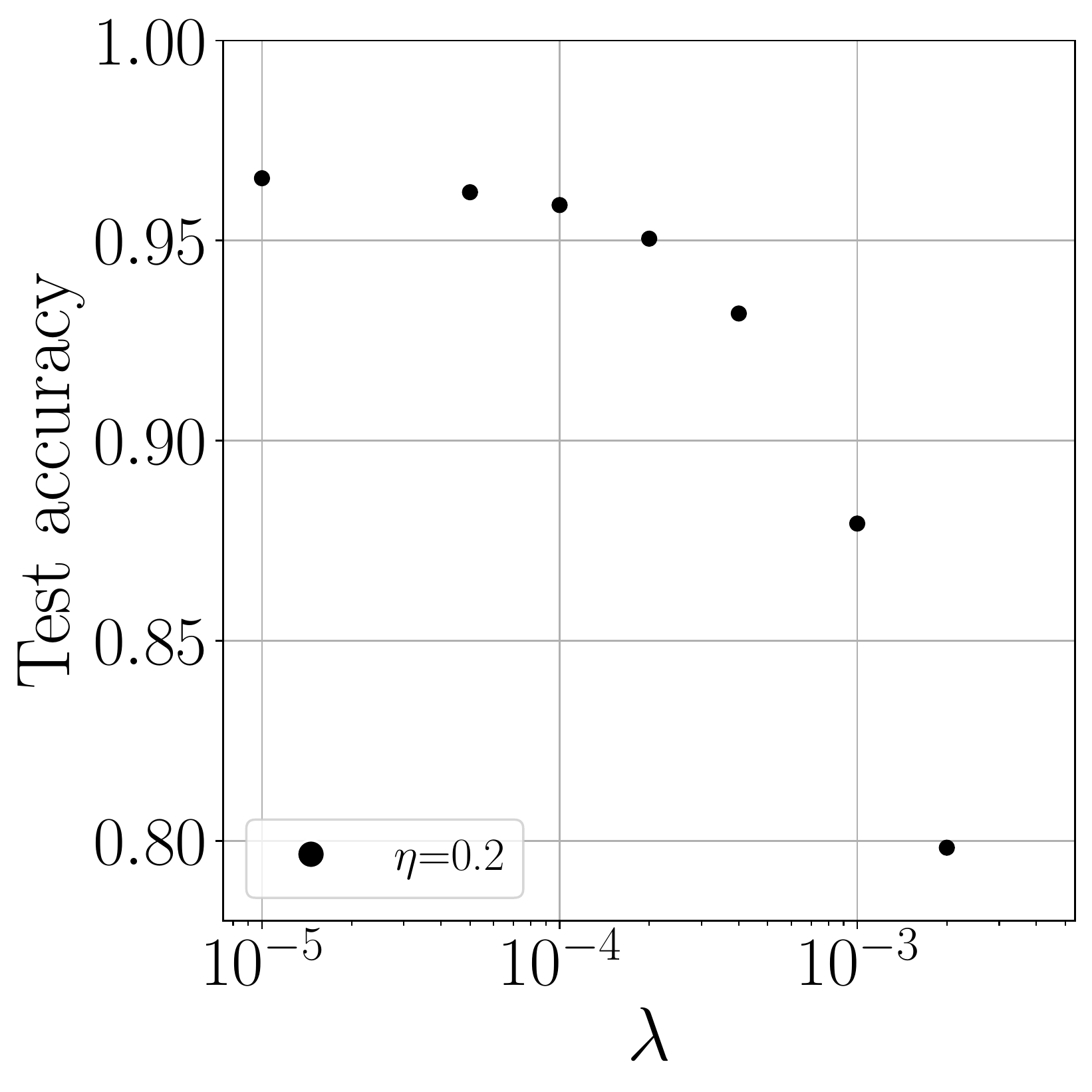}
     }
  \subfloat[]
  {
    \includegraphics[width=0.33 \textwidth]{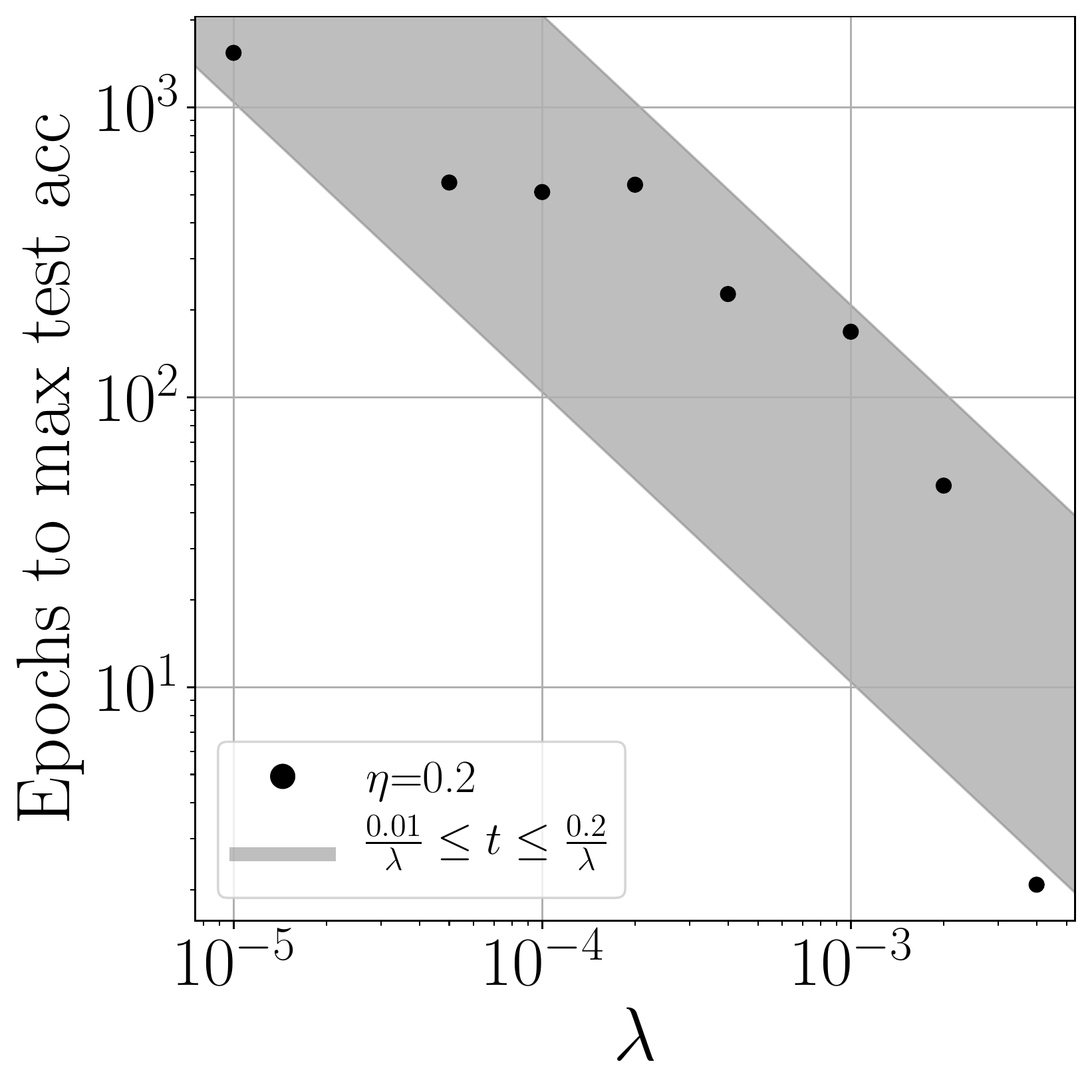}
     }
   \subfloat[]
   {
    \includegraphics[width=0.33 \textwidth]{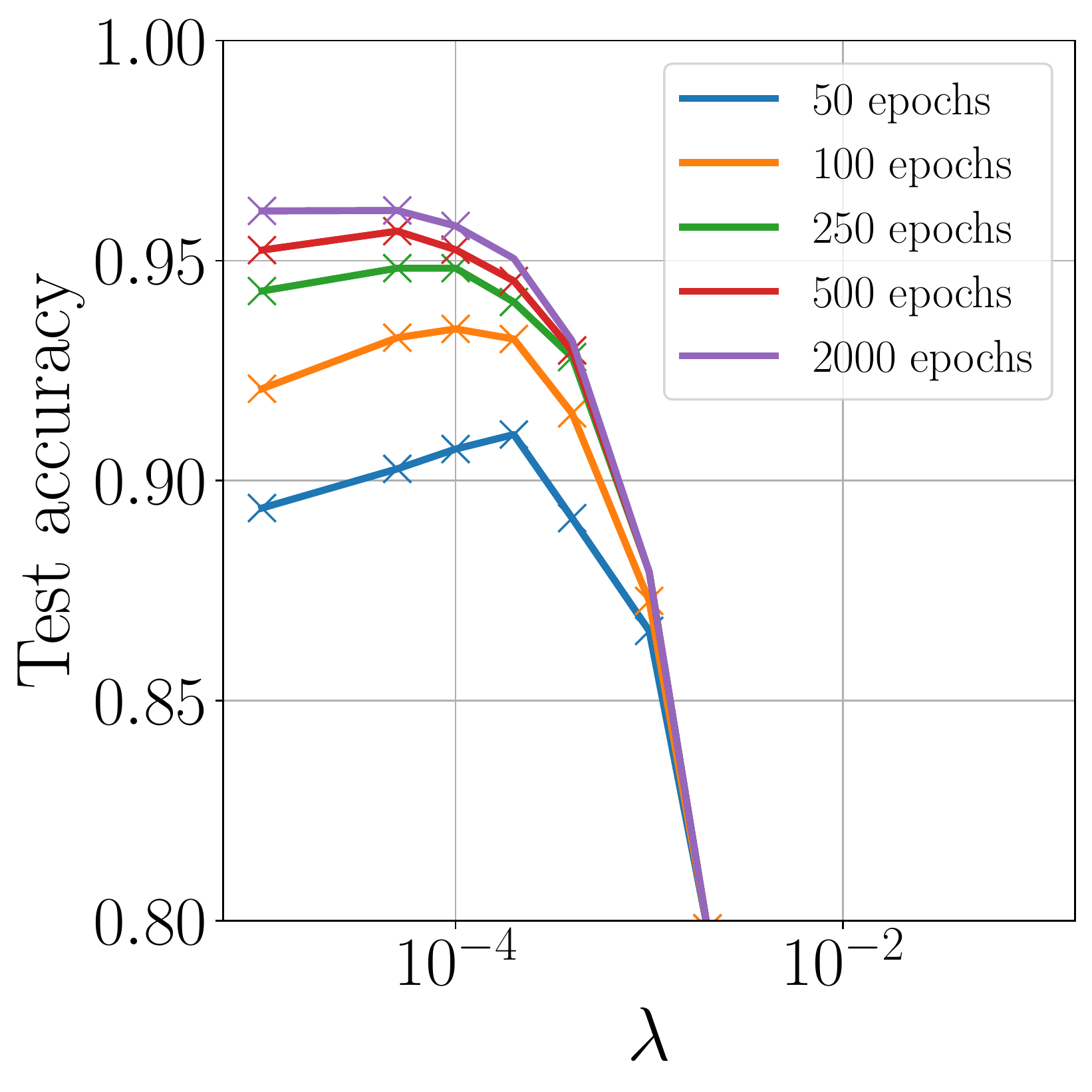}
     }
    \caption{WRN 28-10 with momentum and data augmentation trained with a fixed learning rate.} 
    
\end{figure}

\subsection{More on optimal $L_2$}
\label{sec:moreoptL2}
Here we give more details about the  optimal $L_2$ prediction of section \ref{sec:L2opt}. Figure \ref{fig:optimalL2SM} illustrates how performance changes as a function of $\lambda$ for different time budgets with the predicted $\lambda$ marked with a dashed line.
If one wanted to be more precise, from figure \ref{fig:timedynamics} we see that while the scaling works across $\lambda$'s, generally lower $\lambda$'s have a scaling $\sim 2$ times higher than the larger $\lambda$'s. One could try to get a more precise prediction by multiplying $c$ by two, $c_{\text{small} \lambda}\sim 2 c_{\text{large} \lambda}$, see figure \ref{fig:optimalL22c}. We reserve a more detailed analysis of this more fine-grained prescription for the future. 
\begin{figure}[ht!]
  \centering
  \subfloat{\includegraphics[width=0.33 \textwidth]{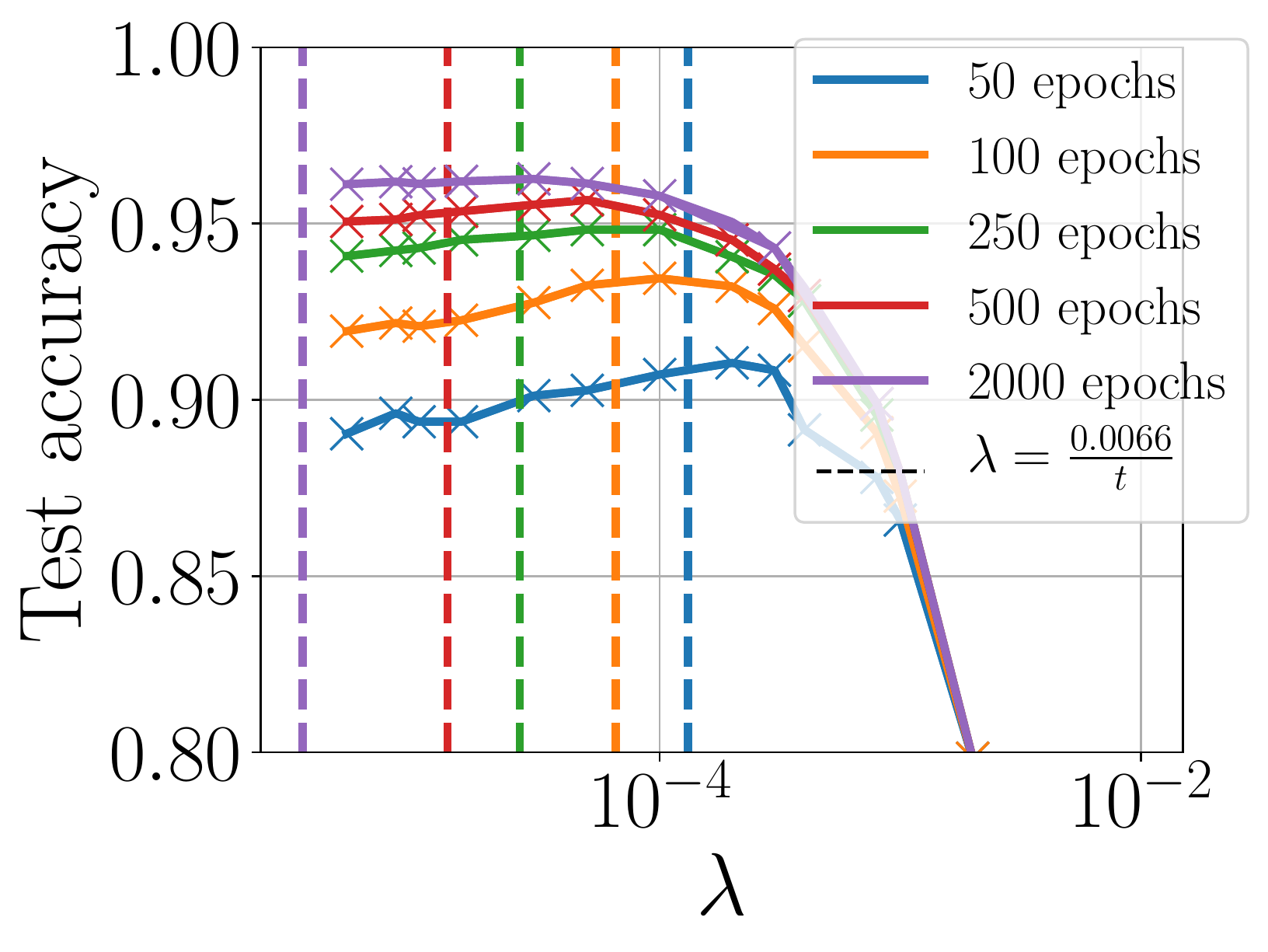} }
\caption{WRN trained with momentum and data augmentation.  Given a number of epochs, we compare the maximum test accuracy as a function of $L_2$ and compare it with the smallest $L_2$ with the predicted one. We see that this gives us the optimal $\lambda$ within an order of magnitude.}
\label{fig:optimalL2SM}
\end{figure}

\begin{figure}[H]
  \centering
    \subfloat[]{\includegraphics[width=0.33 \textwidth]{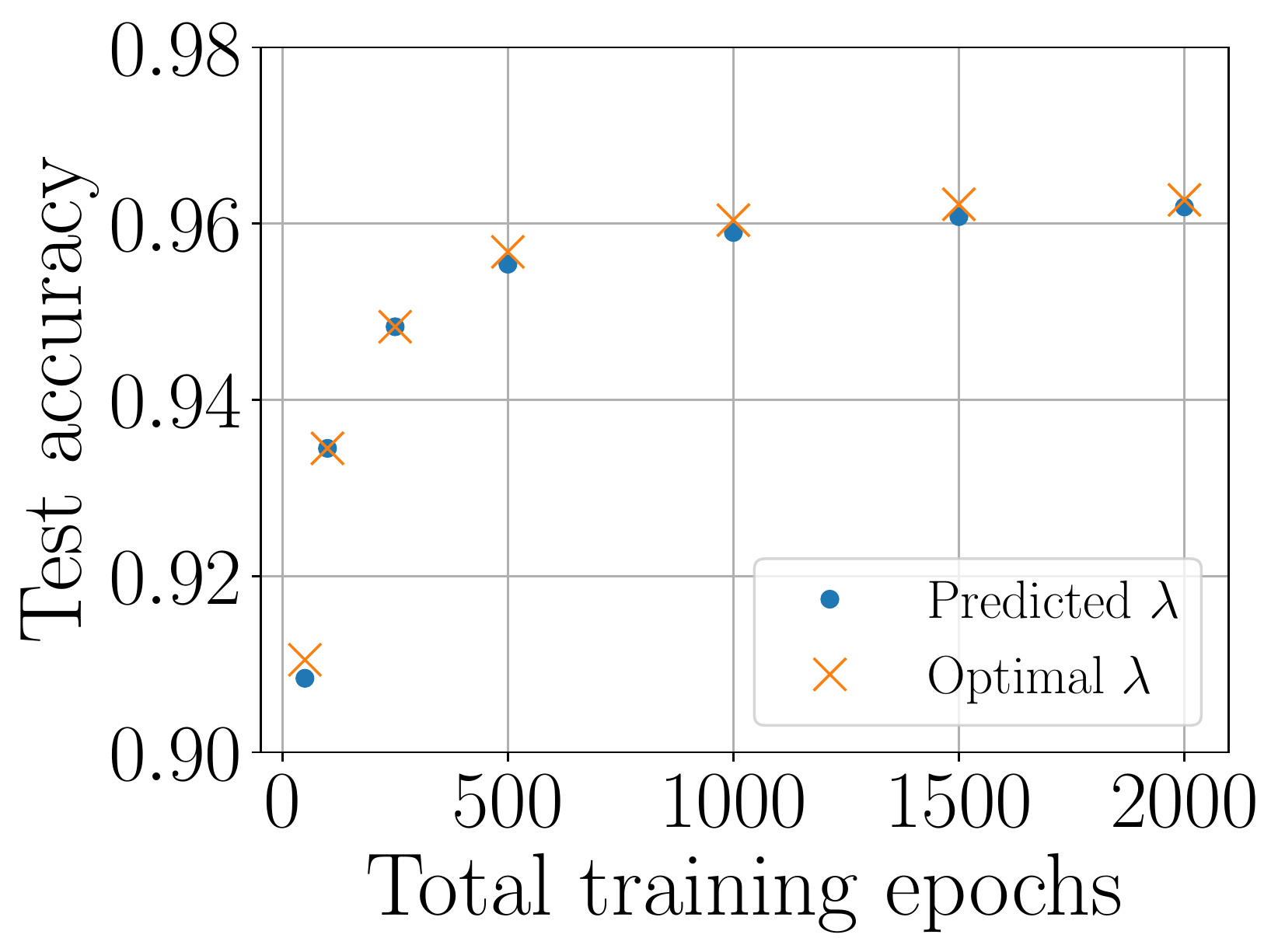}  }
  \subfloat[]{\includegraphics[width=0.33 \textwidth]{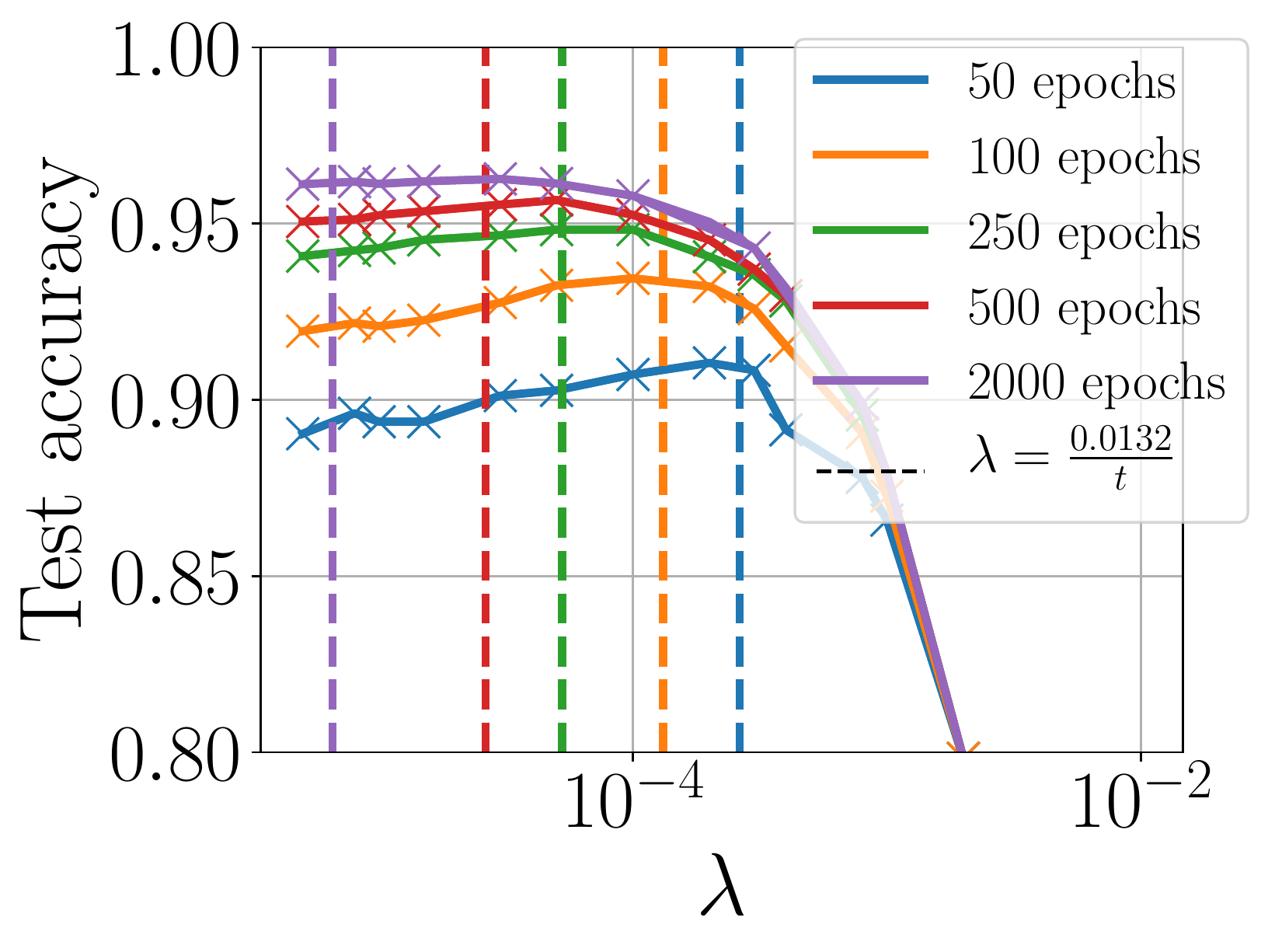} }
  \subfloat[]{\includegraphics[width=0.33\textwidth]{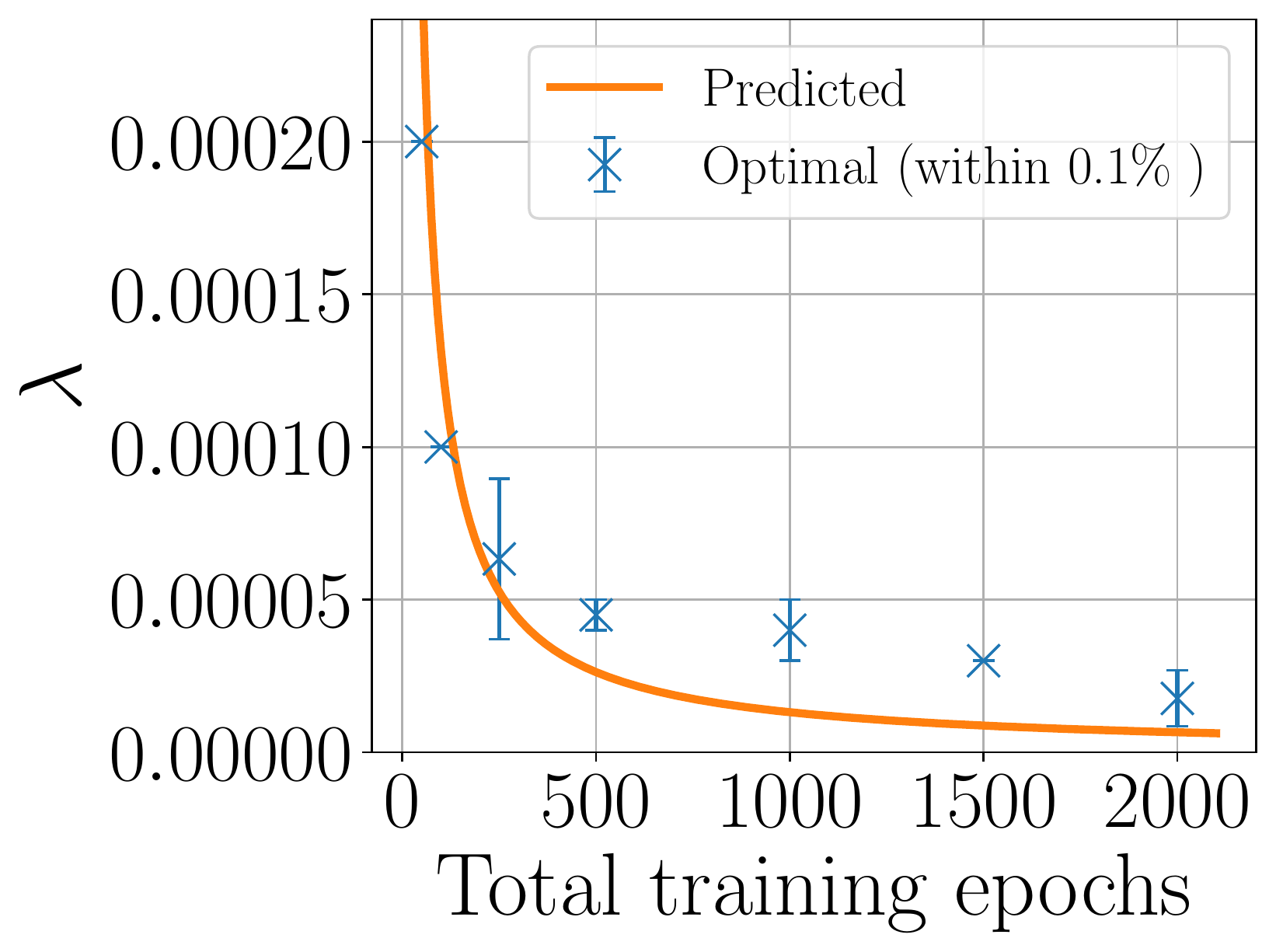}
 }
\caption{Same as previous figure with $c=2 c_{\text{large} \lambda}=0.0132$.
}
\label{fig:optimalL22c}
\end{figure}

\subsection{MSE and the catapult effect}
\label{sec:mse}
In \citet{lewkowycz2020large} it was argued that, in the absence of $L_2$ when training a network with SGD and MSE loss, high learning rates have a rather different final accuracy, due to the fact that at early times they undergo the "catapult effect". However, this seems to contradict with our story around \ref{sec:our_contr} where we argue that performance doesn't depend strongly on $\eta$. In figure \ref{fig:catapult}, we can see how, while when stopped at training accuracy $1$, performance depends strongly on the learning rate, this is no longer the case in the presence of $L_2$ if we evolve it for $t_{test}$. We also show how the training MSE loss has a minimum after which it increases.

\begin{figure}[]
\centering
  \subfloat[]
  {
    \includegraphics[width=0.33 \textwidth]{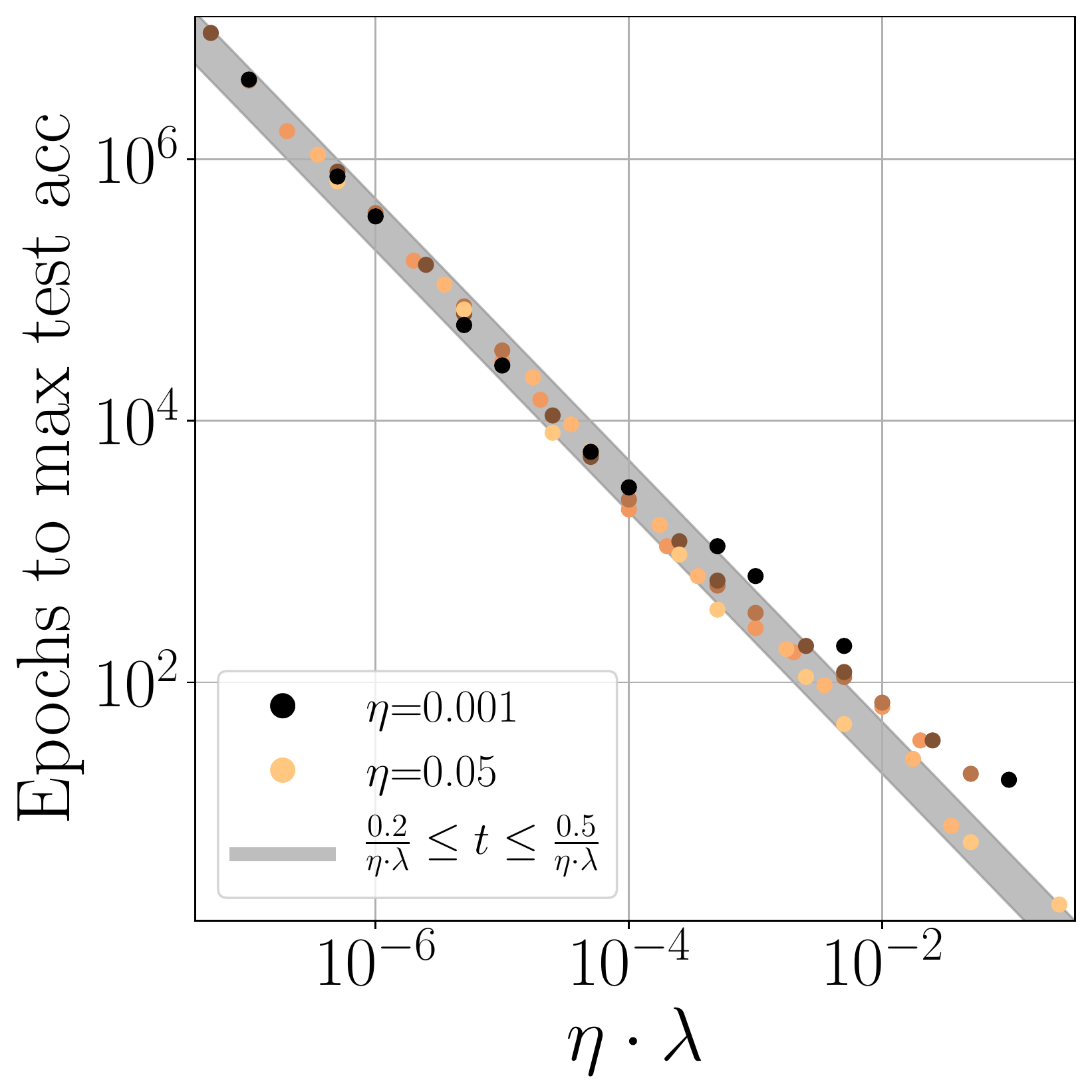}
     }
  \subfloat[]
  {
    \includegraphics[width=0.33 \textwidth]{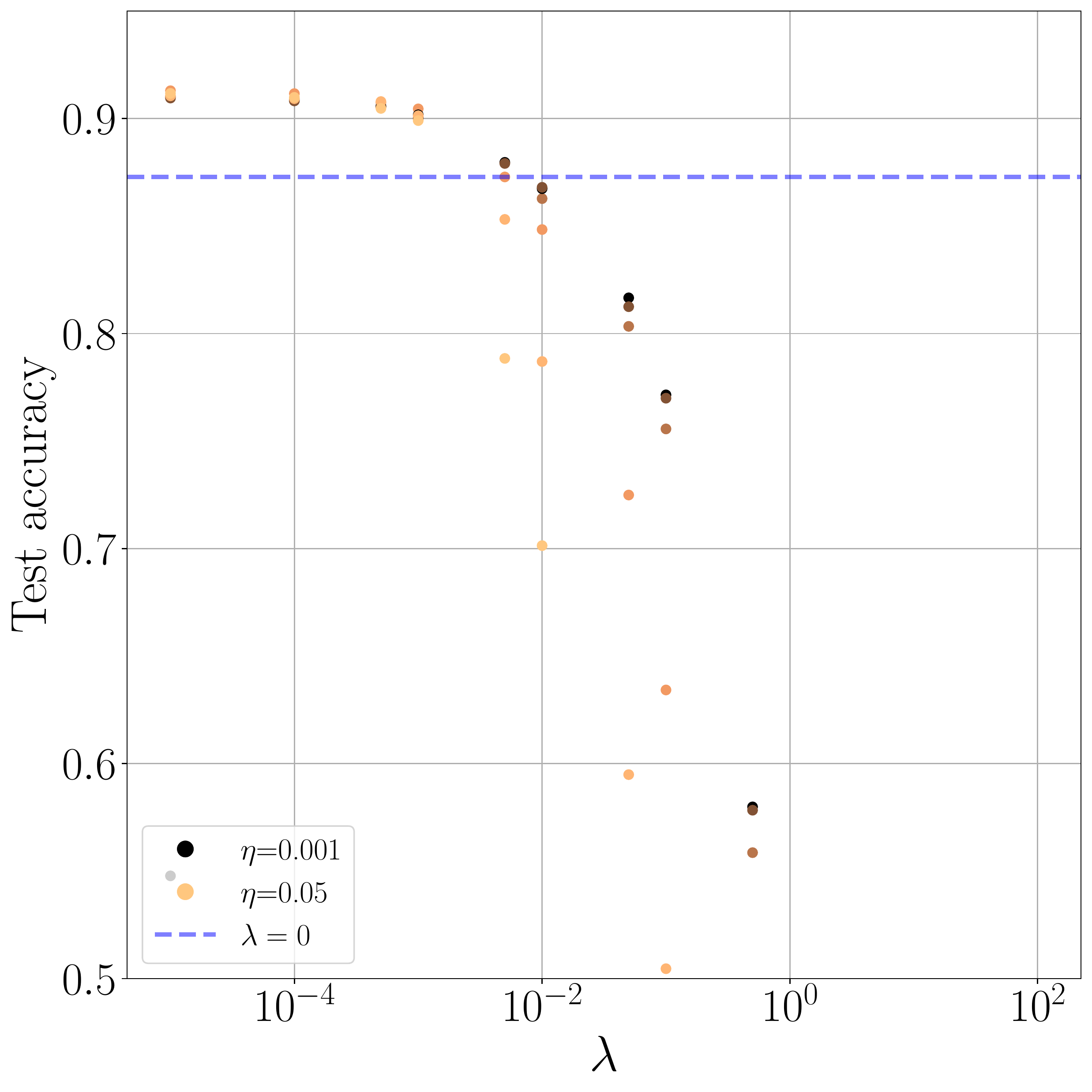}
     }
\\
   \subfloat[]
   {
    \includegraphics[width=0.33 \textwidth]{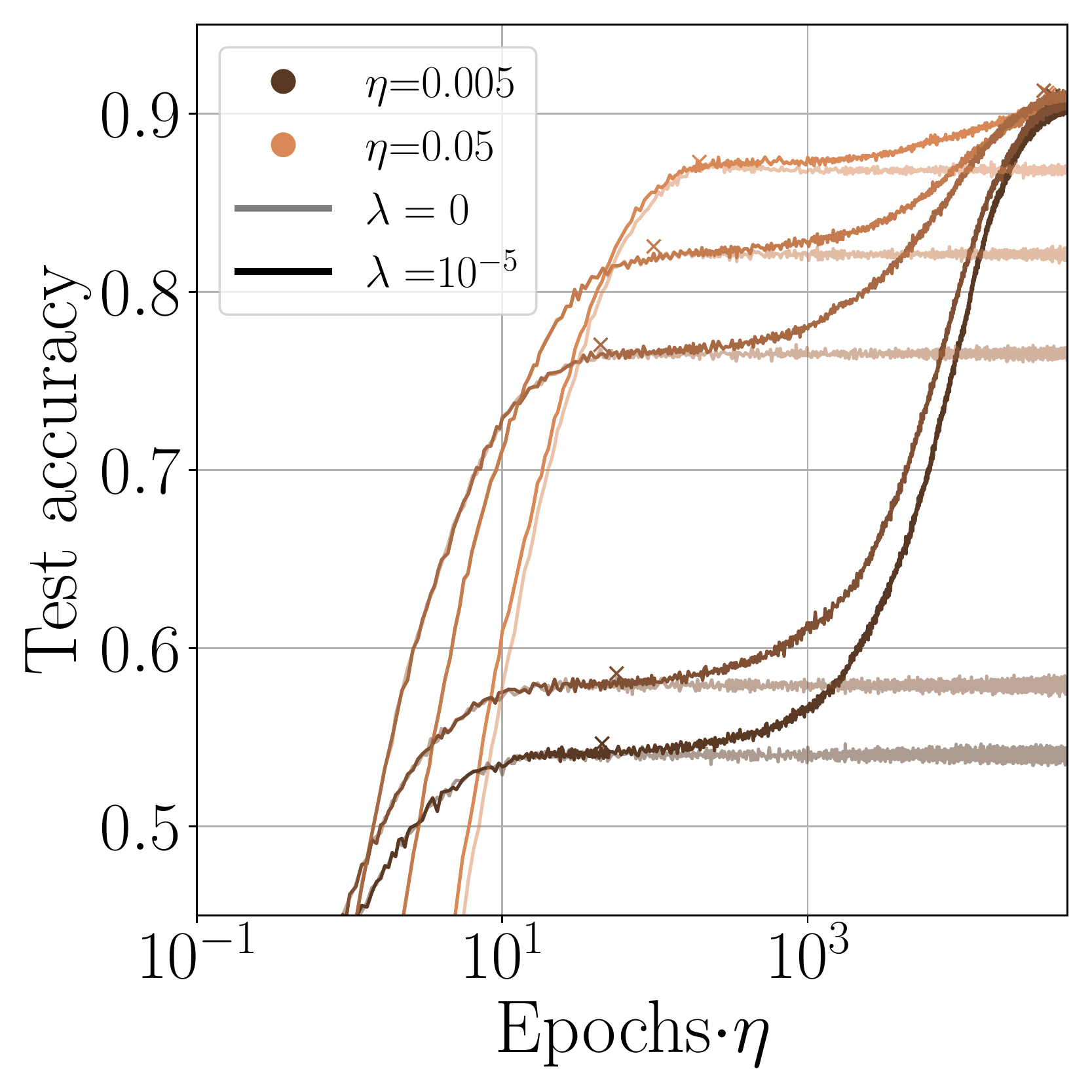}
     }
    \subfloat[]
   {
    \includegraphics[width=0.33 \textwidth]{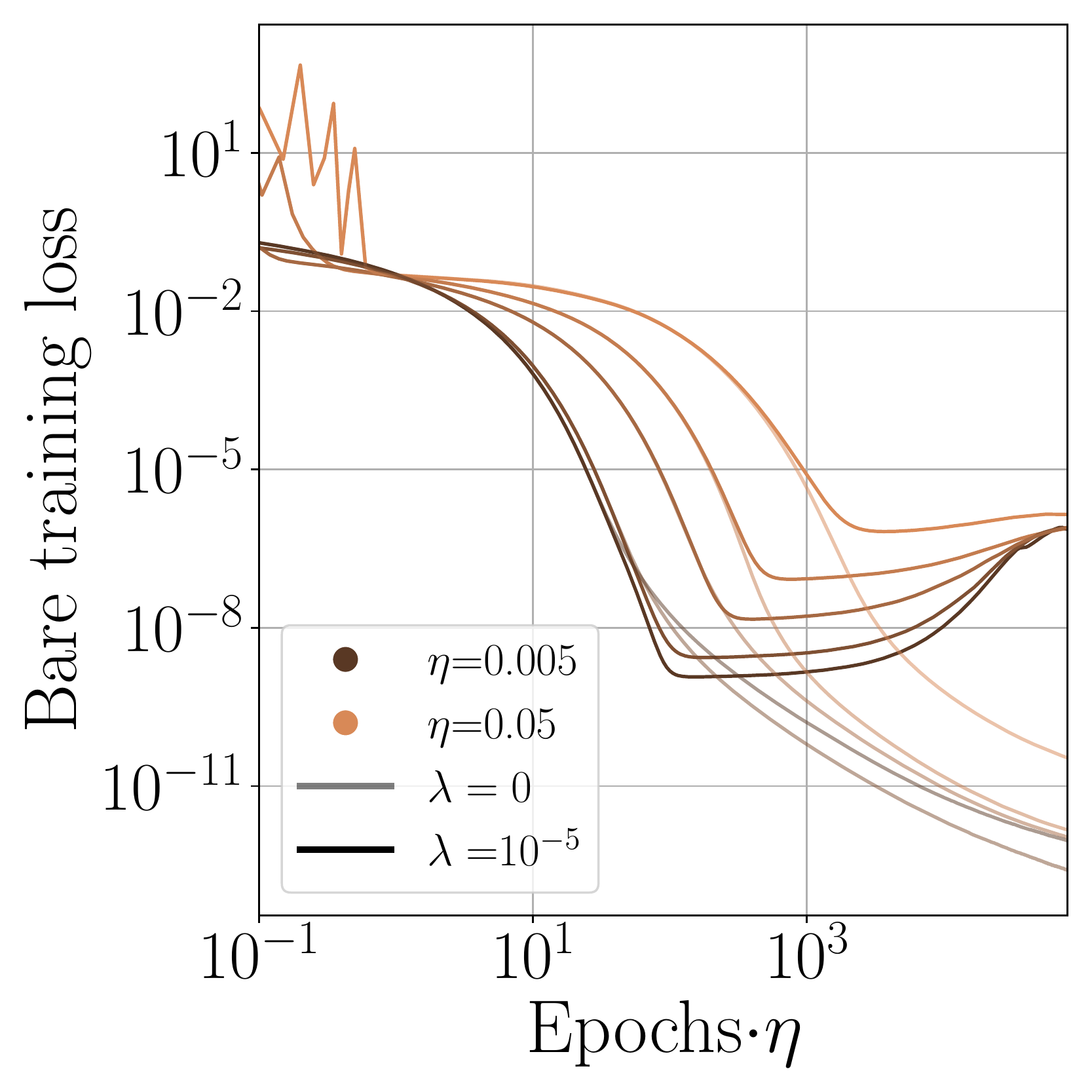}
     }
    \caption{MSE and catapult effect: we see how even if there is a strong dependence of the test accuracy on the learning rate when the training accuracy is $1$, this dependence flattens out when evolved until convergence in the presence of $\lambda$ . The specific values for the $\eta,\lambda$ sweeps are in \ref{sec:expdetails}.} 
    \label{fig:catapult}
\end{figure}

\clearpage
\newpage
\subsection{Dynamics of loss and accuracy}
In figure \ref{fig:SMdynamics} we illustrate the training curves of the experiments we have discussed in the main text and SM.

\begin{figure}[H]
\centering 
\subfloat[FC]{
  \includegraphics[width=0.29\textwidth]{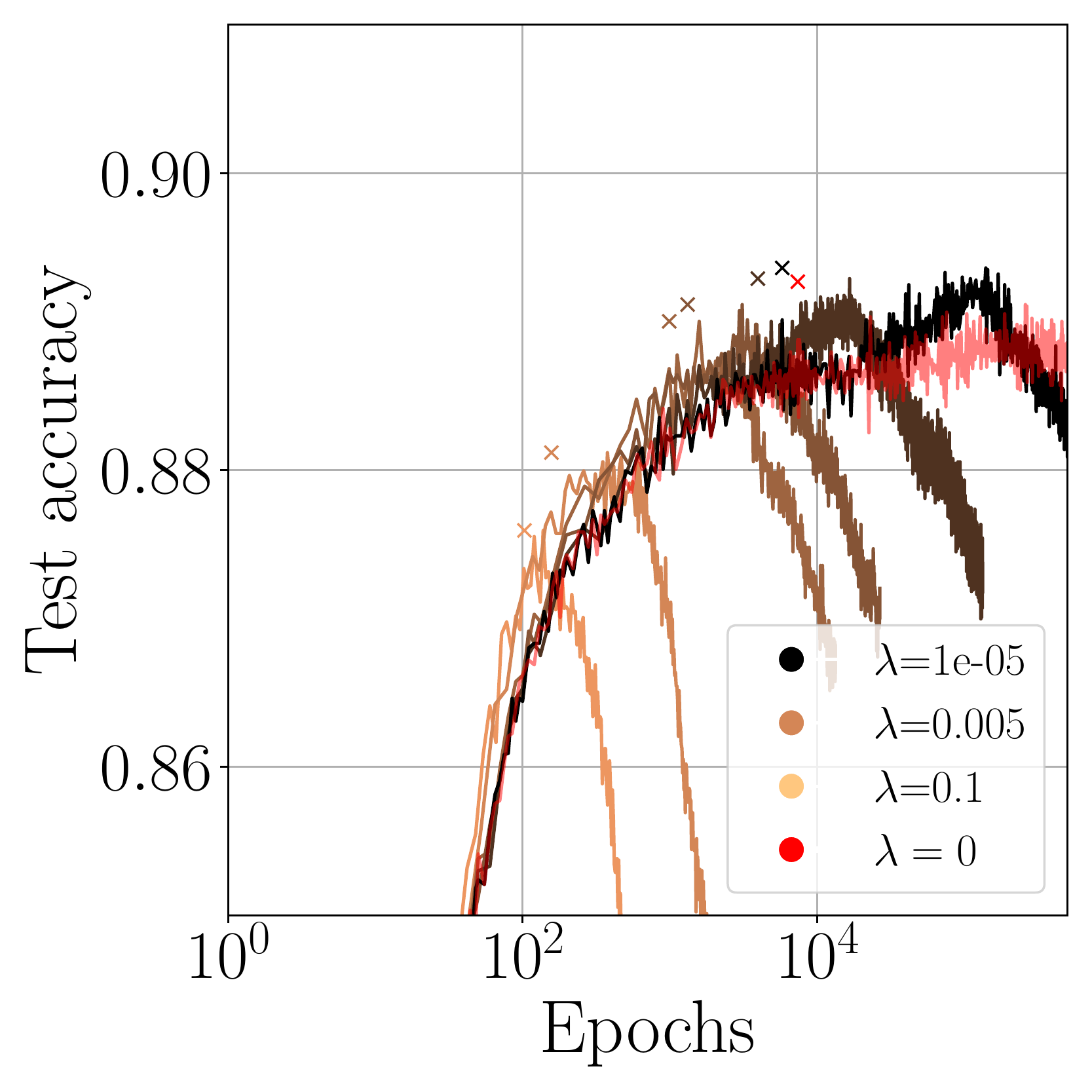}
  
} 
\subfloat[WRN]{
  \includegraphics[width=0.29\textwidth]{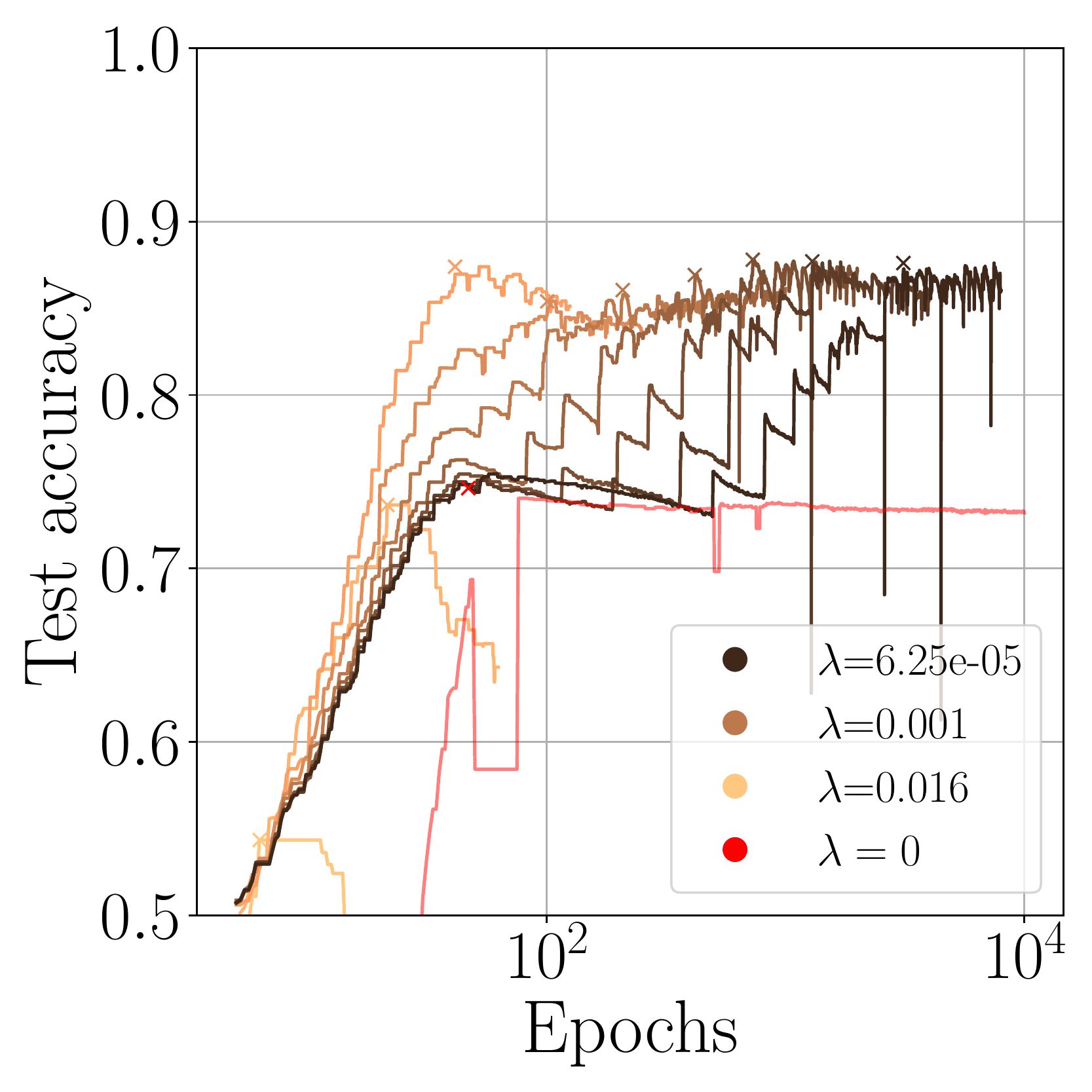}
  
} 
\subfloat[CNN no BN]{
  \includegraphics[width=0.29\textwidth]{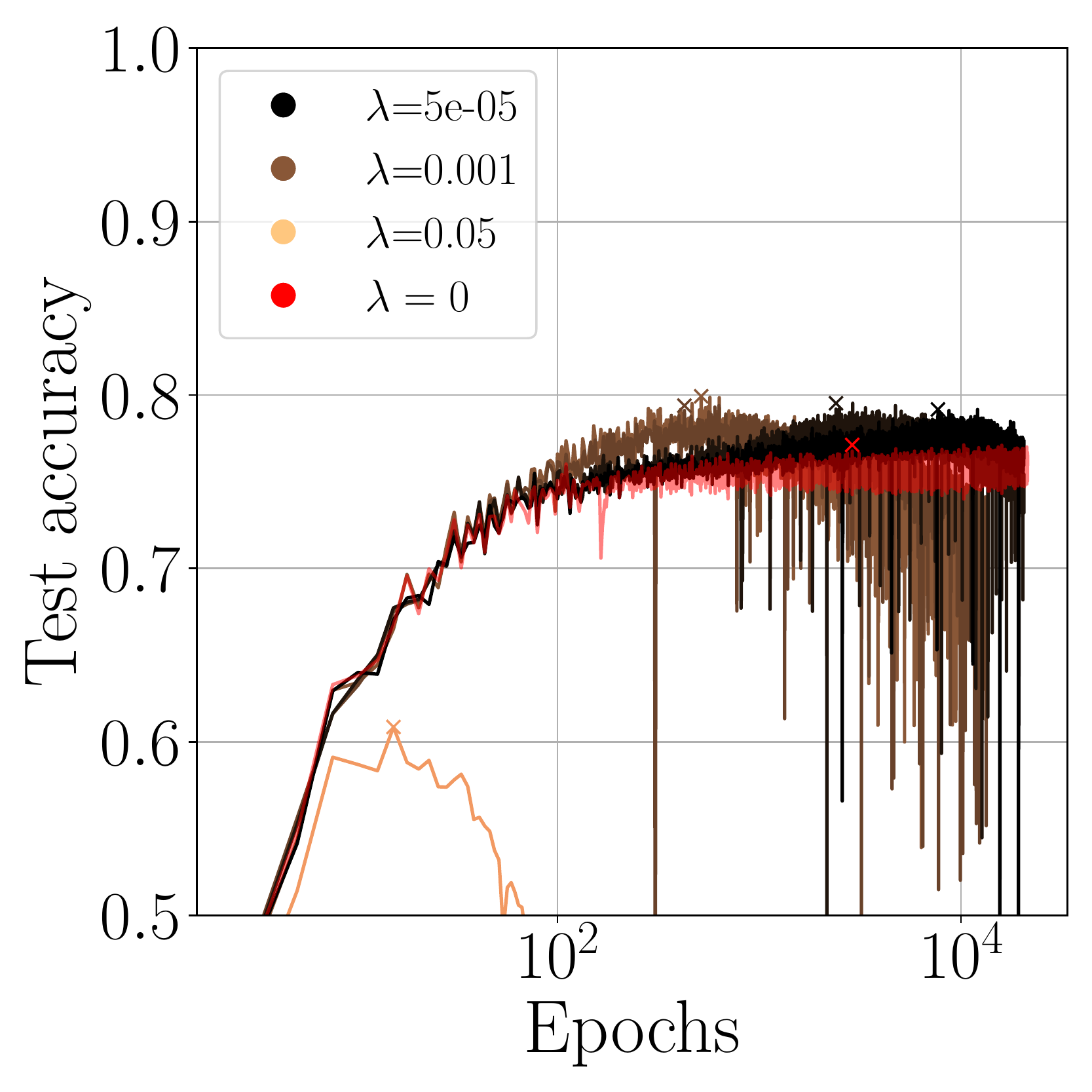}
} 
\newline
\subfloat[FC]{
  \includegraphics[width=0.29\textwidth]{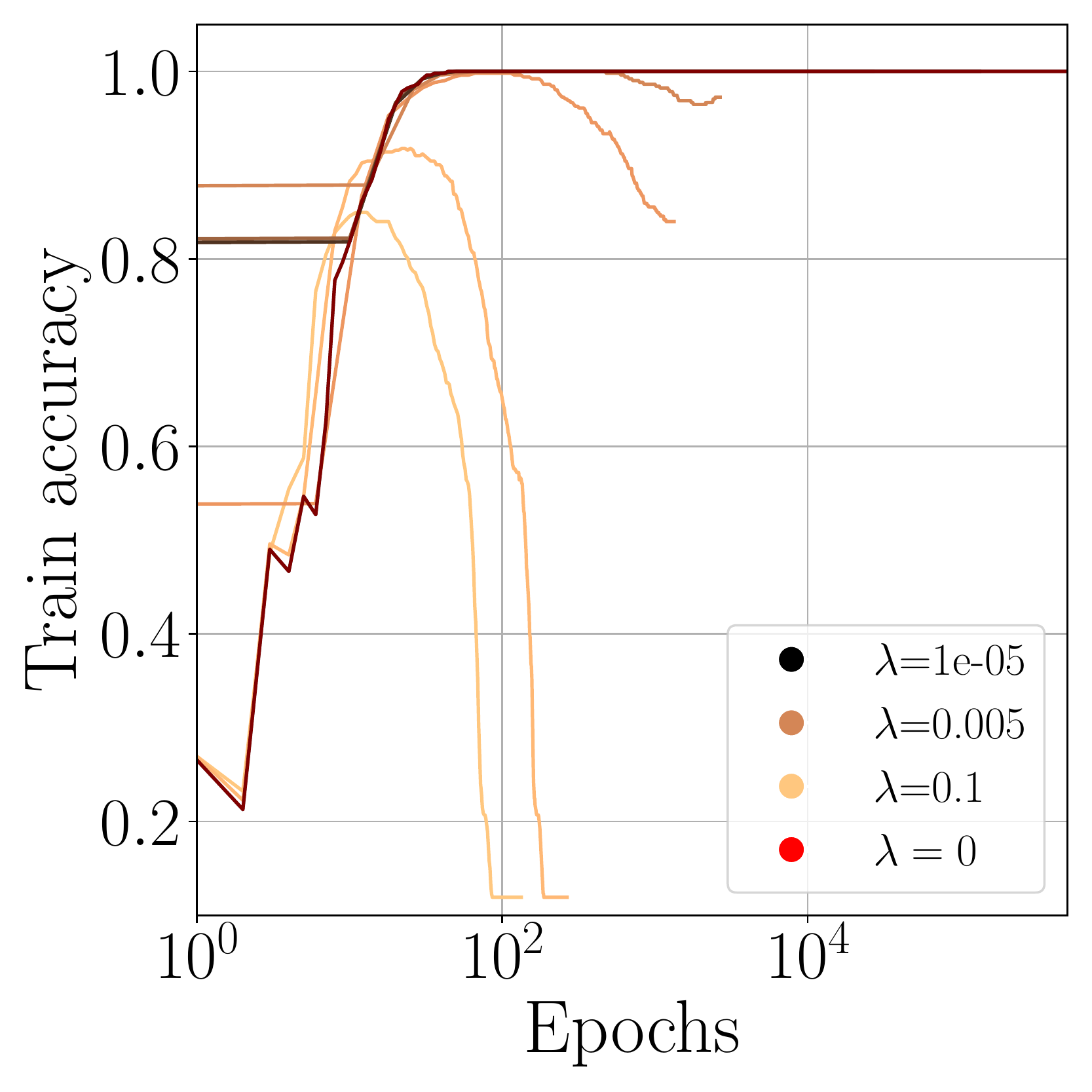}
  
} 
\subfloat[WRN]{
  \includegraphics[width=0.29\textwidth]{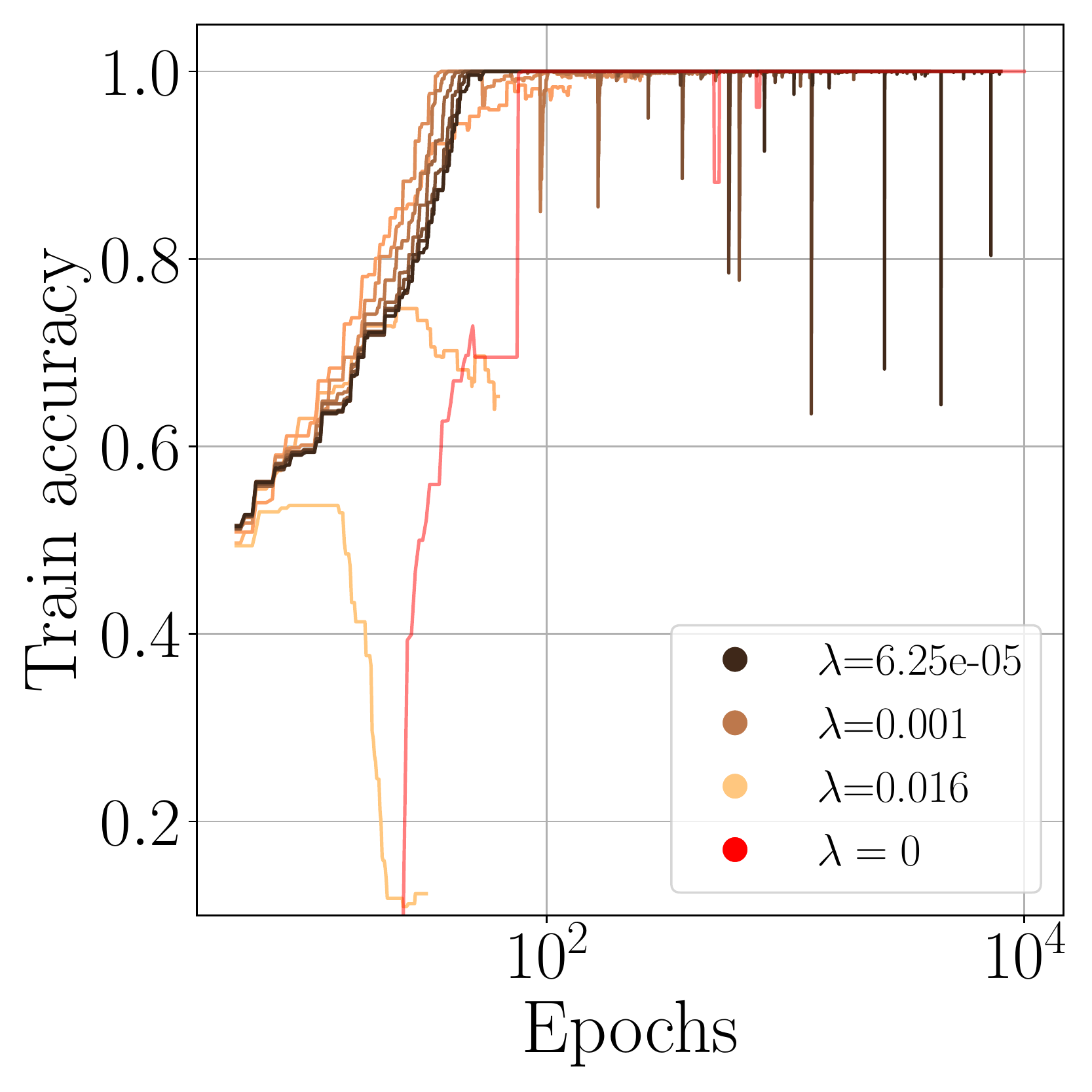}
  
} 
\subfloat[CNN no BN]{
  \includegraphics[width=0.29\textwidth]{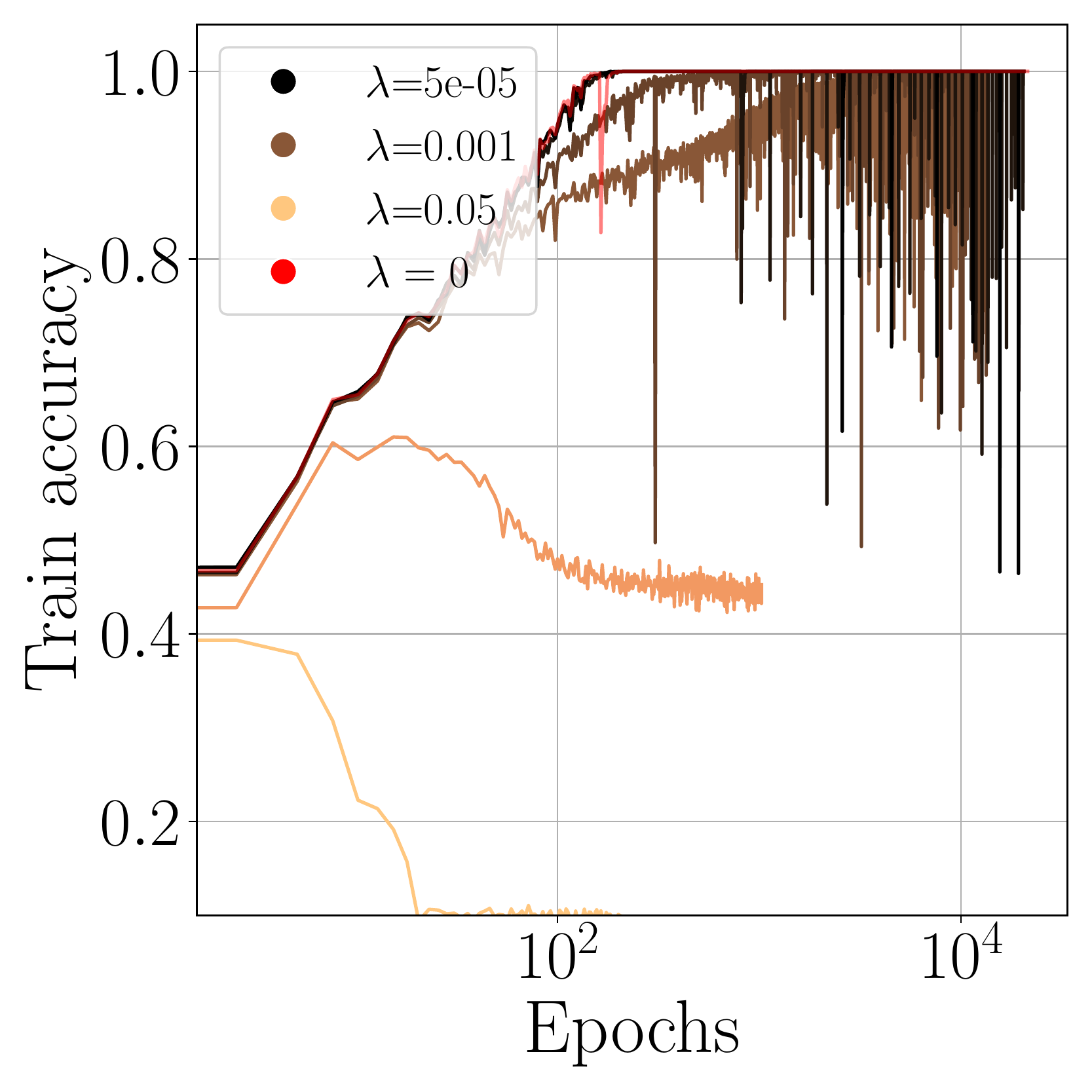}
} 
\newline
\subfloat[FC]{
  \includegraphics[width=0.29\textwidth]{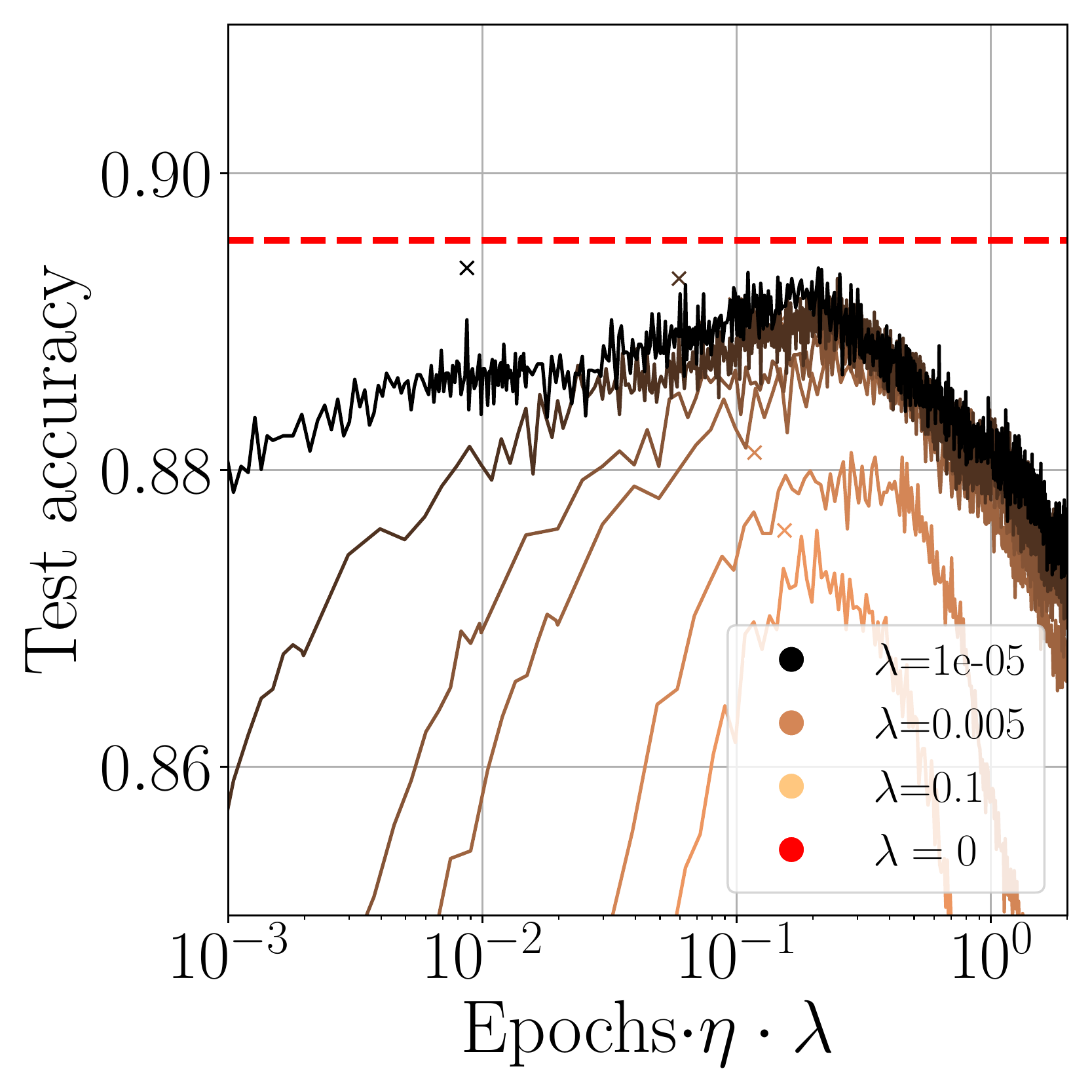}
  
} 
\subfloat[WRN]{
  \includegraphics[width=0.29\textwidth]{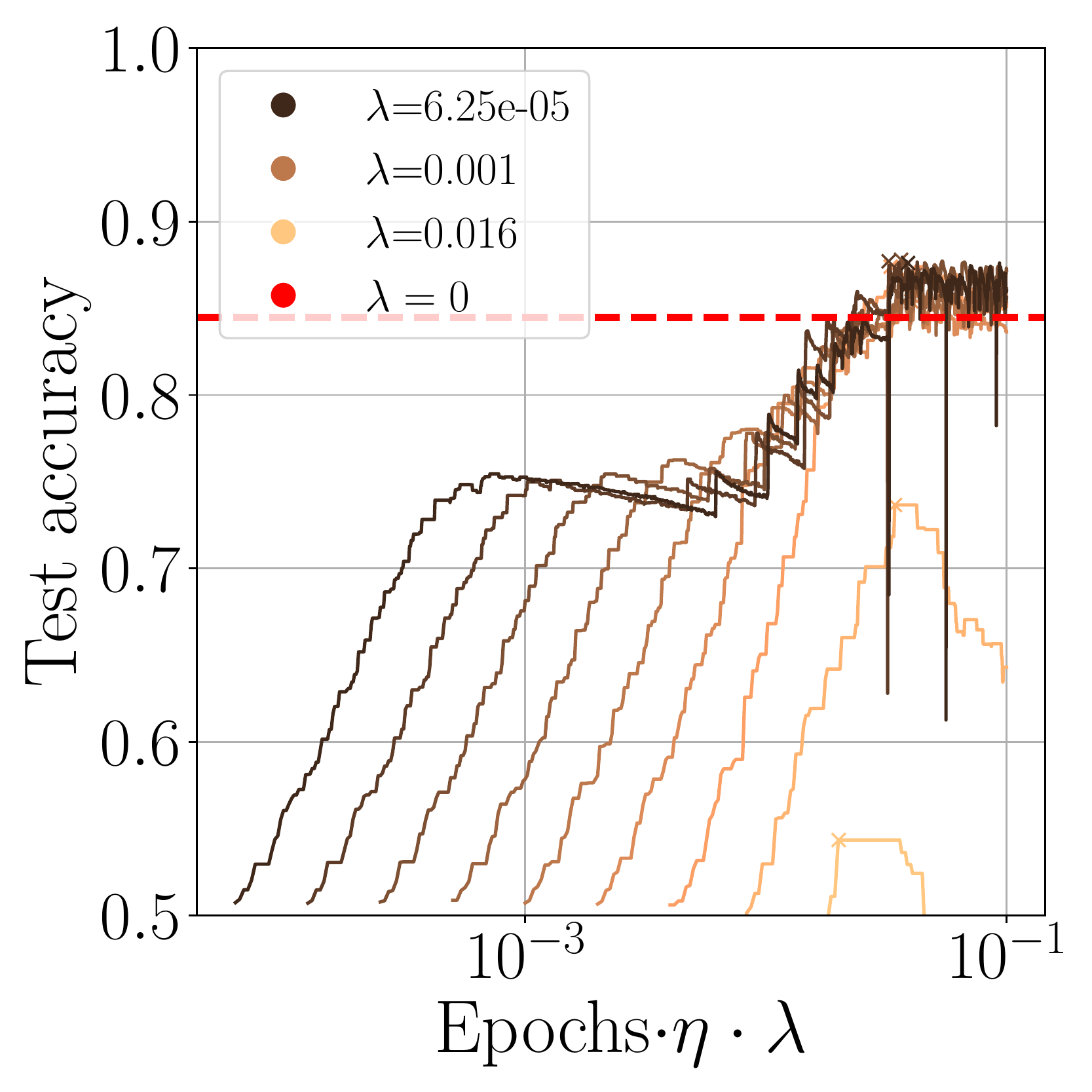}
  
} 
\subfloat[CNN no BN]{
  \includegraphics[width=0.29\textwidth]{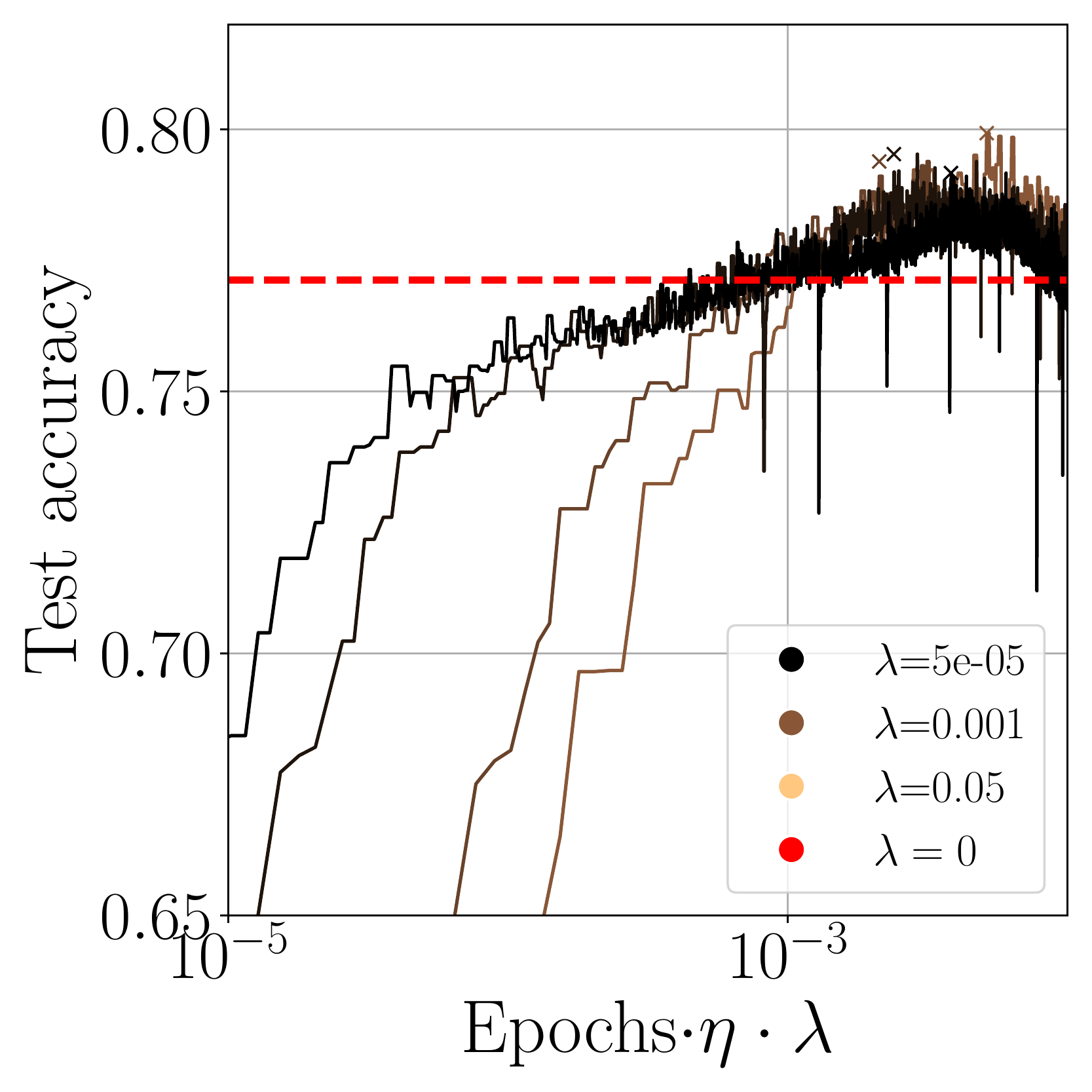}
} 
\newline
\subfloat[FC]{
  \includegraphics[width=0.29\textwidth]{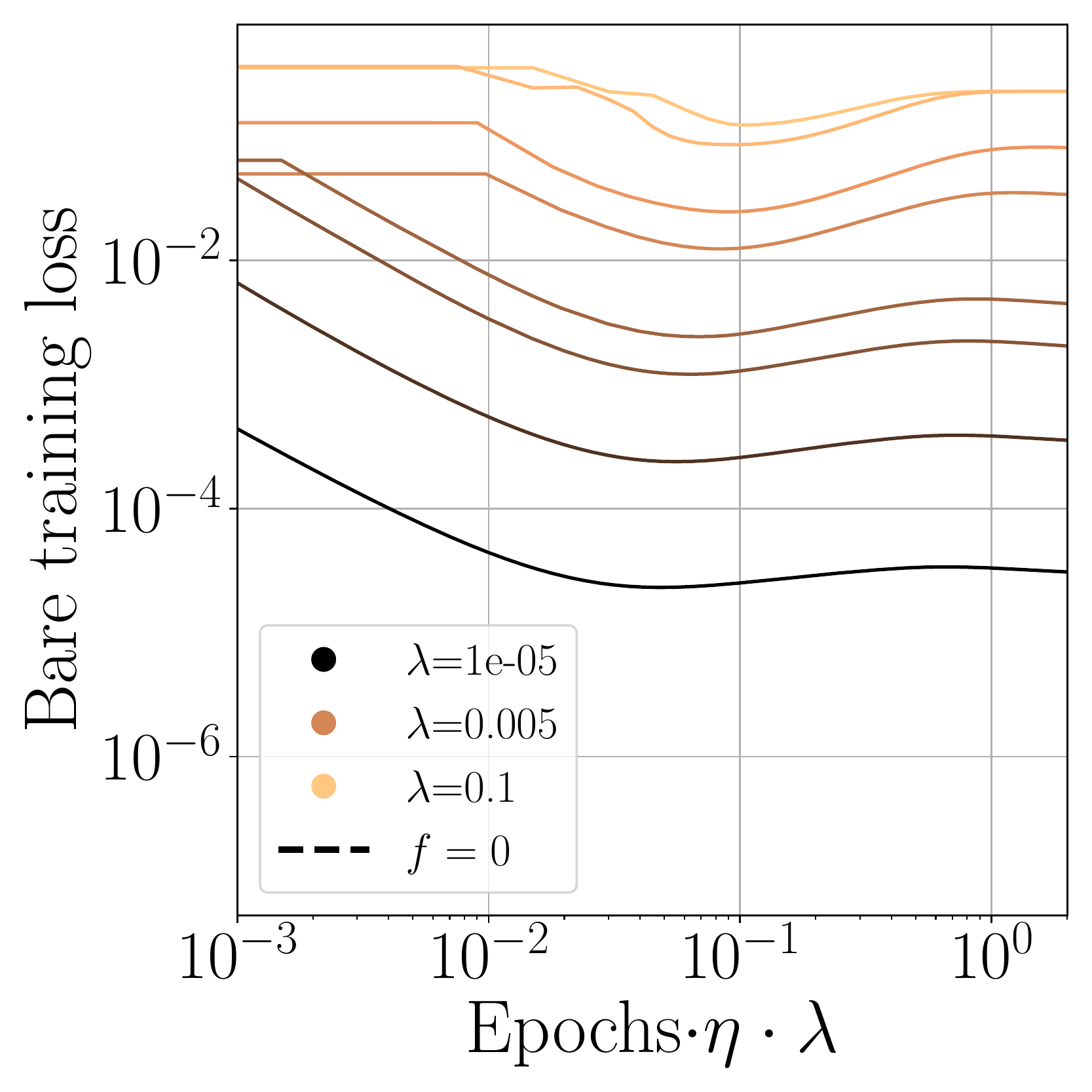}
  
} 
\subfloat[WRN]{
  \includegraphics[width=0.29\textwidth]{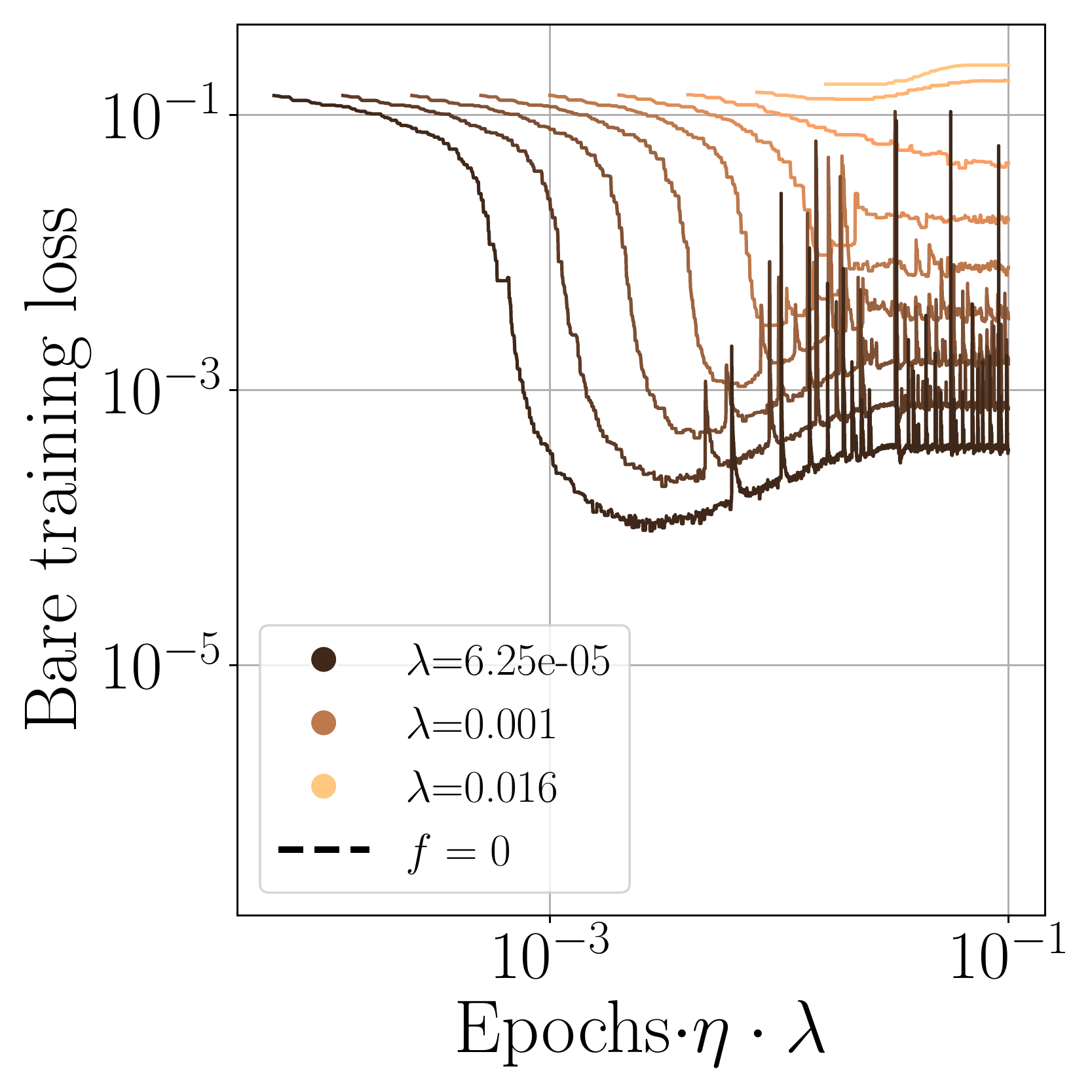}
  
} 
\subfloat[CNN no BN]{
  \includegraphics[width=0.29\textwidth]{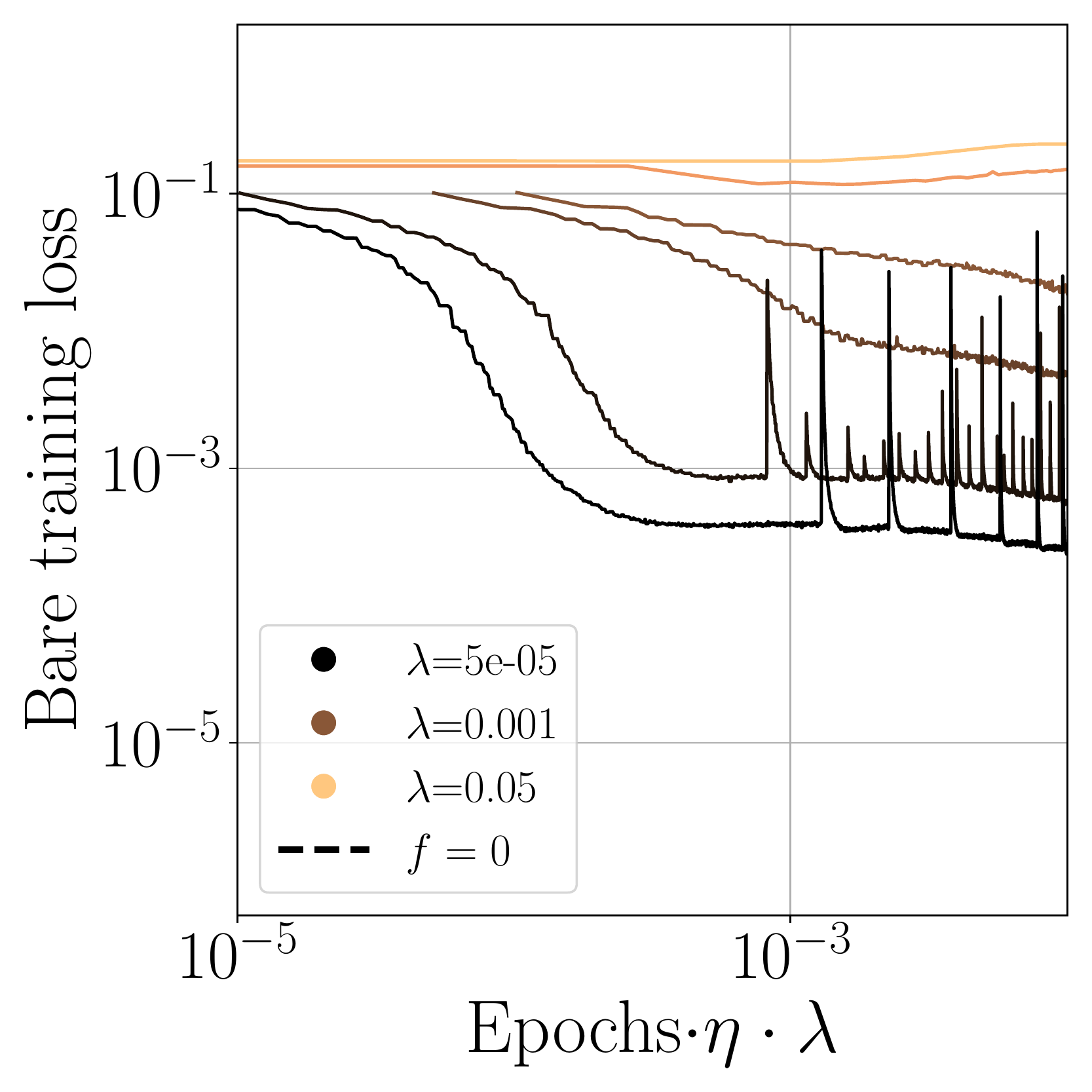}
} 
\caption{This shows the dynamics of experiments in figures \ref{fig:timedynamics}, \ref{fig:CNN}.The learning rates are $0.01$ for the CNN and $0.2$ for the WRN and $0.15$ for the FC.  Each legend has the min, max and median value of the logspace $\lambda$'s with its respective color.}
\label{fig:SMdynamics}

\end{figure} 
\section{Examples of setups where the scalings don't work}

We will consider a couple of setups which don't exhibit the behaviour described in the main text: a 3 hidden layer, width 64 fully-connected network trained on CIFAR-10 and ResNet-50 trained on ImageNet. We attribute this difference to deviations from the overparametrized/large width regime. 
In this situation, the optimal test accuracy with respect to $\lambda$ has a maximum at some $\lambda_{\rm opt} \not = 0$. 

For the FC experiment, the time it takes to reach this maximum accuracy scales like $1/\lambda$ for  $\lambda \gtrsim \lambda_{opt}$, but becomes constant (equal to the value for  $\lambda=0$) for $\lambda \lesssim \lambda_{\rm opt}$. This peak of the maximum test accuracy happens before the training accuracy reaches $1$. Generically, we don't observe that a network trained with cross-entropy and without regularization to have a peak in the test accuracy at a finite time. 

We do not have as clear an understanding of the ImageNet experimental results because they involve a learning rate schedule.
Performance for small $\lambda$s does not improve even if when evolving for a longer time. However, we do observe that performance is roughly constant when $\eta \cdot \lambda$ is held fixed.

\begin{figure}[ht!]
  \centering
    \subfloat[]{\includegraphics[width=0.33\textwidth]{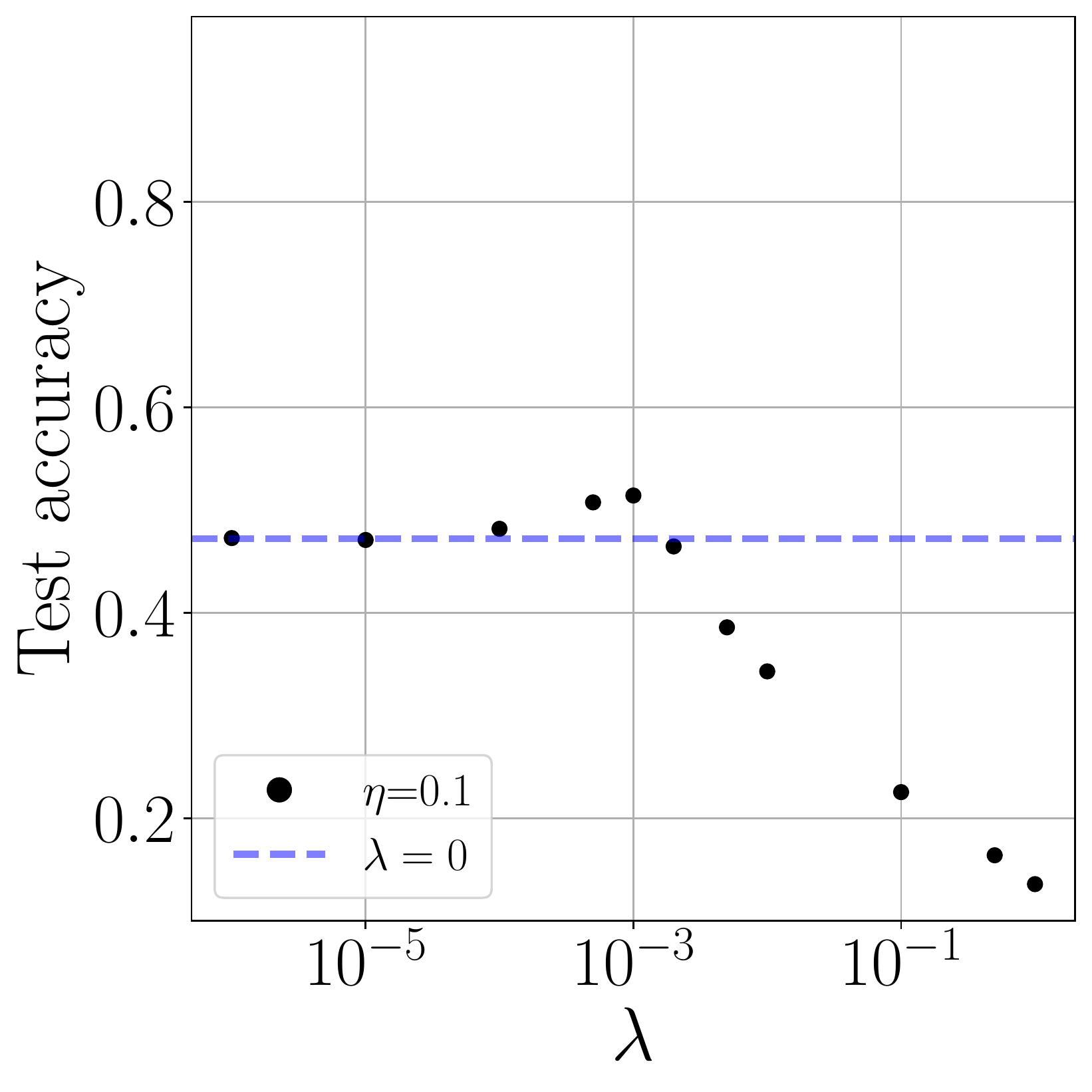}
    
}
  \subfloat[]{\includegraphics[width=0.33\textwidth]{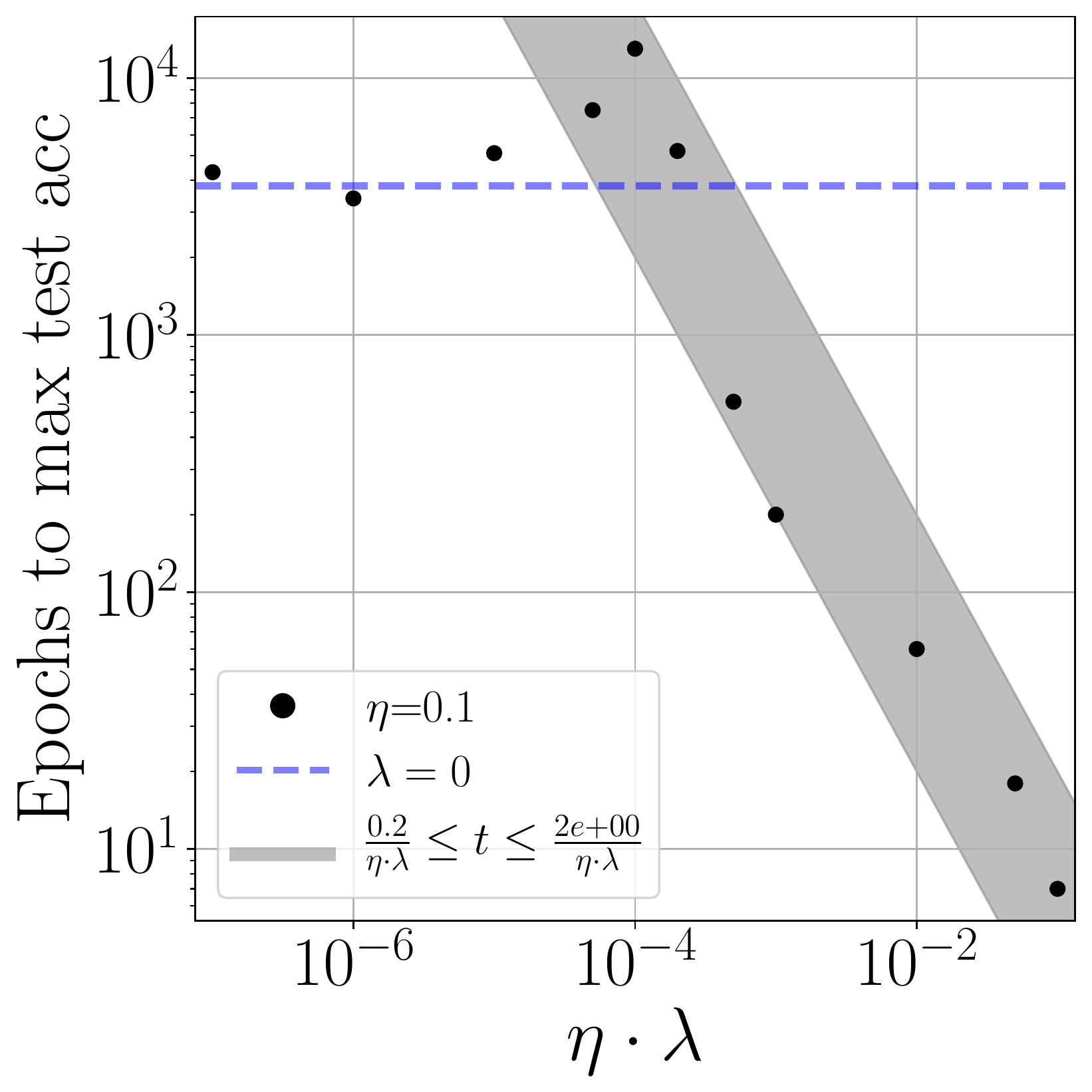} }
   \subfloat[]{\includegraphics[width=0.33\textwidth]{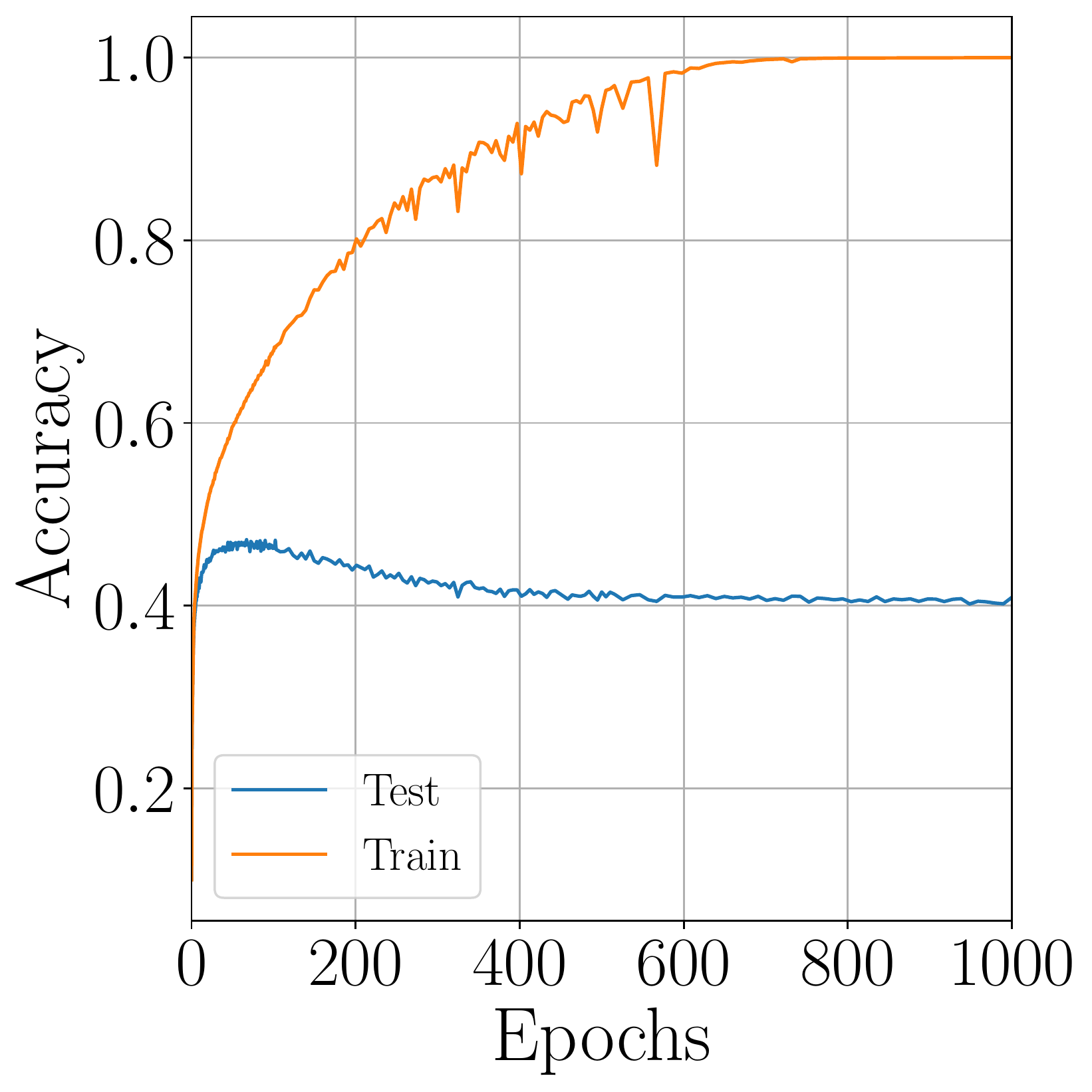} }
  \caption{Experiments with FC of width 64 and depth 3 trained on CIFAR-10.}
   \label{fig:cifarfc}

\end{figure}

\begin{figure}[ht!]
\centering
\subfloat[]{\includegraphics[width=0.5\textwidth]{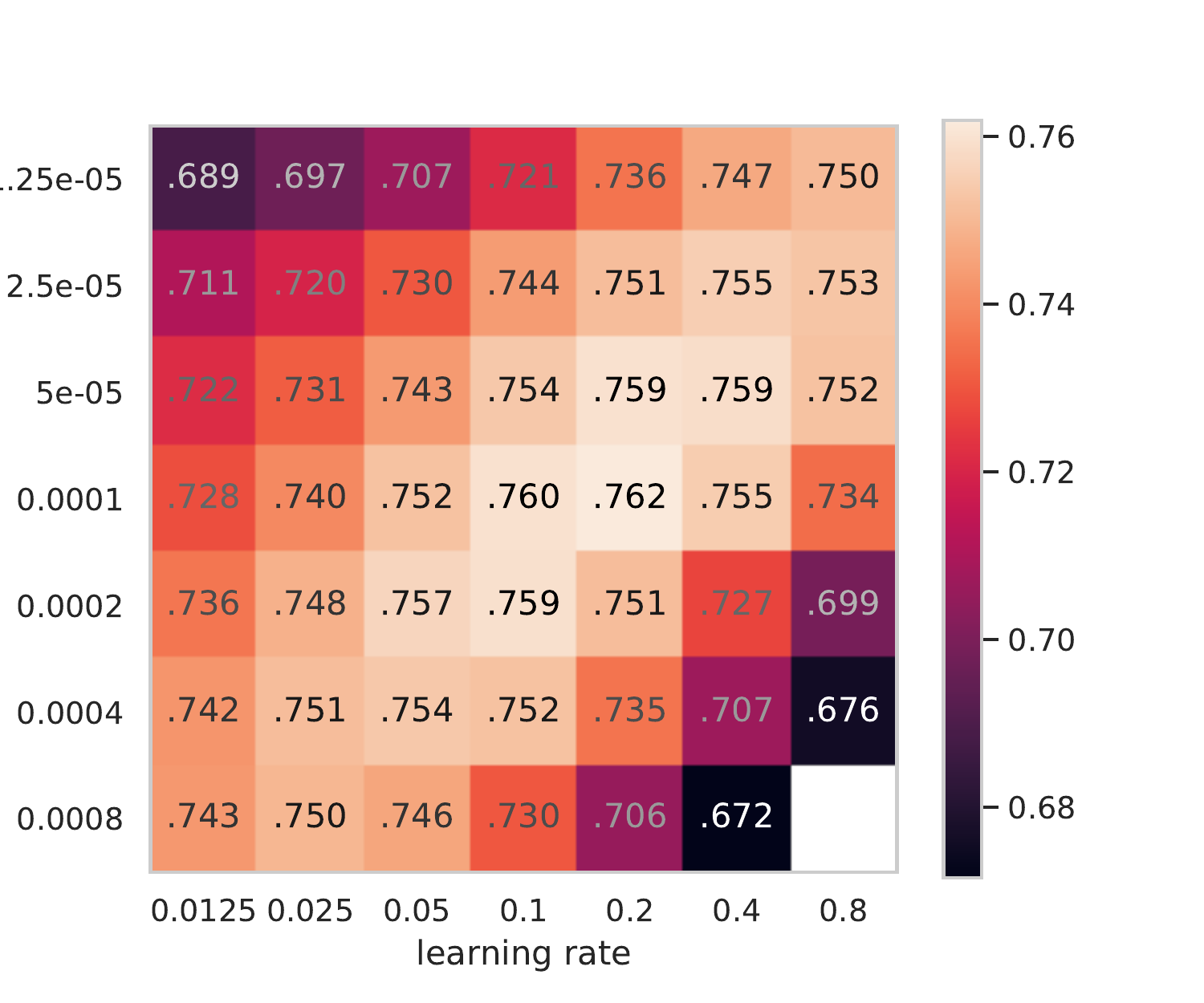} }
\subfloat[]{\includegraphics[width=0.5\textwidth]{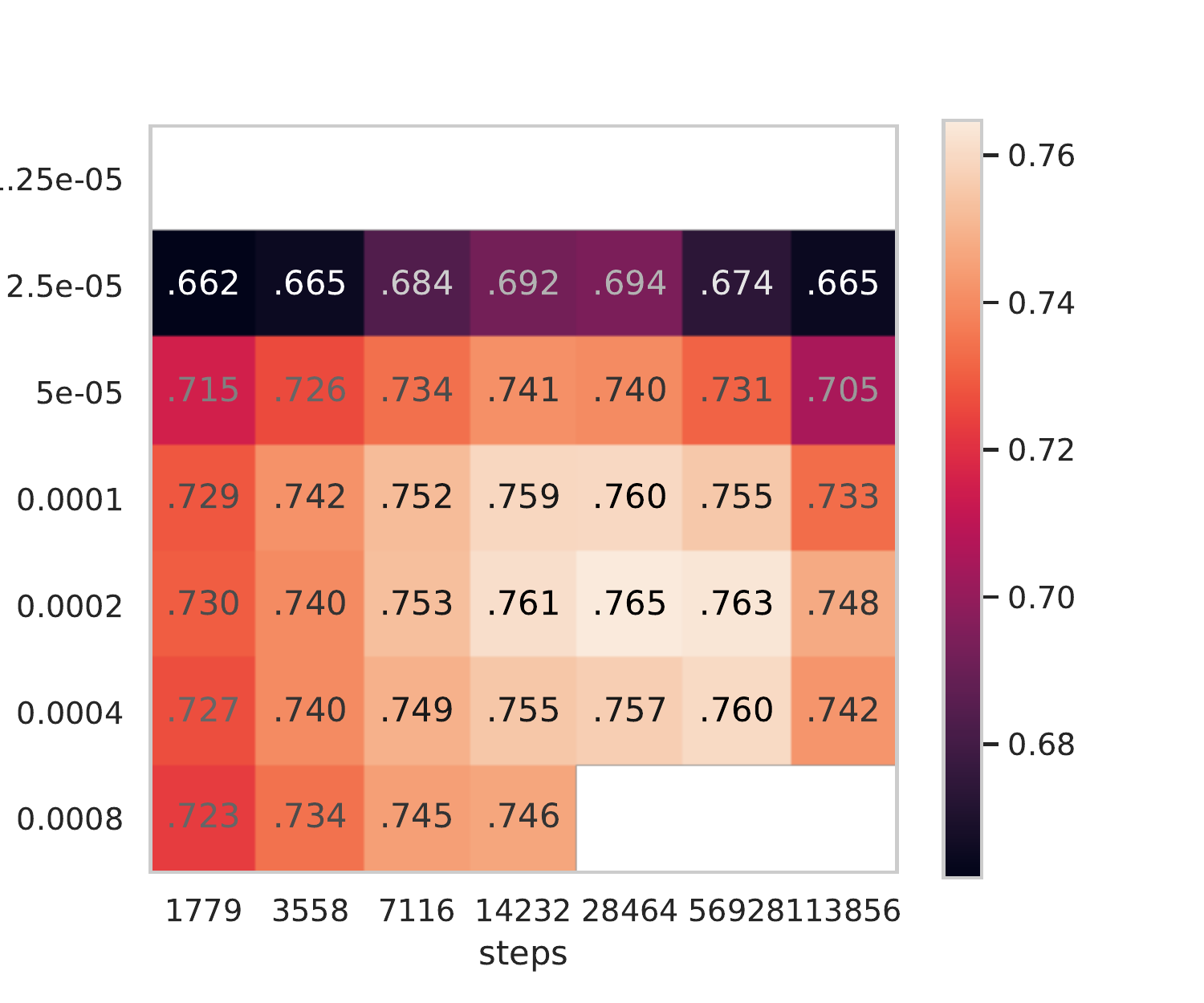} }
 \caption{ResNet-50 trained on ImageNet. (a) Evolved for 14,232 epochs for different $\eta,\lambda$. While changing $\eta$ or $\lambda$ independently has a strong effect of performance, we see that performance is rather similar along the diagonal. (b) Fixed $\eta=0.1$ and evolve for different $\lambda$ and number of epochs $T$ (rescaling the learning rate schedule to $T$). In constrast to the overparameterized case, we see that one cannot reach the same performance with a smaller $\lambda$ by increasing $T$.}
 \label{fig:imagenet}
\end{figure}
\section{More on \AutoLtwo}
\label{sec:AutoL2}
The algorithm is the following:

\begin{algorithm}[]

 \textsc{minloss},\textsc{minerror}=$\infty,\infty$
 
 \textsc{min\_step},L2=0,0.1
 
 \For{t in steps}{
  \textsc{update\_weights\;}
  \If{$t ~ \text{mod} ~ k = 0$}{
  \textsc{make\_measurements}\;
  \eIf{\textsc{error\_or\_loss\_increases and} t \textsc{> min\_step} }{
   \textsc{L2}=\textsc{L2}/10\;
   \textsc{min\_step}=$0.1/\textsc{L2}+t$\;
   }{
   \textsc{minloss},\textsc{minerror}=min($\textsc{loss}_{t-k}$,\textsc{minloss}),min($\textsc{error}_{t-k}$,\textsc{minerror})\;
  }
 }}
  \SetKwFunction{FMain}{\textsc{error\_or\_loss\_increases }}
  \SetKwProg{Fn}{Function}{:}{}
  \Fn{\FMain{\textsc{loss},\textsc{error},\textsc{minloss},\textsc{minerror}}}
  {
        \If{$\textsc{loss}_t$>\textsc{minloss} \textsc{and} $\textsc{loss}_{t-k}$>\textsc{minloss}} {
\KwRet True\;}
        \If{$\textsc{error}_t$>\textsc{error} \textsc{and} $\textsc{error}_{t-k}$>\textsc{error}} {
\KwRet True\;}

        \KwRet False\;
  }
 \caption{\textsc{AutoL2}}
\end{algorithm}

We require the loss/error to be bigger than its minimum value two measurements in a row (we make measurements every 5 steps), we do this to make sure that this increase is not due to a fluctuation. After decaying, we force $\lambda$ to stay constant for a time $0.1/\lambda$ steps, we choose the refractory period to scale with $1/\lambda$ because this is the physical scale of the system.

To complement the \AutoLtwo~discussion of section \ref{sec:L2sch} we have done another experiment where the learning rate is decayed using the schedule described in \ref{sec:lrsch}. Here we see how while \AutoLtwo~trains faster in the beginning, the optimal $\lambda=0.0005$ outperforms it. We have not hyperparameter tuned the possible parameters of \AutoLtwo. 

\begin{figure}[ht!]
  \centering
    \subfloat[]{\includegraphics[width=0.33\textwidth]{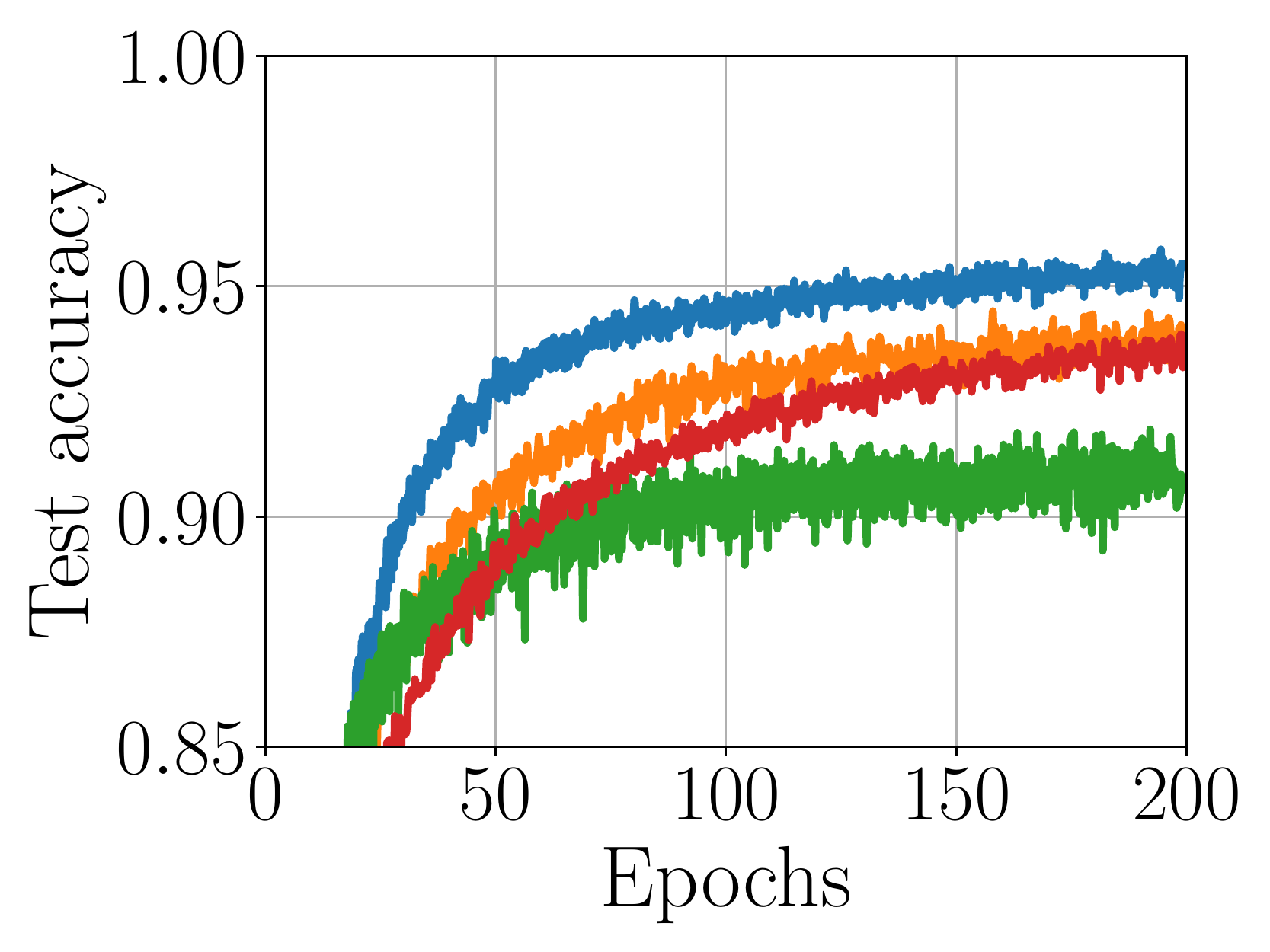}
  }
  \subfloat[]{\includegraphics[width=0.33\textwidth]{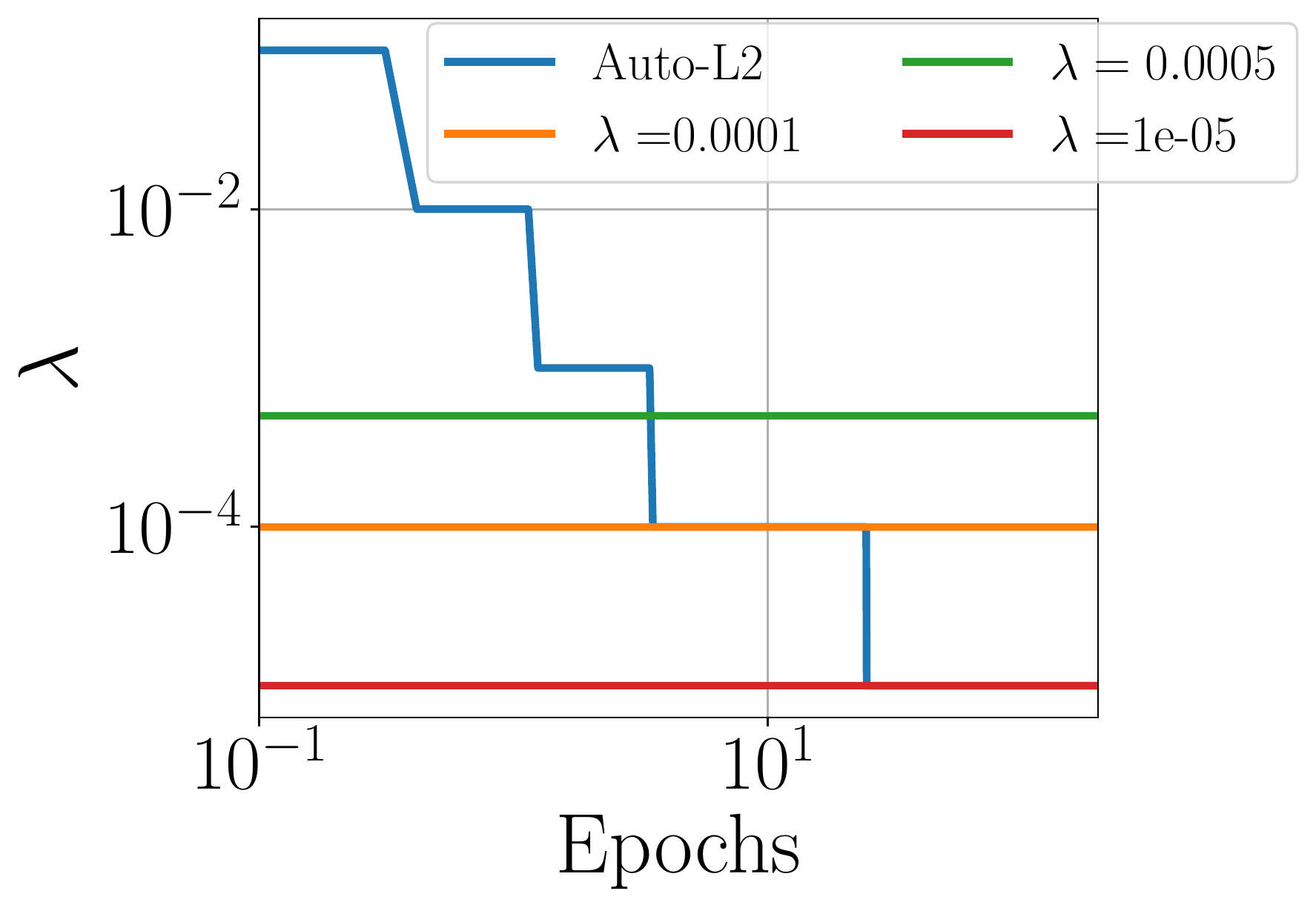}

}

\caption{Here we have a WRN trained with momentum and data augmentation for $200$ epochs. We compare the \AutoLtwo~ with different fixed $L_2$ parameters and we see how it trains faster and gets better accuracy. }
\label{fig:L2schedule}
\end{figure}

We have also applied \AutoLtwo~ to other setups, and in the absence of learning rate schedules beats the optimal $L_2$ parameter. See figure \ref{fig:L2other}.

\begin{figure}[ht!]
  \centering
    \subfloat[]{\includegraphics[width=0.33\textwidth]{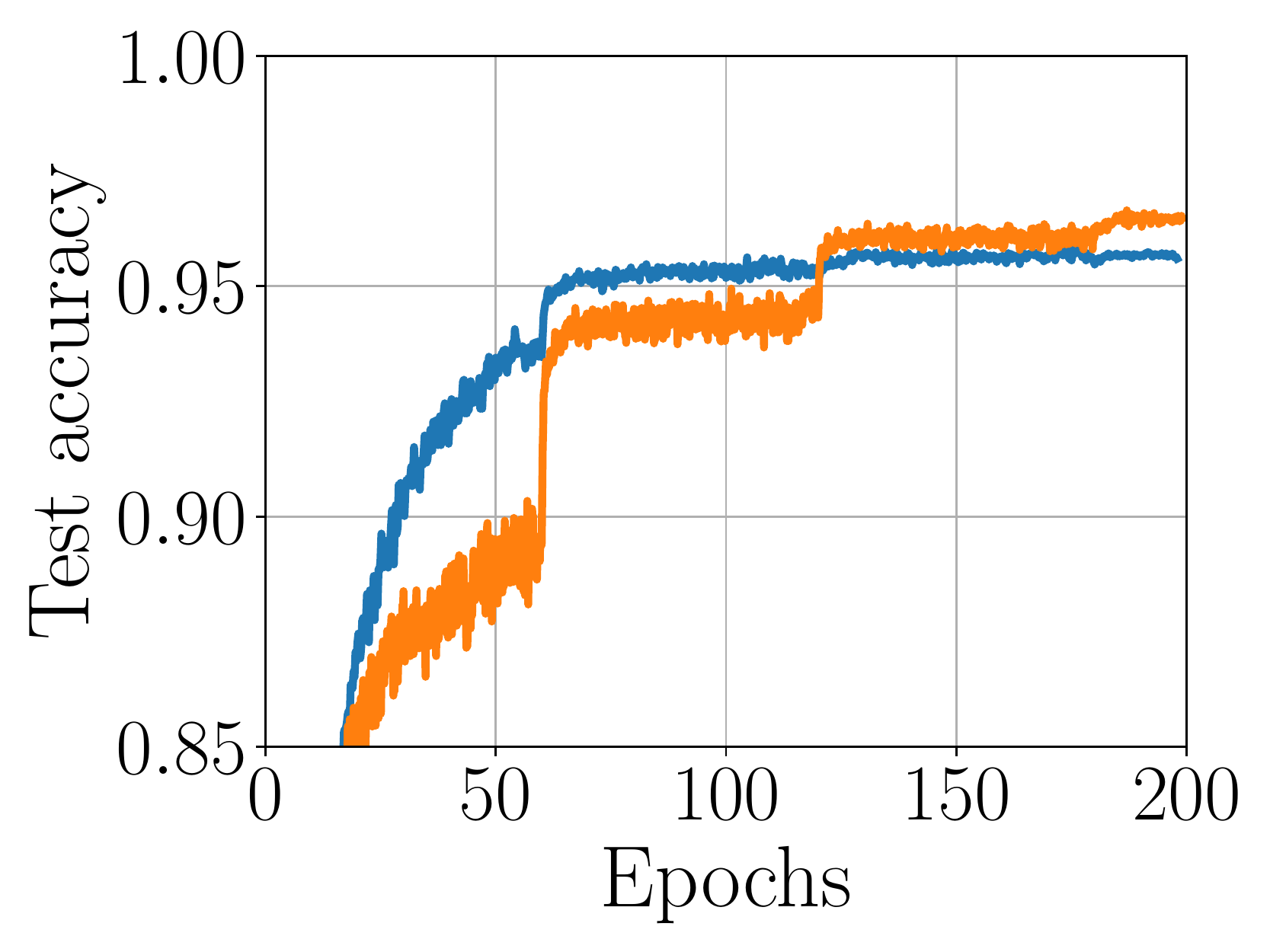}
    
}
  \subfloat[]{\includegraphics[width=0.33\textwidth]{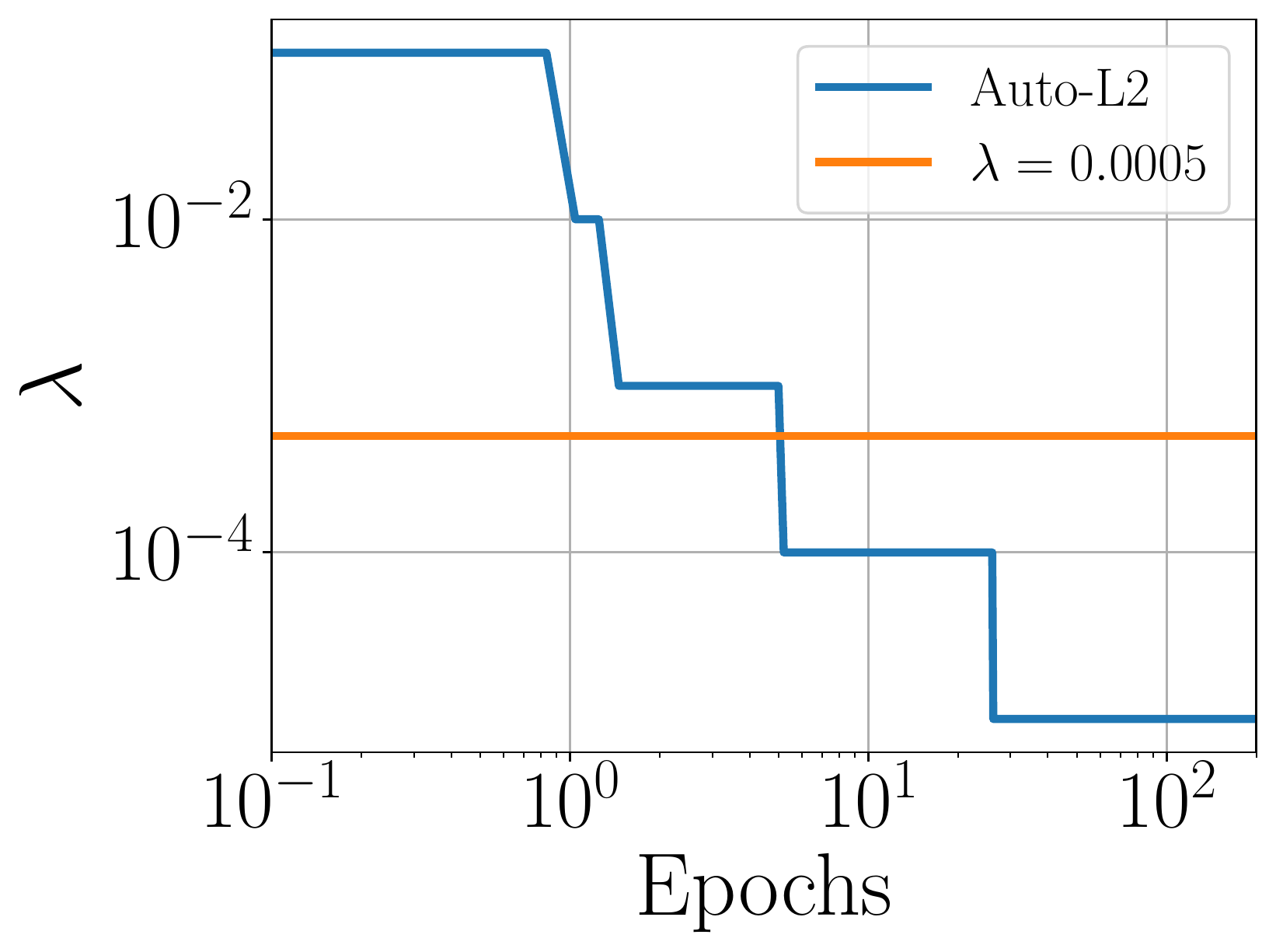} }
  \caption{Setup of \ref{fig:L2schedule} in the presence of a learning rate schedule. While \AutoLtwo~ trains faster and better in the beginning, it can't keep pace with the big jumps of constant $\lambda$.}
\end{figure}

\begin{figure}[ht!]
  \centering
{\includegraphics[width=0.75\textwidth]{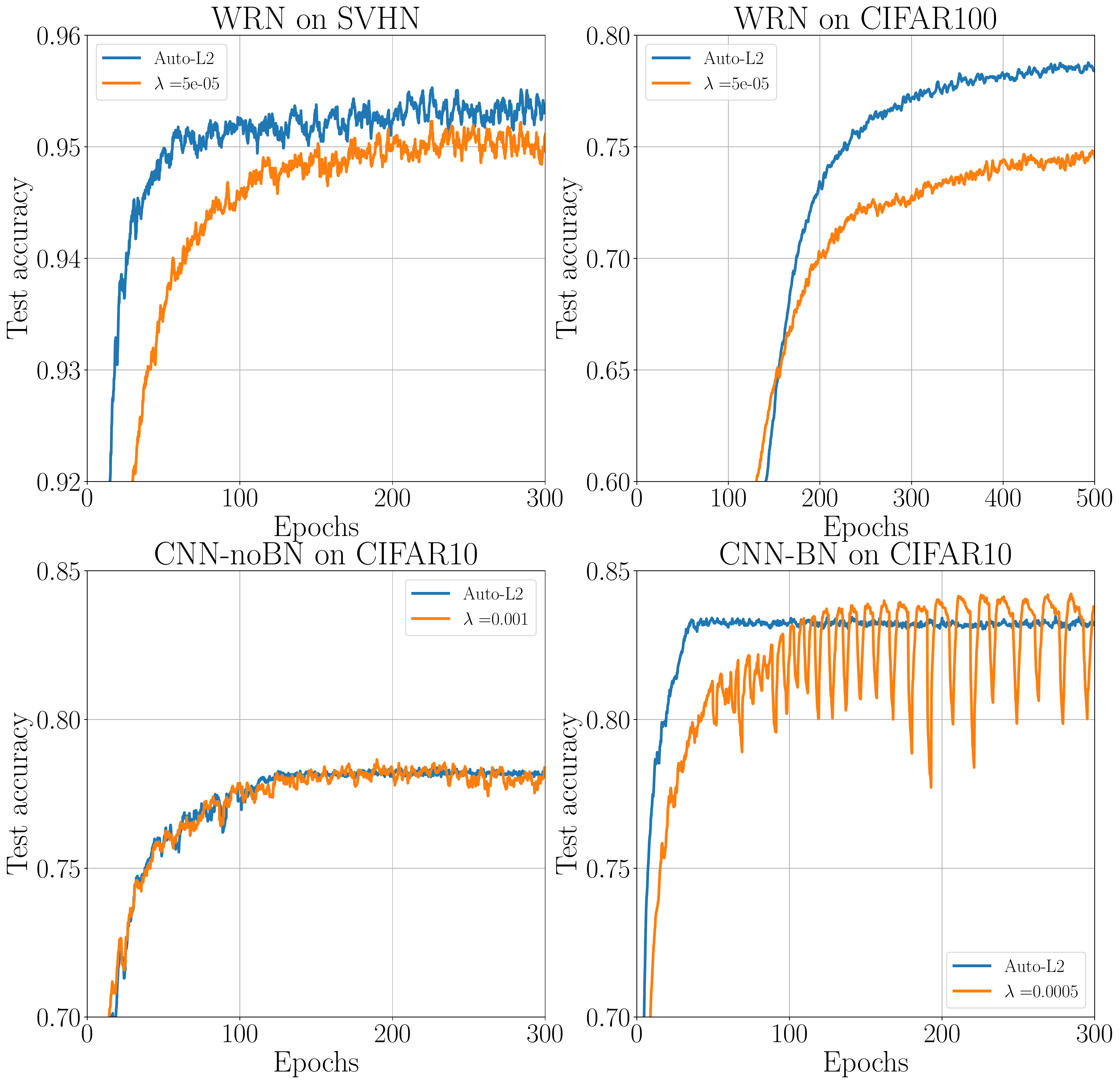}
    
}
  \caption{\AutoLtwo ~ for the other architectures considered in the main text compared with the optimal $L_2$ parameters.}
  \label{fig:L2other}
\end{figure}
\newpage
\section{Different initialization scales}
\label{sec:different_init}

We have discussed the dependence on the time to convergence in $\eta,\lambda$. Another quantity which is relevant for learning is $\sigma_w$ the scale of initialization, which for WRN we set to $1$. We can repeat the WRN experiment of figure \ref{fig:fig1a} for different $\sigma_w$. We see that the final performance depends very mildly on $\sigma_w$. This is what we expect when reaching equilibrium: the dependence of properties at initialization is eventually washed away. 

\begin{figure}[ht!]
  \centering
{\includegraphics[width=0.75\textwidth]{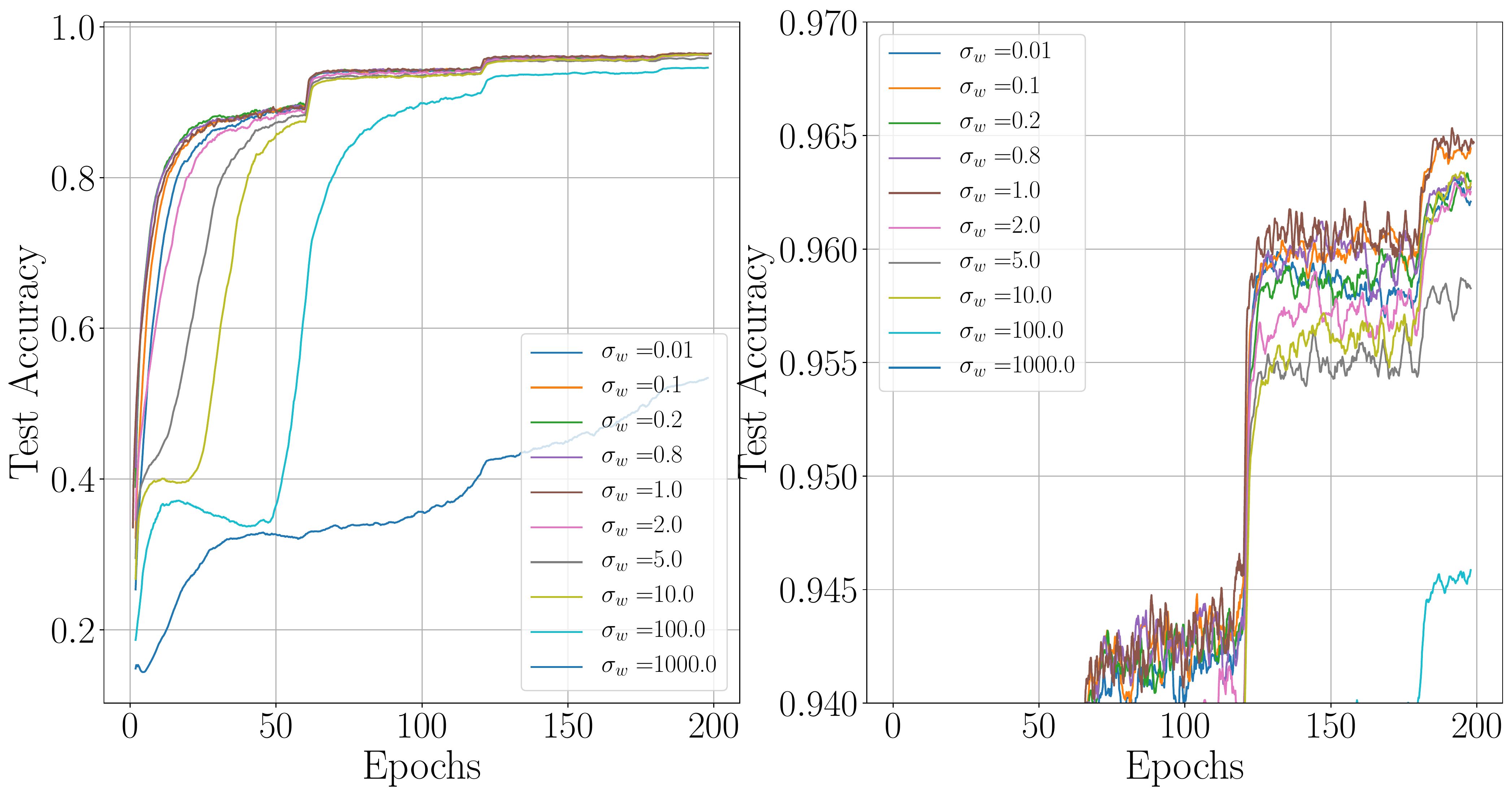}
    
}
  \caption{While models with different $\sigma_w$ behave rather differently at early times, at late times, this $\sigma_w$ dependence washes off for longer times.}
  \label{fig:difinit}
\end{figure}

\end{document}